%% file: neurips_2025.tex
\documentclass{article}
\PassOptionsToPackage{numbers, compress}{natbib}
\usepackage[preprint]{neurips_2025}

\input{macro}

\usepackage{pgfplotstable}
\usepackage{xspace}
\usepackage{paracol}
\usepackage{bibunits}
\usepackage[toc, title]{appendix}
\usepackage{minitoc}

\usepackage[colorinlistoftodos,bordercolor=orange,backgroundcolor=orange!20,linecolor=orange,textsize=scriptsize]{todonotes}

\PassOptionsToPackage{numbers, compress}{natbib}

\usepackage[utf8]{inputenc} % allow utf-8 input
\usepackage[T1]{fontenc}    % use 8-bit T1 fonts
\usepackage{hyperref}       % hyperlinks
\usepackage{url}            % simple URL typesetting
\usepackage{booktabs}       % professional-quality tables
\usepackage{amsfonts}       % blackboard math symbols
\usepackage{nicefrac}       % compact symbols for 1/2, etc.
\usepackage{microtype}      % microtypography
\definecolor{grey_plot}{HTML}{7f7f7f}
\definecolor{red_plot}{HTML}{EB2644}

\usepackage{algorithm}
\usepackage[noend]{algpseudocode}
\usepackage{setspace}

\algnewcommand{\IfThenElse}[3]{% \IfThenElse{<if>}{<then>}{<else>}
  \algorithmicif\ #1\ \algorithmicthen\ #2\ \algorithmicelse\ #3}

\newcommand{\opt}{\mathtt{OPT}}
\newcommand{\upd}{\mathtt{UPDATE}}

\newcommand{\method}{\texttt{DES-LOC}\xspace}
\newcommand{\methodadam}{\texttt{DES-LOC}\texttt{-Adam}\xspace}
\newcommand{\methodadopt}{\texttt{DES-LOC}\texttt{-ADOPT}\xspace}
\newcommand{\methodsgdm}{\texttt{DES-LOC}\texttt{-SGDM}\xspace}
\newcommand{\adam}{\texttt{Adam}\xspace}
\newcommand{\ddp}{\texttt{DDP}\xspace}

\newcommand{\localsgd}{\texttt{Local} \texttt{SGD}\xspace}
\newcommand{\local}{\texttt{Local}\xspace}
\newcommand{\localadopt}{\texttt{Local} \texttt{ADOPT}\xspace}
\newcommand{\localadam}{\texttt{Local} \texttt{Adam}\xspace}
\newcommand{\adopt}{\texttt{ADOPT}\xspace}
\newcommand{\fullmethod}{\textbf{Desynced Low Communication Adaptive Optimizers}\xspace}
\newcommand{\fedavg}{\texttt{FedAvg}\xspace}

\title{\method: \fullmethod for Training Foundation Models}

\let\hypersetup\truehypersetup
\makeatletter
\newcommand{\myfnsymbol}[1]{%
  \expandafter\@myfnsymbol\csname c@#1\endcsname
}
\newcommand{\@myfnsymbol}[1]{%
  \ifcase #1
  \or 1% 1
  \or 2% 2
  \or 3% 3
  \or 4% 4
  \or 5%
  \or \TextOrMath{\textasteriskcentered}{*}% 3
  \or \TextOrMath{\textasteriskcentered}{*}\TextOrMath{\textasteriskcentered}{*}% 4
  \or \TextOrMath{\textdagger}{\dagger}% 5
  \or \TextOrMath{\textasteriskcentered}{*},\TextOrMath{\textasteriskcentered}{*}\TextOrMath{\textasteriskcentered}{*}% 6
  \fi
}
\newcommand{\affiliationA}{\@myfnsymbol{1}}
\newcommand{\affiliationB}{\@myfnsymbol{2}}
\newcommand{\affiliationC}{\@myfnsymbol{3}}
\newcommand{\affiliationD}{\@myfnsymbol{4}}
\newcommand{\affiliationE}{\@myfnsymbol{5}}
\newcommand{\equalcontributor}{\@myfnsymbol{6}}
\newcommand{\biequalcontributor}{\@myfnsymbol{7}}
\newcommand{\correspondingA}{\@myfnsymbol{8}}

\author{
Alex Iacob\textsuperscript{\correspondingA,\affiliationA,\affiliationB}
\And  
Lorenzo Sani\textsuperscript{\affiliationA,\affiliationB}
\And Mher Safaryan\textsuperscript{\affiliationC}
\And 
Paris Giampouras\textsuperscript{\equalcontributor,\affiliationD}
\And  
Samuel Horváth\textsuperscript{\equalcontributor,\affiliationE}
\And  
Andrej Jovanović\textsuperscript{\equalcontributor,\affiliationA}
\And  
Meghdad Kurmanji\textsuperscript{\equalcontributor,\affiliationA}
\And 
Preslav Aleksandrov\textsuperscript{\affiliationA}
\And 
William F. Shen\textsuperscript{\affiliationA}
\And 
Xinchi Qiu\textsuperscript{\affiliationA}
\And 
Nicholas D. Lane\textsuperscript{\affiliationA,\affiliationB}
}
\makeatother
\addtocontents{toc}{\protect\setcounter{tocdepth}{2}} 

\begin{document}
\doparttoc
\faketableofcontents

{
\begingroup
\begin{figure}[t]
\vspace{-2.25cm}
    \quad
    \begin{subfigure}{0.1275\textwidth}
        \includegraphics[width=\textwidth]{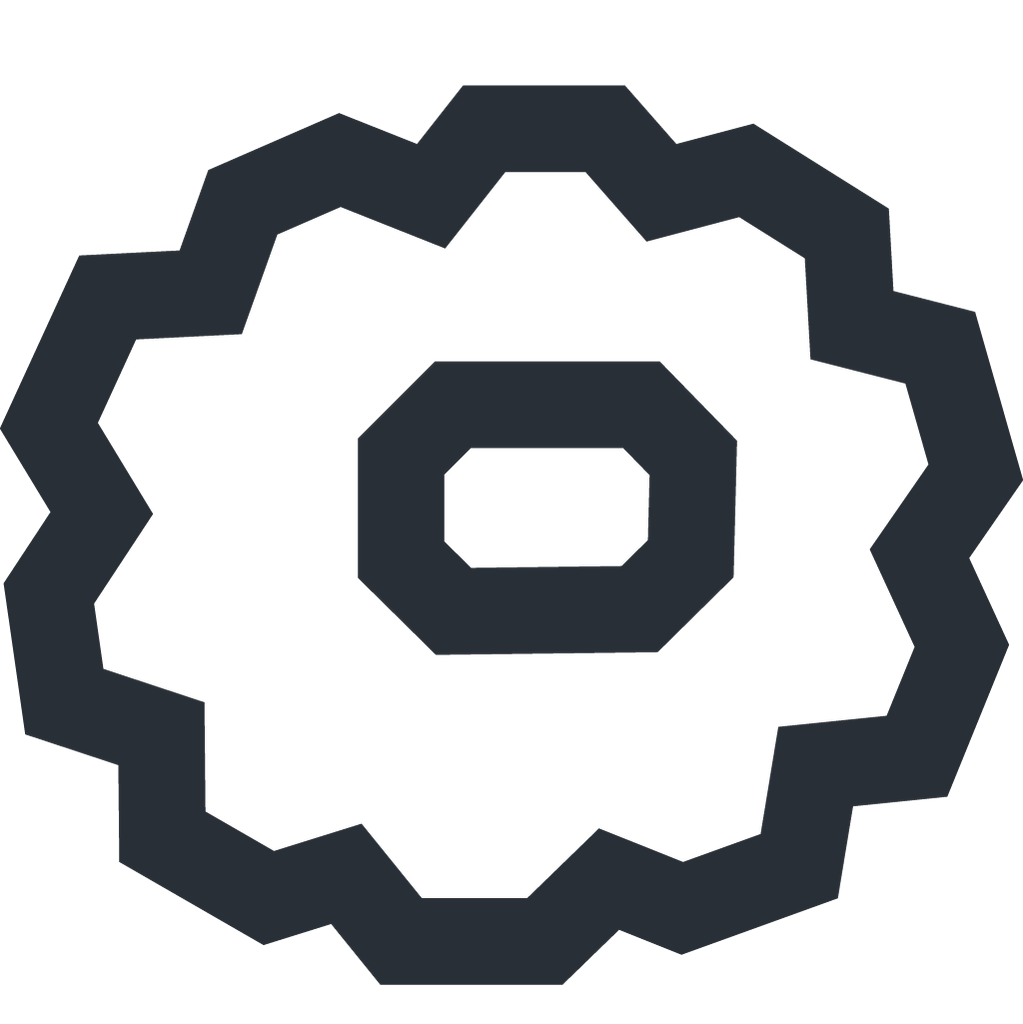}
    \end{subfigure}
    \hfill
    \begin{subfigure}{0.1\textwidth}
        \includegraphics[width=\textwidth]{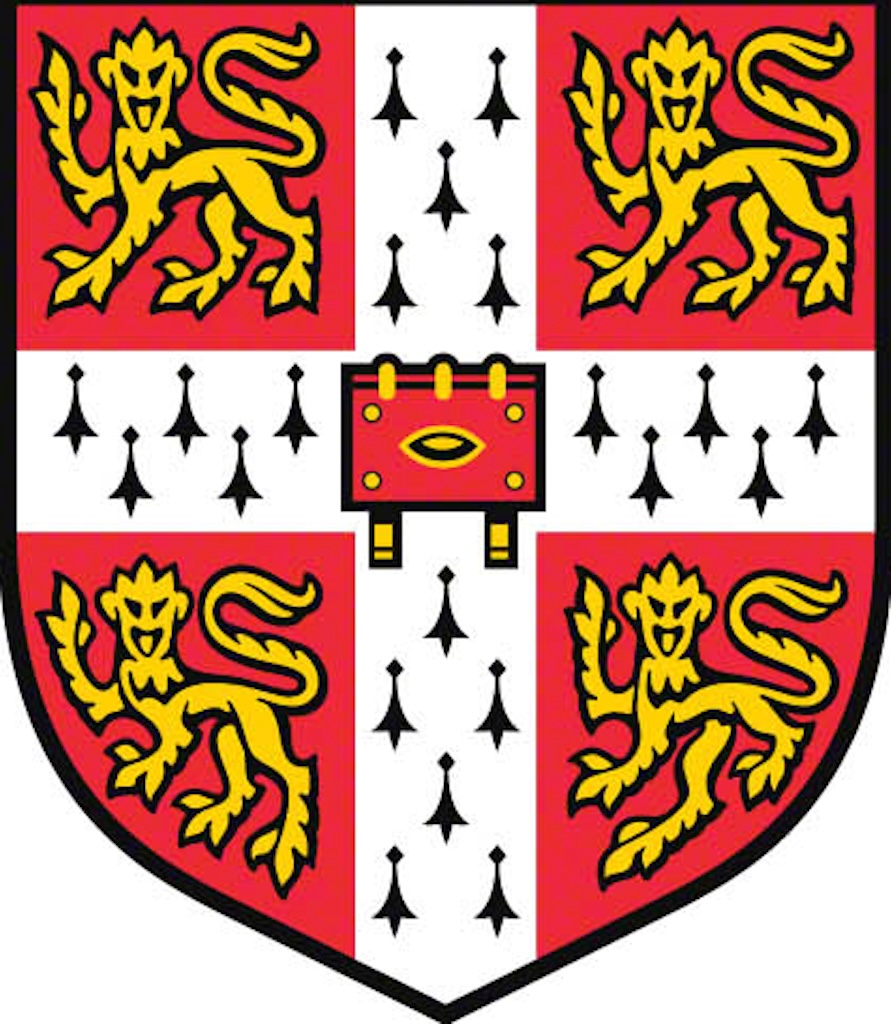}
    \end{subfigure}
    \hfill
    \begin{subfigure}{0.1275\textwidth}
        \includegraphics[width=\textwidth]{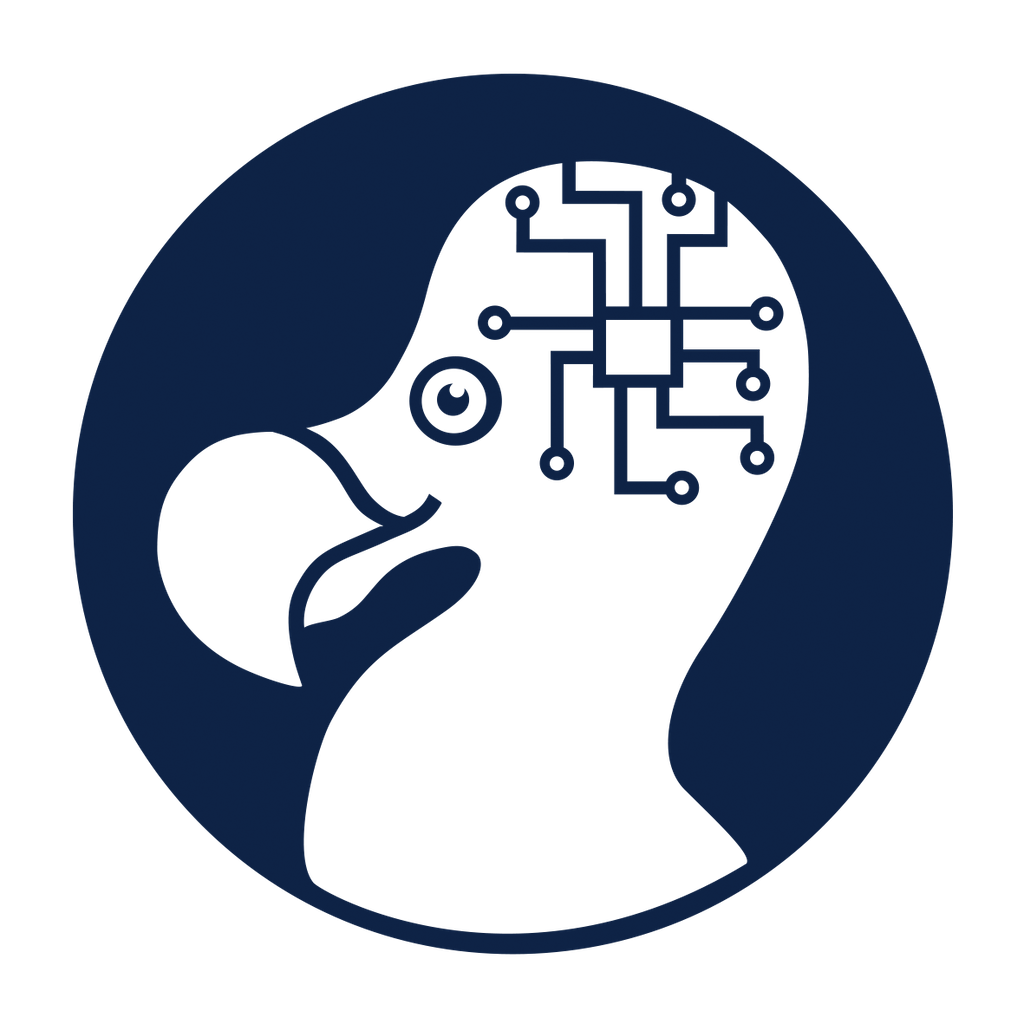}
    \end{subfigure}
    % \\
    % \vspace{-0.5cm}
    % \hrulefill
    \vspace{-0.75cm}
\end{figure}

\endgroup
}
\maketitle

\renewcommand{\thefootnote}{\myfnsymbol{footnote}}
\footnotetext{\textsuperscript{\textdagger}\href{mailto:aai30@cam.ac.uk}{\nolinkurl{aai30@cam.ac.uk}}; \textsuperscript{*} Equal contributions; \textsuperscript{1}University of Cambridge; \textsuperscript{2}Flower Labs; \textsuperscript{3}Institute of Science and Technology Austria; \textsuperscript{4}University of Warwick; \textsuperscript{5}Mohamed bin Zayed University of Artificial Intelligence}

\input{files/abstract}
\input{files/intro}

\input{files/methods}
\input{files/theory}

\input{files/experimenta_design}
\input{files/eval}
\input{files/background}
\input{files/limitations}
\input{files/conclusion}

\subsubsection*{Acknowledgments}
All costs for the computational resources used for this work were funded by Flower Labs, and the research conducted by a team of researchers from Flower Labs, The University of Cambridge, The Institute of Science and Technology of Austria, The University of Warwick, and Mohamed bin Zayed University of Artificial Intelligence. Support for university-based researchers came from a variety of sources, but in particular, the following funding organizations are acknowledged: the European Research Council (REDIAL), the Royal Academy of Engineering (DANTE), the Ministry of Education of Romania through the Credit and Scholarship Agency, and the European Union’s Horizon 2020 research and innovation programme under the Marie Skłodowska-Curie grant agreement No 101034413.

% \putbib[neurips]
% \end{bibunit}
{
\bibliographystyle{abbrvnat}
\bibliography{neurips}
}

\clearpage
\newpage

% ─────────────────────────────────────────────────────────────────────────
% Appendix with working TOC
% ─────────────────────────────────────────────────────────────────────────
\appendix
\setcounter{parttocdepth}{2}
\mtcsettitle{parttoc}{A Table of Contents}
\addtocounter{section}{1}

\addcontentsline{toc}{section}{Appendix} %

\renewcommand \thepart{} %
\renewcommand \partname{}
\part{\Large{\centerline{Appendix}}}

\parttoc

\clearpage

\input{files/appendix}

\end{document}

%% file: macro.tex
\usepackage{multirow}
\usepackage[english]{babel}
\usepackage{amssymb,amsfonts}
\usepackage{float}
\usepackage{booktabs}
\usepackage{color}
\usepackage{listings}
\usepackage{rotating}
\usepackage{pifont}
\usepackage{stmaryrd}
\usepackage{array}
\usepackage{subcaption}
\usepackage{enumitem}

\usepackage{graphicx}
\usepackage{multicol}
\usepackage{framed}
\usepackage[
	operators,
	sets,
	landau,
	probability,
	notions,
	logic,
	ff,
	mm,
	advantage,
	events,
	complexity,
	asymptotics,
	keys,
	lambda]{cryptocode}
\usepackage{adjustbox}
\usepackage{varioref}
\usepackage{hyperref}
\usepackage{breakcites}
\usepackage{placeins}
\usepackage{orcidlink}
\usepackage{tcolorbox}

\usepackage{tikz}
\usetikzlibrary{arrows.meta,decorations.pathmorphing,decorations.footprints,fadings,calc,trees,mindmap,shadows,decorations.text,patterns,positioning,shapes,matrix,fit}

\usepackage{url} %
\usepackage{framed}
\usepackage{comment}
\usepackage[nameinlink,capitalize]{cleveref}

\usepackage{pgf, tikz}
\usetikzlibrary{arrows, automata, positioning}

\floatstyle{boxed}
\definecolor{cadmiumorange}{rgb}{0.93, 0.53, 0.18}
\definecolor{emerald}{rgb}{0.31, 0.78, 0.47}
\definecolor{amaranth}{rgb}{0.9, 0.17, 0.31}
\definecolor{candypink}{rgb}{0.89, 0.44, 0.48}
\definecolor{caribbeangreen}{rgb}{0.0, 0.8, 0.6}
\definecolor{cornflowerblue}{rgb}{0.39, 0.58, 0.93}
\definecolor{limegreen}{rgb}{0.2, 0.8, 0.2}
\definecolor{mayablue}{rgb}{0.21,0.49,0.74}

\usepackage{utfsym}

\newcommand{\introtakeawaybox}[2][]{%
  \begin{tcolorbox}[%
    colback=cornflowerblue!10,
    colframe=cornflowerblue!90,
    boxrule=0.5pt,
    arc=0mm,
    left=5pt,
    right=5pt,
    top=2pt,
    bottom=2pt,
    fontupper={\fontsize{9.5pt}{11.75pt}\selectfont}
  ]
    \IfNoValueTF{#1}{\textbf{Contributions:} #2}{\textbf{Contributions #1:} #2}%
  \end{tcolorbox}%
  \vspace*{-0.1cm}
}

\newcommand{\empiricaltakeawaybox}[2][]{%
  \begin{tcolorbox}[
    colback=cornflowerblue!10, %
    colframe=cornflowerblue!90,
    float=t,
    boxrule=0.5pt,
    arc=0mm,
    left=5pt,
    right=5pt,
    top=2pt,
    bottom=2pt,
    fontupper={\fontsize{9.5pt}{11.75pt}\selectfont}
  ]
    \IfNoValueTF{#1}{\textbf{Empirical Takeaway:} #2}{\textbf{Empirical Takeaways #1:} #2} %
  \end{tcolorbox}%
  \vspace*{-0.1cm}
}

\newcommand{\takeawaybox}[2][]{%
  \begin{tcolorbox}[
    colback=cornflowerblue!10, %
    colframe=cornflowerblue!90,
    boxrule=0.5pt,
    arc=0mm,
    left=5pt,
    right=5pt,
    top=2pt,
    bottom=2pt,
    fontupper={\fontsize{9.5pt}{11.75pt}\selectfont}
  ]
    \IfNoValueTF{#1}{\textbf{Takeaway:} #2}{\textbf{#1} #2} %
  \end{tcolorbox}%
  \vspace*{-0.1cm}
}
\hypersetup{
  linkcolor  = amaranth,
    filecolor=magenta,      
    urlcolor=teal,
    citecolor=mayablue, %
    colorlinks = true,
    breaklinks = true,
    bookmarksopen=true
} 

\usepackage{epigraph}

\newcommand{\ignore}[1]{}

\newcommand{\eqdef}{\stackrel{\mathrm{def}}{=}}

\def\<{\left\langle}
\def\>{\right\rangle}
\def\[{\left[}
\def\]{\right]}
\def\({\left(}
\def\){\right)}

\newcommand{\reals}{\mathbb{R}}

\newcommand{\diag}{\mbox{\bf diag}}

\newcommand{\clip}{\textbf{clip}}

\newcommand{\sgn}{\text{sgn}}

\usepackage{amsthm}
\newtheorem{thm}{Theorem}

\newtheorem{asp}{Assumption}
\newtheorem{lemma}[thm]{Lemma}

\usepackage{amsmath,amsfonts,bm}

\def\eqref#1{equation~\ref{#1}}
\def\1{\bm{1}}

\DeclareMathAlphabet{\mathsfit}{\encodingdefault}{\sfdefault}{m}{sl}
\SetMathAlphabet{\mathsfit}{bold}{\encodingdefault}{\sfdefault}{bx}{n}

\newcommand{\E}{\mathbb{E}}

\newcommand{\R}{\mathbb{R}}

\renewcommand{\eqref}[1]{(\ref{#1})}

%% file: files/abstract.tex
\vspace{-0.5cm}
\begin{abstract}
%We address the challenge of scaling foundation model training under constrained bandwidth by extending infrequent parameter averaging (Local SGD) to adaptive optimizers while limiting the communication costs associated with synchronizing optimizer states.
Scaling foundation model training with Distributed Data Parallel~(\ddp) methods is bandwidth-limited.
Existing infrequent communication methods like \localsgd were designed to synchronize only model parameters and cannot be trivially applied to adaptive optimizers due to additional optimizer states.
Current approaches extending \localsgd either lack convergence guarantees or require synchronizing all optimizer states, tripling communication costs.
We propose \fullmethod~(\method), a family of optimizers assigning independent synchronization periods to parameters and momenta, enabling lower communication costs while preserving convergence. Through extensive experiments on language models of up to $1.7$B, we show that \method can communicate $\mathbf{170}\times$ less than \ddp and $\mathbf{2}\times$ less than the previous state-of-the-art \localadam. Furthermore, unlike previous heuristic approaches, \method is suited for practical training scenarios
% exhibiting node churn.
prone to system failures.
\method offers a scalable, bandwidth-efficient, and fault-tolerant solution for foundation model training.
%, tailored to today's large models, long training durations, and modern infrastructure.
\end{abstract}
\vspace{-0.5cm}

%\mher{Title suggestion: {\em Desynched/Decoupled Averaging Schedules in Local Adaptive Optimizers for Low-communication Training of Foundation Models}}

%% file: files/intro.tex
\section{Introduction}
\label{sec:intro}

Training foundation models requires distributing optimization across multiple workers to accommodate memory requirements and leverage additional compute. However, frequent gradient communication in standard Distributed Data Parallelism (\ddp)~\citep{PyTorchDistributed, Horovod,FSDP_ZeRO, FSDP_Pytorch} increases networking costs and limits scalability.
Early works like \localsgd~\citep{LocalSGD} and \fedavg~\citep{fedavg} reduced this overhead by synchronizing across workers \emph{infrequently}, averaging model parameters only after $K\gg1$ local steps rather than gradients every step.
However, modern foundation model training, e.g., of Large Language Models~\citep{llama3}, does not use Stochastic Gradient Descent, but rather \textbf{adaptive optimizers}~\citep{Adam,LION,LAMB,ADOPT} to scale effectively to larger batches~\citep{NoiseIsNotTheMainFactorSGDAdam}, at the expense of maintaining additional optimizer states.
 
Some extensions of \localsgd to adaptive optimizers~\citep{Photon,DiLoCo} average only model parameters; yet, this poses challenges. First, they lack convergence guarantees. Second, keeping optimizer states local~\citep{DiLoCo,DiLoCoScalingLaws,AsyncDiLoCo} accumulates noisy small-batch gradients and does not provide a means of adding new workers. This makes them unsuitable for 
% node-churn environments~\citep{NodeChurnSpotnik}--where workers from the training pool are frequently re-allocated across training jobs. 
environments prone to random system failures.
Third, re-initializing optimizer states~\citep{LLMFL,Photon,DEPT} destabilizes training by triggering loss spikes~\citep{LLMFL,Photon}.

\localadam~\citep{LocalAdam} addresses these challenges, proving periodic synchronization \emph{can} converge faster than standard \adam with \ddp, and remain \textbf{robust} to the addition of new workers.
However, it requires synchronizing optimizer states alongside model parameters, tripling communication payload size compared to \localsgd and \ddp.
Hence, our work aims to answer the following question:

\vspace{-0.25cm}
\begin{quote}
\centering
\emph{
Can independently syncing parameters and momenta improve communication efficiency for adaptive optimizers while maintaining convergence and robustness?}
\end{quote}
\vspace{-0.25cm}

As a result of our inquiry, we propose a new optimizer family, \fullmethod (\method), which sets independent synchronization frequencies for model parameters and optimizer states. This approach reduces communication overhead by synchronizing optimizer states less often. For base adaptive optimizers like \adam~\citep{Adam} and \adopt~\citep{ADOPT}, \method decouples the synchronization intervals for parameters, first momentum, and second momentum.

Empirically, we find that \method outperforms \localadam~\citep{LocalAdam} in communication efficiency by $\mathbf{2}\times$ and \ddp by $\mathbf{170}\times$ when training language models while offering several advantages:
\introtakeawaybox{
\begin{enumerate}[leftmargin=*]
\item \textbf{Provable convergence.} We prove convergence (see \cref{sec:theory}) for \method under two settings: non-convex objectives when using SGD with momentum (\texttt{SGDM}), and weakly convex objectives when using \adam. Since momentum sync frequencies appear in higher-order terms, our theory shows them to be less important. Our \adam proof assumes homogeneous losses, while our \texttt{SGDM} analysis allows heterogeneous losses typical in federated or distributed settings.
\item \textbf{Communication reduction.} Aligned with theory, we empirically show that parameters require equal or more frequent synchronization than momenta, and that less frequent momentum sync reduces communication ($\mathbf{2}\times$ vs \localadam, $\mathbf{170}\times$ vs \ddp). We demonstrate these savings persist under heterogeneous data sampling~(\cref{app:subsubsec:comm_reduction_het_data}), consistent with our \methodsgdm analysis.
\item \textbf{Scalability to large models.} We validate \method at billion-scale language model training with extended durations, demonstrating competitive \texttt{ICL} performance against both \localadam and \ddp.
% \item \textbf{Resilience to data heterogeneity.} We show that \method maintains convergence even when training with heterogeneous data across workers, without increasing communication frequency.
\item \textbf{Hardware robustness.} Unlike previous heuristic methods~\citep{DiLoCo,Photon}, \method avoids persistent local states, enabling it to seamlessly integrate new workers during batch-size scheduling~\citep{DontDecayLearningRateIncreaseBatchSize,llama3} or
% node churn environments where workers are \textbf{transiently} allocated to the training job.
to support environments prone to random system failures.
\end{enumerate}
}

Given the shift toward larger models and extended pre-training~\citep{SmolLM2,llama3} far beyond compute-optimal token counts~\citep{TrainingComputeOptimalLLMs}, \method's convergence guarantees, reduced communication, and strong long-horizon performance make it a compelling replacement for \ddp, enabling efficient scaling across geographically distributed data centers without additional communication infrastructure.

%% file: files/methods.tex
\section{\fullmethod~(\method)}\label{sec:methods}

%Before introducing \fullmethod, it is crucial to characterize the relation between the rate of change of optimizer states and Local \adam. 
We start by characterizing the relation between the rate of change of optimizer states and \localadam, and how these can be leveraged to lower the communication cost. 
Consider the \adam update:
\begin{align}
    u_t &= \beta_1 u_{t-1} + (1-\beta_1)g_t,\\
    v_t &= \beta_2 v_{t-1} + (1-\beta_2) g_t \odot g_t.
\end{align}\label{eq:adam_update}%
For \localadam, convergence is contingent on $\beta_2$ satisfying $1-\beta_2=\widetilde{\mathcal O}\!\bigl(K^{-3/2}R^{-1/2}\bigr)$~\citep{LocalAdam} 
% \begin{equation}
% 1-\beta_2=\widetilde{\mathcal O}\!\bigl(K^{-3/2}R^{-1/2}\bigr),
% \end{equation}
where $K$ is the number of local steps and $R$ the total communication rounds.
Large $K$ or $R$, typical in foundation model training~\citep{Photon}, implies $\beta_2\!\to\!1$, and conversely larger $\beta_2$ permits higher $K$ or $R$. 

A useful summary measure is the number of steps until a state's weight decays to a fraction $\psi$, $\tau_{\psi}(\beta)=\frac{\ln \psi}{\ln\beta}$.
% \begin{equation}
% \label{eq:half_life}
% \tau_{\psi}(\beta)=\frac{\ln \psi}{\ln\beta}.
% \end{equation}
Following \citet{AdemaMix}, we use the half-life $\tau_{0.5}$ as our primary measure, omitting $\beta$ when clear. For typical values of $\beta$, we have $\tau_{0.5}(0.95)\approx13.5$~\citep{SmolLM2}, $\tau_{0.5}(0.999)\approx692.8$~\citep{Adam}, and $\tau_{0.5}(0.9999)\approx6931$~\citep{ADOPT}. Intuitively, larger half-lives imply synchronizing gradients over longer horizons as the optimizer is less sensitive to new gradients; choosing $\beta=0$ ignores all previous momenta, whereas $\beta \to 1$ progressively attenuates signal from the current gradient.

While the half-life captures the horizon for which an optimizer state remains relevant to model updates, it provides no information on its absolute rate of change.
With coordinate-wise clipping, each gradient component satisfies $\lvert(g_t)_i\rvert \le \rho$.
Unrolling \adam's recursions for $K$ local steps gives:
\begin{align}
    u_{t+K} &= \beta_1^{K}u_t
                +(1-\beta_1)\!\sum_{k=0}^{K-1}\beta_1^{k}\,g_{t+K-1-k}, \\
    v_{t+K} &= \beta_2^{K}v_t
                +(1-\beta_2)\!\sum_{k=0}^{K-1}\beta_2^{k}\,
                   \bigl(g_{t+K-1-k}\odot g_{t+K-1-k}\bigr).
\end{align}\label{eq:adam_update_unrolled}
Since $\lvert g_{t,i}\rvert\le \rho$ and $\lvert(g_t\odot g_t)_i\rvert\le \rho^{2}$, the maximal $\ell_{\infty}$ drift of each moment is (see \cref{app:derivation_max_chance}):  
\begin{align}
\label{eq:abslute change_u}
    \bigl\|u_{t+K}-u_t\bigr\|_{\infty} &\le 2
        \rho\,\bigl(1-\beta_1^{K}\bigr),
\end{align}
\vspace{-1.5em}
\begin{align}
\label{eq:abslute change_v}
    \bigl\|v_{t+K}-v_t\bigr\|_{\infty} &\le 2
        \rho^{2}\,\bigl(1-\beta_2^{K}\bigr). 
\end{align}
From the above, large $\beta$ values and small clip bounds $\rho$, a common practice in foundation model training~\citep{gpt3,BLOOM}, limit the absolute changes in optimizer states. We can construct similar reasoning for other optimizers~\citep{NesterovIlya,ADOPT}, and norm-based clipping~\citep{OnTheDifficultyOfTrainingLanguageModels,gpt3}. From the above, the half-life of an optimizer state should inform its synchronization frequency. For example, if $\tau_{0.5}(0.95)\approx13.5$ and $K=256$, synchronization only affects few initial local steps. Over the course of the local training, the impact of the synchronised optimizer state shall decay to $0$ given Equations \ref{eq:abslute change_u} and \ref{eq:abslute change_v}. Conversely, if $K=16$, synchronization approximately matches the half-life, strongly influencing local updates.

\subsection{\method Algorithm}\label{subsec:method_algorithm}
\input{algorithms/desync_opt_generic}
Motivated by the above insights, we formalize \fullmethod as a family of optimizers offering the same convergence and robustness as \localadam but with significantly lower communication costs. Our approach applies generically to adaptive optimizers parameterized by $\opt : (\mathbb{R}^{d},\mathbb{R}^{d},\mathbb{R}_{>0},\{\mathbb{R}^{d}\})\!\to\!\mathbb{R}^{d}$, with $N$ optimizer states $\{s^{j}_{-1}\}_{j=1}^{N}\subset \mathbb{R}^{d}$, each updated by $\upd^{j} : (\mathbb{R}^{d},\mathbb{R}^{d})\!\to\!\mathbb{R}^{d}$. Coordinate-wise clipping is defined as $[\clip(X,\rho)]_i = \sgn([X]_i)\cdot\min\{|X_i|, \rho\}$. We focus our analysis on \texttt{SGDM} and \adam.

As shown in \cref{alg:desync_generic}, \method synchronizes parameters $x \in \mathbb{R}^{d}$ and optimizer states $\{s^{j}\}_{j=1}^{N}$ at state-specific intervals $K_x,\{K_j\}_{j=1}^{N}\!\in\!\mathbb{N}_{+}$. Setting $N=2$, $s^1_t=u_t$, $s^2_t=v_t$, and using update rules $\upd^{1},\upd^{2}$ from Eq.~\ref{eq:adam_update} yields \methodadam~(see \cref{alg:desync_adam}). 

\takeawaybox[Toy Example]{%
To highlight \method's practical benefit, \cref{fig:toy_example_iid} illustrates a scenario where \method and \localadam converge under noisy gradients, while prior heuristic methods~\citep{DiLoCo,Photon,DEPT,LLMFL} fail.
}

\begin{figure}
    \centering
\noindent\subfloat[Distance to the optimum.]
        {\includegraphics[height=4.25cm, width=0.485\columnwidth]{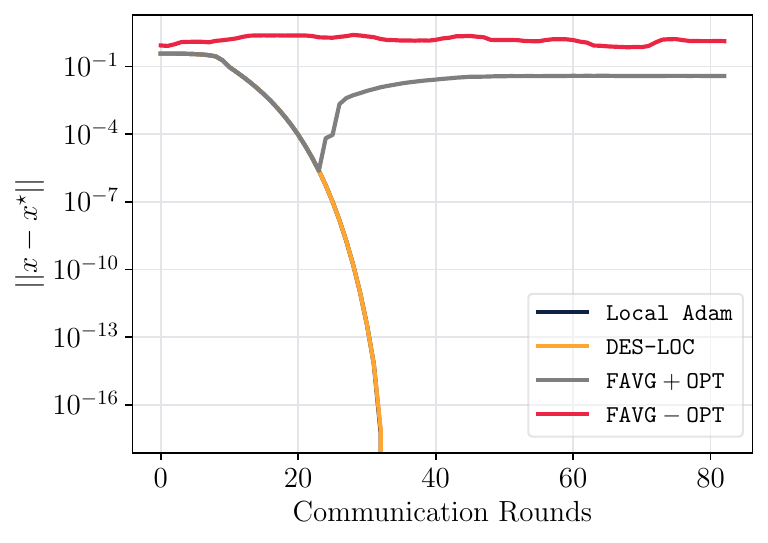}\label{fig:toy_distances_iid}}  \hfill
    \noindent\subfloat[2-D contour.]
    {\includegraphics[height=4.25cm, width=0.485\columnwidth]{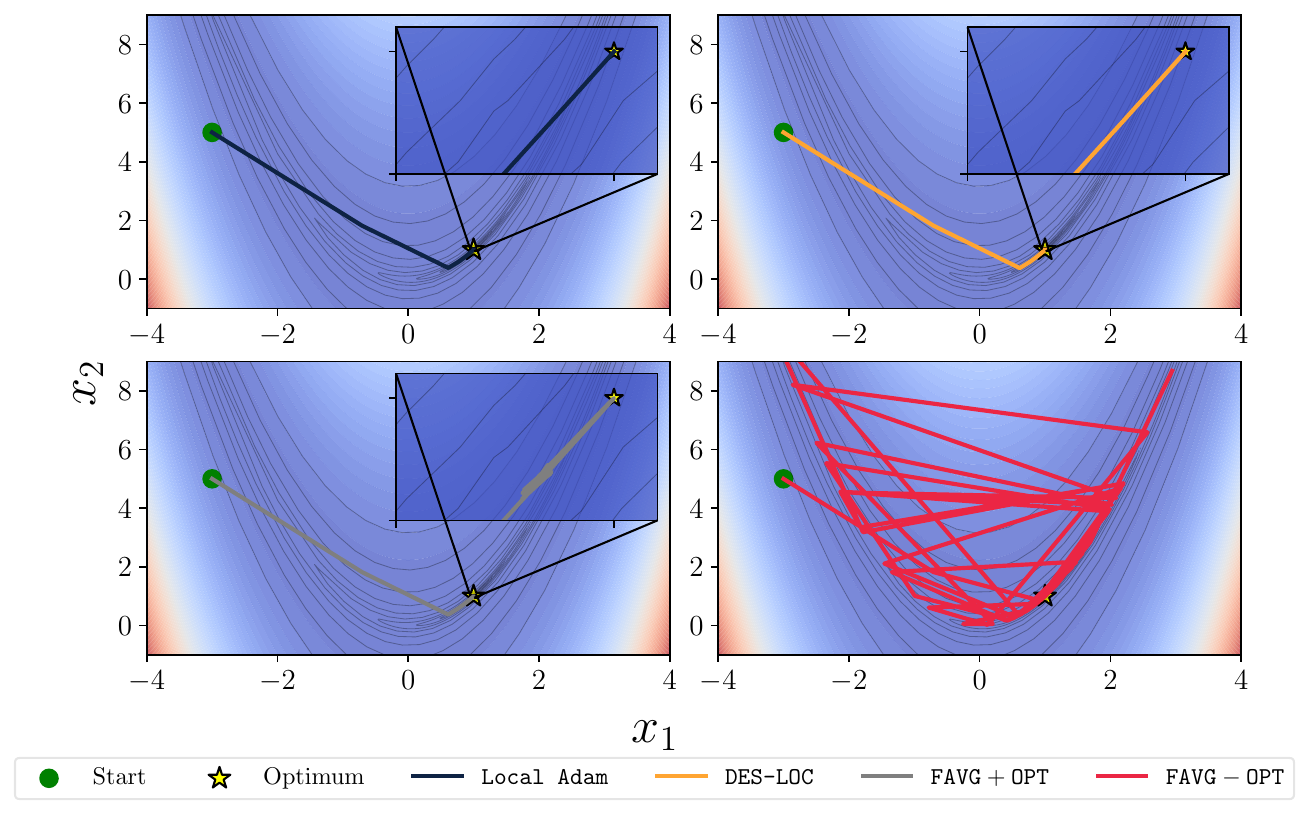}} \hfill
    \caption{We present a toy problem where \method ($K_x=192,K_u=192,K_v=692$) and \localadam ($K=K_x$) both converge to the optimum (overlapping in \cref{fig:toy_distances_iid}). Methods keeping optimizer states local \colorbox{grey_plot}{} ~\citep{DiLoCo,Photon} fail, causing oscillations without convergence. Periodically resetting states \colorbox{red_plot}{}~\citep{LLMFL,DEPT} similarly stalls due to repeated oscillations. We optimize the non-convex function $f(x_1,x_2)=(1 - x_1)^2+100(x_2 - x_1^2)^2$ with $M=256$ workers and IID Gaussian noise ($\sigma=1.5$).
    }
    \label{fig:toy_example_iid}
\end{figure}

%% file: algorithms/desync_opt_generic.tex
\begin{algorithm}
\caption{\method}
\label{alg:desync_generic}
\footnotesize
\begin{onehalfspace}
\begin{algorithmic}[1]
% ---------------- REQUIRE ------------------------------------------------
  \Require \textbf{Model tensors, update functions, hyper-parameters} \\
          \quad $x_0 \in \mathbb{R}^{d}$, $\{s^{j}_{-1}\}_{j=1}^{N} \in (\mathbb{R}^d)^N$ — initial parameter vector, the initial $N$ optimizer states\\
          \quad $\{\upd^{j}\}_{j=1}^N:(\mathbb{R}^d\times\mathbb{R}^d\to\mathbb{R}^d)^N$ — updates optimizer state $j$ from its previous state and the gradient.\\
          \quad $\opt:\mathbb{R}^d\times\mathbb{R}^d\times\mathbb{R}_{+}\times(\mathbb{R}^d)^N\to\mathbb{R}^d$ — update params from the gradient, lr, and optimizer states. \\
          \quad $\rho \in \mathbb{R}_{+}$, $\{\eta_t\}_{t=0}^{T-1} \in (\mathbb{R}_{+})^{T-1}$ — clipping radius for $\clip(\cdot,\rho)$, learning-rate for each time-step\\      
          \quad $T,M \in \mathbb{N}_{+}$ — total optimization steps and number of workers\\
           \quad \textcolor{blue}{$K_x \in \mathbb{N}_{+},\,\{K_j\}_{j=1}^{N} \in (\mathbb{N}_{+})^N$} — communication periods (steps) 
  \Ensure $x_T,\;\{s_{T-1}^{j}\}_{j=1}^{N}$

% ------------------------------------------------------------------------
  \State \textbf{for each worker} $m$: $x_0^m \gets x_0, s_{-1}^{j,m} \gets s_{-1}^j$ 
         \hfill\textcolor{gray}{\scriptsize local init}
  \For{$t = 0,\dots,T-1$} \hfill\textcolor{gray}{\scriptsize training loop}
    \ForAll{workers $m=0,\dots,M-1$ \textbf{in parallel}}
% -------- gradient -------------------------------------------------------
      \State $g_t^m \gets \nabla F(x_t^m;\xi_t^m)$
             \hfill\textcolor{gray}{\scriptsize stochastic grad}
      \State $\widehat{g}_t^m \gets \clip(g_t^m,\rho)$
             \hfill\textcolor{gray}{\scriptsize per-coordinate clipping}
% -------- per–state updates ----------------------------------------------
      \For{$j = 1$ \textbf{to} $N$}
        \If{$\textcolor{blue}{t \bmod K_j = 0}$}
             \hfill\textcolor{blue}{\scriptsize sync $s^j$}
          \State $s_t^{j,m} \gets
                 \upd^{j}\,\!\bigl(\mathbb{E}_m[s_{t-1}^{j,m}],\;
                                    \widehat{g}_t^m\bigr)$
        \Else
          \State $s_t^{j,m} \gets
                 \upd^{j}\,\!\bigl(s_{t-1}^{j,m},\;
                                    \widehat{g}_t^m\bigr)$
        \EndIf
      \EndFor
% -------- parameter update -----------------------------------------------
      \If{$\textcolor{blue}{t \bmod K_x = 0}$}
             \hfill\textcolor{blue}{\scriptsize sync $x$}
        \State $x_{t+1}^m \gets
               \opt\,\!\bigl(\mathbb{E}_m[x_t^{\,m}],\;
                              \widehat{g}_t^m,\eta_t,\{s_t^{j,m}\}_{j=1}^{N}\bigr)$
      \Else
        \State $x_{t+1}^m \gets
               \opt\,\!\bigl(x_t^{\,m},\;
                              \widehat{g}_t^m,\eta_t,\{s_t^{j,m}\}_{j=1}^{N}\bigr)$
      \EndIf
    \EndFor
  \EndFor
\end{algorithmic}
\end{onehalfspace}
\end{algorithm}

%% file: files/theory.tex
\section{Convergence Guarantees for \method}\label{sec:theory}

In this section, we provide theoretical support for the proposed \method approach and demonstrate that synchronizing optimizer states is less critical to overall convergence than model averaging. To keep the presentation concise, we focus on a version of the Adam optimizer that uses only a single momentum state (i.e., SGD with momentum). Extensions to the full Adam optimizer with both momentum states can be carried out using analysis techniques from \cite{li2022distributed} for convergence in expectation, and from \cite{LocalAdam} for high-probability guarantees; we provide an informal result here and defer all detailed proofs and technical discussions to the appendix.

% \subsection{\methodsgdm}

Formally, we consider the following optimization problem:
\begin{equation}\label{opt-problem}
\textstyle
    \min_{x\in \reals^d} f(x):= \frac{1}{M}\sum_{m=1}^M f_m(x),
    \quad\text{with}\quad f_m(x) = \mathbb{E}_{\xi\sim\mathcal{D}_m} [F_m(x;\xi)].
\end{equation}

In this setup, all $M$ machines collaboratively minimize the objective in \eqref{opt-problem}. Generally, we assume each machine $m$ has access to only dataset $\mathcal{D}_m$, which can differ from device to device. This recovers the homogeneous distribution case when all machines have the same dataset $\mathcal{D}_1=\mathcal{D}_2=\dots=\mathcal{D}_M$ and minimize the same loss $f_1(x)=f_2(x)=\dots=f_m(x)=f(x)$. As in practice, we assume each machine $m$ computes mini-batch stochastic gradients corresponding to randomly selected samples $\xi\sim\mathcal{D}_m$ from dataset $\mathcal{D}_m$. To derive convergence bounds, we further assume the following standard technical assumptions on the problem structure and stochastic gradients.

\begin{asp}[Lower bound  and smoothness]\label{ass:smooth}
    The overall loss function $f\colon\R^d\to\R$ is lower bounded by some $f^{*} \in \mathbb{R}$ and all local loss functions $f_m$ are $L$-smooth:
    $$\|\nabla f_m(x) - \nabla f_m(y)\| \leq L \|x-y\|, \quad \text{for any } x,y\in\R^d.$$ 
\end{asp}

\begin{asp}[Unbiased noise with bounded stochastic variance]\label{ass:boundgrad}
    The stochastic gradient $g^m$ of local loss function $f_m$ computed by machine $m$ is unbiased and the noise has bounded variance:
    $$\E[g^m] = \nabla f_m(x), \quad \E[\|g_t^m - \nabla f_m(x)\|^2] \le \sigma^2, \quad \text{for any } x\in\R^d.$$
\end{asp}

\begin{asp}[Bounded heterogeneity]\label{ass:het}
    For any $x\in\R^d$, the heterogeneity is bounded by
    $$\textstyle\frac{1}{M}\sum_{m=1}^M\|\nabla f_m(x) \|^2 \le G^2 + B^2\|\nabla f(x)\|^2.$$
\end{asp}

All three assumptions are standard and widely used in the convergence analysis of optimization algorithms \cite{Yu2019,pmlr-v119-karimireddy20a,wang2021fieldguidefederatedoptimization,Yuan2022}. Note that the bounded heterogeneity condition recovers the homogeneous case when $G^2=0$ and $B^2=1$.
To facilitate the technical presentation of the analysis, we view model and optimizer state synchronizations through assigning probabilities to each averaging event. Particularly, instead of averaging model parameters every $K_x$ steps (i.e., $\textcolor{blue}{t \bmod K_x = 0}$), we average \textcolor{blue}{with probability $p_x = \frac{1}{K_x}$}, which are statistically equivalent. In the following theorem, we provide convergence rate of SGDM optimizer under such probabilistic and decoupled synchronization:

\begin{thm}
Let Assumptions \ref{ass:smooth}, \ref{ass:boundgrad} and \ref{ass:het} hold. Then, choosing the step size $\eta = \min(\eta_0, \frac{1}{\sqrt{T}})$ with
\begin{equation}\label{eta-psi}
\textstyle
\eta_0 \eqdef \frac{1}{4L}\min\left(1-\beta, \frac{1}{6\sqrt{\psi\max(1,B^2-1)}} \right),
\quad\text{where}\quad
\psi \eqdef \frac{4(1-p_x)}{p_x^2} \cdot \frac{(1-\beta)(1-p_u)}{1-(1-p_u)\beta},
\end{equation}
the average iterates $x_t = \E_m[x_t^m]$ of \methodsgdm converge with the following rate:
\begin{equation}\label{rate-sgdm}
\textstyle
\frac{1}{T}\sum_{t=0}^{T-1}\E{\|\nabla f(x_t)\|^2}
\le \frac{4}{\sqrt{T}}\left(f(x_0) - f^*
+ \frac{L\sigma^2 }{2M} \right)
+ \mathcal{O}\left(\frac{1+\psi}{T}\right).
\end{equation}
\end{thm}

We now discuss the convergence result and its implications.
First, the obtained rate \eqref{rate-sgdm} is asymptotically optimal for this setup \citep{arjevani2023lowerbound}. Notably, the leading term $\mathcal{O}(\frac{1}{\sqrt{T}})$ is unaffected by the number of local steps or by the decoupled synchronization approach we propose. Interestingly, probabilities $p_x$, $p_u$, and the momentum parameter $\beta$ appear only in the higher-order term $\mathcal{O}(\frac{1}{T})$, and thus have a limited impact on the convergence speed. In particular, setting $p_x = 1$ and $p_u = 0$ (which implies $\psi = 0$) recovers standard mini-batch SGDM and its corresponding convergence rate \cite{Liu2020}.

Regarding the relative importance of model and optimizer state synchronization steps, it is evident from \eqref{eta-psi} that model synchronization has a greater impact on convergence due to the dependence $\psi = \mathcal{O}(\frac{1}{p_x^2})$. Moreover, momentum averaging can be turned off entirely ($p_u = 0$) without affecting the asymptotic behavior of the rate. Clearly, the same is not true for model averaging: with vanishing $p_x$, the $\psi$ term becomes unbounded and breaks the rate. However, since the $\psi$ term also appears in the step-size restriction \eqref{eta-psi}, increasing the frequency $p_u$ of momentum averaging---while not changing the asymptotic rate---allows for a larger step size in theory, potentially leading to faster convergence in practice. Overall, the obtained theory justifies the hypothesis that momentum states can be synchronized less frequently than the model parameters and that more averaging improves convergence through supporting larger step sizes.

For \methodadam, we generalize the convergence result of \citep{LocalAdam} as follows,

\begin{thm} [Informal]\label{thm:adam_desync}
Let $K_{\mathrm{lcm}} = \mathrm{lcm}\{K_x,K_u,K_\upsilon\}$\footnote{Least common multiple.}, $f$ be weakly-convex and the same assumptions as in Theorem 3 of \citep{LocalAdam},  then with probability $\geq 1-\delta$, \method yields,

\begin{equation}
\frac{1}{K_{\mathrm{lcm}} R} \sum^{R-1}_{r=0} \sum^{K_{\mathrm{lcm}}-1}_{k=0} \|\nabla f(\bar{z}_{r,k})\|^2  = \tilde{\mathcal{O}}\left(  \sqrt{\frac{L\Delta\sigma^2}{MK_{\mathrm{lcm}}R}} \right),
\end{equation}
where $\bar{z}_{r,k}$ is an auxiliary sequence of $x_{r,k}$, and $\sigma$ bounds stochastic noise (see Appendix for details).
\end{thm}
Theorem \ref{thm:adam_desync} generalizes the convergence result of local Adam \cite{LocalAdam} to the case of different state-specific intervals $K_x,K_u,K_\upsilon$.
A full version of Theorem \ref{thm:adam_desync} is provided in the Appendix.

%% file: files/experimenta_design.tex
\section{Experimental Design}
\label{sec:exp_design}

Our experimental setup addresses the following research questions:

\begin{itemize}[noitemsep,topsep=0pt,parsep=2pt,partopsep=0pt]
    \item[\textbf{RQ1}] Do \emph{theoretical} rates of change predict the \emph{empirical} evolution of optimizer states?
    \item[\textbf{RQ2}] How does the synchronization frequency of a model/optimizer state impact performance?
    \item[\textbf{RQ3}] To what extent can \method cut communication w.r.t. \localadam in practical scenarios?
    \item[\textbf{RQ4}] How does \method scale with increasing model size and longer training horizons?
    %\item[\textbf{RQ5}] Can \method cope with node churn environments/changes in the number of workers?a
\end{itemize}

% We summarize the main setup, \cref{app:toc} contains the implementation details and hyper-params.

\subsection{Experimental Setup}\label{sec:exp_setup}

\textbf{Models and data.} Unless noted, we train a $135$M-parameter \texttt{GPT}-style model (see \cref{tab:model_architectures}) with sequence length $2048$. Following \citet{Photon}, we distinguish worker batch size $\mathcal{B}_w$ from global batch size $\mathcal{B}=\sum_{w=0}^{M-1}\mathcal{B}_w$. By default, we evenly split a global batch of $2$M tokens across $M=4$ workers, sampling \texttt{IID} from \texttt{SmolLM2}~\citep{SmolLM2}: $70\%$ \texttt{Fineweb-Edu}~\citep{FineWeb}, $10\%$ \texttt{Cosmopedia}~\citep{Cosmopedia}, $10\%$ \texttt{Python-Edu}, $5\%$ \texttt{FineMath 4+}, and $5\%$ \texttt{Infi-WebMath 4+}. The $135$M model trains for $6.4$B tokens ($2.4\times$ compute-optimal~\citep{TrainingComputeOptimalLLMs}). For \textbf{RQ4}, we scale to $1.7$B for $40$B tokens ($2\times$ compute-optimal) following recent practice~\citep{llama3,BeyondChinchilla,SmolLM2}. In heterogeneous experiments, each worker samples one dataset component except the \texttt{Fineweb-Edu} worker, which samples the \texttt{SmolLM2} mixture.

\textbf{Optimizers.} We use \adam~\citep{Adam} and its problem-independent variant \adopt~\citep{ADOPT}. By modifying the second-moment update, \adopt guarantees optimal-rate convergence for any $\beta_2$ and stabilizes small per-worker batches without altering \adam's core properties. For the $135$M-parameter experiments, we grid-search $(\beta_1,\beta_2,\eta)$ under \ddp; the $1.7$B model adopts hyperparameters from~\citet{SmolLM2,ADOPT}. Learning rates follow the warmup-stable-decay (\texttt{WSD}) schedule~\citep{BeyondFixedTrainingDuration,SmolLM2}. We favor \adopt with default $\beta_2=0.9999$ in high-$\beta$ regimes where \adam is often unstable.

\textbf{Baselines.} We compare \method with: (i) fully synchronous \adam/\adopt via \ddp; (ii) Local \adam/\adopt; (iii) \texttt{FedAvg}/\localsgd persistently keeping optimizer states~\citep{Photon,DiLoCo}, which we call \texttt{FAVG}$+$\texttt{OPT}; and (iv) \texttt{FedAvg} resetting optimizer states~\citep{LLMFL,DEPT}, which we call \texttt{FAVG}$-$\texttt{OPT};. Persistent-state \texttt{FedAvg} corresponds to \method with infinite state sync periods ($K_u,K_v=\infty$), providing an upper bound on communication efficiency. We expect \ddp to serve as an upper bound on performance for the machine-learning objective. When discussing hardware robustness, we are concerned with
% \textbf{node churn} environments, which we defined as training environments where workers are repeatedly re-allocated and thus cannot be assumed to be persistently dedicated to a training job.
environments prone to systems failures and the repeated re-allocation of workers.

\textbf{Metrics.} We evaluate \method and baselines by (i) perplexity and (ii) per-worker asymptotic communication cost assuming a bandwidth-optimal \texttt{Ring-AllReduce}~\citep{Horovod} algorithm scaling linearly with model size. For the $1.7$B model, we report standard in-context-learning (\texttt{ICL}) benchmarks~\citep{gpt3} as they become discriminative at larger scales, we use a zero-shot setting for \texttt{ICL} tasks unless stated otherwise following \citet{SmolLM2} and report the best performing communication-efficient method in \textcolor{blue}{blue} with the best-performing overall in \textbf{bold}. To fairly compare optimizer-state changes across decay rates, we measure their \emph{relative} rates of change as $\|s_{t+K}-s_t\|_2/\|s_t\|_2$. For convergence plot comparisons, we report metric means and standard deviations computed over the last round~(shown next to labels). In addition, we provide in the supplementary materials an analysis on the wall-clock time benefits of our approach compared to the baseline, along with our system modeling.

% \ls{I've been asked to add that we provide "wall-clock estimates and modeling in the appendix". I think this is the best place to have such statements. Follows a draft that can be just copied in: "In addition, we provide in the supplementary materials an analysis on the wall-clock time benefits of our approach compared to the baseline, along with the system modeling we adopt."}

%% file: files/eval.tex
\section{Evaluation}\label{sec:evaluation}

Our results show optimizer states change at different rates~(\cref{subsect:eval:rates_of_change}), forming a clear synchronization hierarchy~(\cref{eval:subsec:relative_importance}). \method reduces communication $2\times$ vs. \localadam~(\cref{eval:subsec:baseline_comparison}) while converging robustly
% under node churn 
% with the addition of new workers
with adding workers
and scaling effectively to large models~(\cref{eval:subsec:larger_models}).

\subsection{Higher $\beta$ Optimizer States Have Slower Empirical Rates of Change~(RQ1)}\label{subsect:eval:rates_of_change}

\Cref{fig:eval:rates_of_change_adopt} shows that relative rates of change for the two momenta in \local \adopt/\adam scale with their decay rates under gradient clipping ($\rho=1$). Supported by our theoretical discussion on momenta half-lives (\cref{sec:methods}), the second momentum evolves substantially slower than the first at high-$\beta_2$. For \localadam, the second momentum remains slower even when $\beta_2\approx\beta_1$, potentially because gradient variance~\citep{Adam} evolves slower than the mean direction~(first momentum).
\begin{figure}
    \centering
    \noindent\subfloat[First momentum change rate for \localadopt.]
    {\includegraphics[width=0.485\columnwidth]{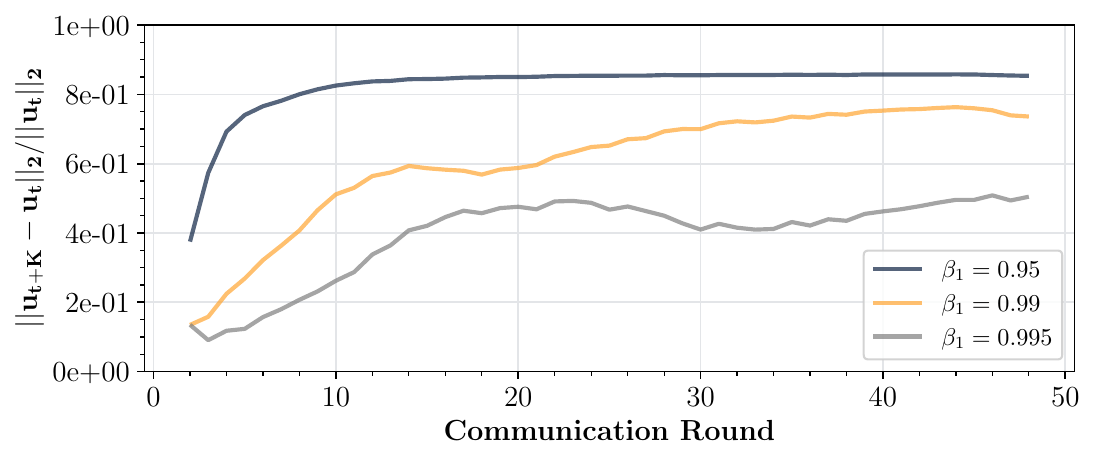}}  \hfill
    \noindent\subfloat[Second momentum change rate for \localadopt.]
    {\includegraphics[width=0.485\columnwidth]{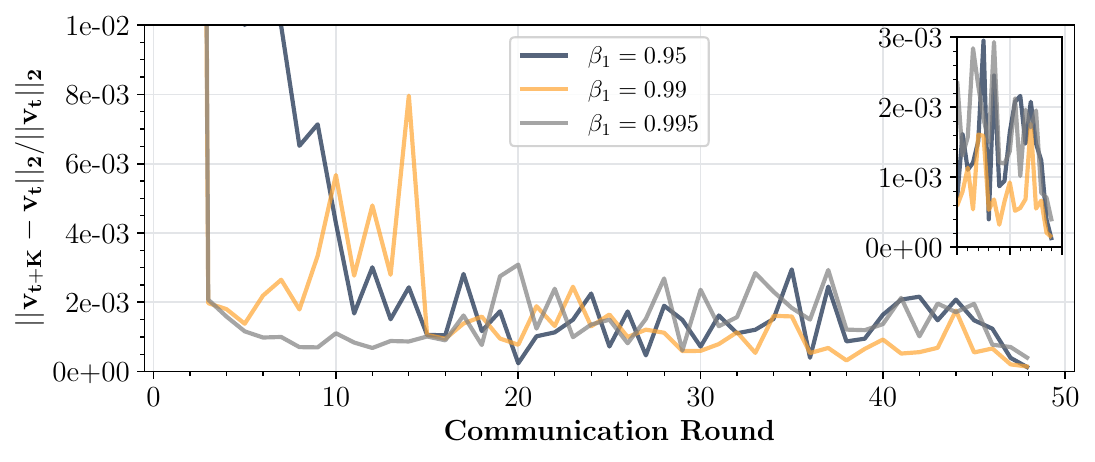}} 
    \hfill
    \noindent\subfloat[First momentum change rate for \localadam.]
    {\includegraphics[width=0.485\columnwidth]{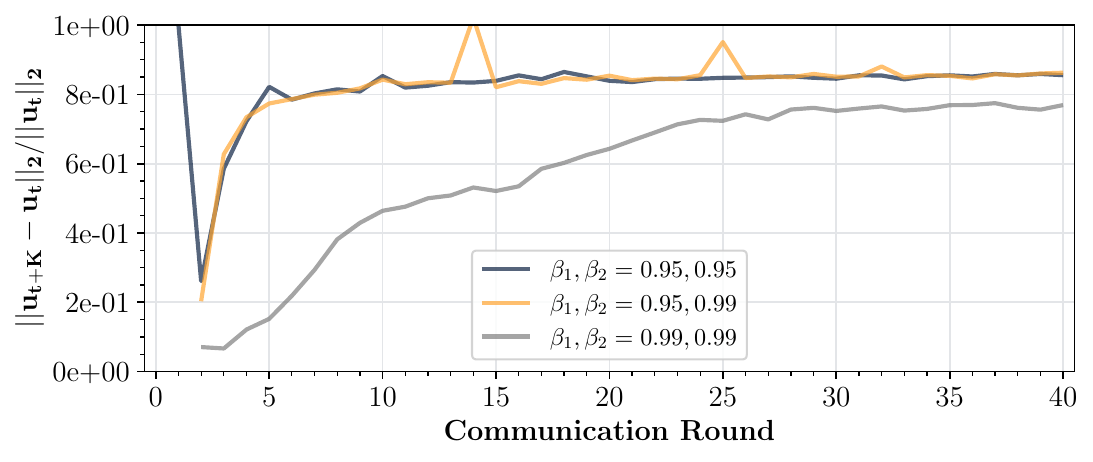}}  \hfill
    \noindent\subfloat[Second momentum change rate for \localadam.]
    {\includegraphics[width=0.485\columnwidth]{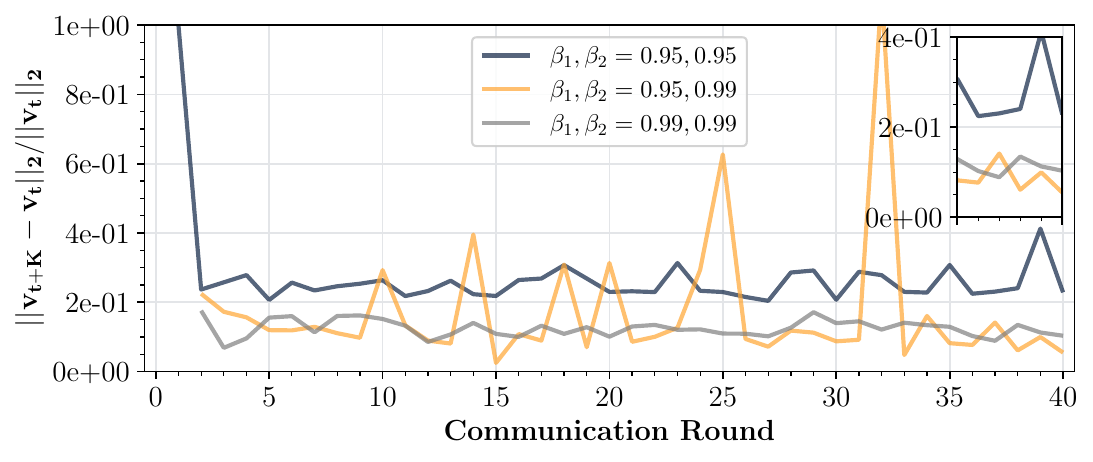}} \hfill
    \caption{Relative rates of change for first and second momenta across rounds using standard \texttt{Local} \adopt/\adam ($K=64$). For \adopt ($\beta_2=0.9999$), increasing $\beta_1\geq0.99$ greatly slows the first-momentum rate of change. The second momentum evolves $\sim100\times$ slower (note y-axis is in log scale), consistent with their decay rates and half-lives. For \adam, higher $\beta_1,\beta_2$ slow both momenta.
    }
    \label{fig:eval:rates_of_change_adopt}

\end{figure}

\takeawaybox[Takeaway:]{%
As discussed in  \cref{sec:methods,sec:theory}, when $\beta_1 \ll \beta_2$, the second momentum evolves slower than the first, proportional to half-life ratio of the two $\frac{\tau_{0.5}(\beta_2)}{\tau_{0.5}(\beta_1)}= \frac{\ln(\beta_1)}{\ln(\beta_2)}$. %This validates the theoretical basis of \method and suggests practical benefits from distinct synchronization frequencies.
}

\subsection{Parameters Require Frequent Sync, Momenta Sync Proportional to $\beta$~(RQ2)}\label{eval:subsec:relative_importance}

\Cref{fig:eval:independent_sync_frequencies} evaluates the effect of independently varying synchronization periods ($K_x,K_u,K_v$) for parameters and optimizer states. We consider two baseline periods ($K_b=16,256$), chosen based on the fastest state's half-life ($\tau_{0.5}(0.95)\approx13.5$). Frequent parameter synchronization ($K_x$) is crucial for performance, while synchronizing momenta ($K_u,K_v$) significantly impacts training only if their half-lives align with the base frequency $K_b$. Otherwise, synchronization frequency primarily influences communication costs rather than model quality. \adam results can be seen in \cref{app:additional_results}.

\begin{figure}[H]
    \centering
    \noindent\subfloat[\method vary $K_x$ , fixed $K_u=K_v=256$]
    {\includegraphics[width=0.485\columnwidth]{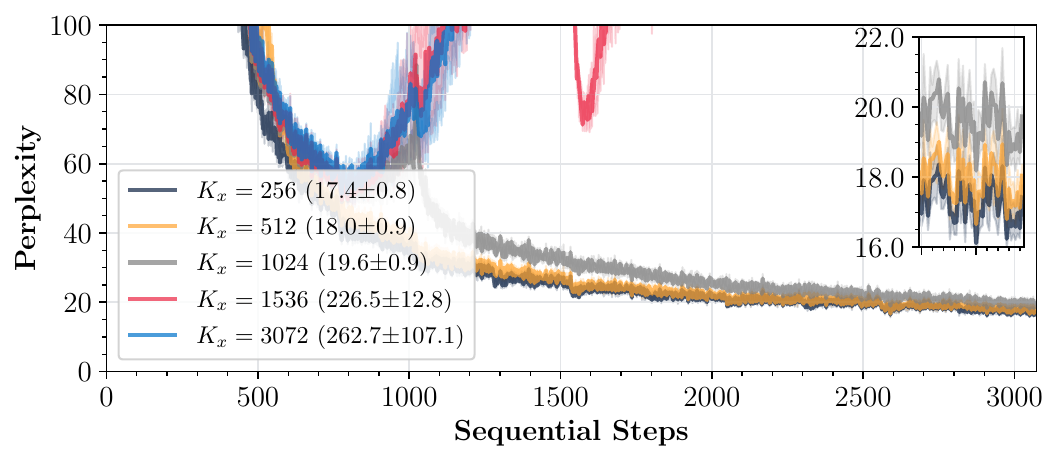}}\label{fig:eval:independent_sync_frequencies_a}  \hfill
    \noindent\subfloat[\method vary $K_v$ , fixed $K_x=K_u=256$]
    {\includegraphics[width=0.485\columnwidth]{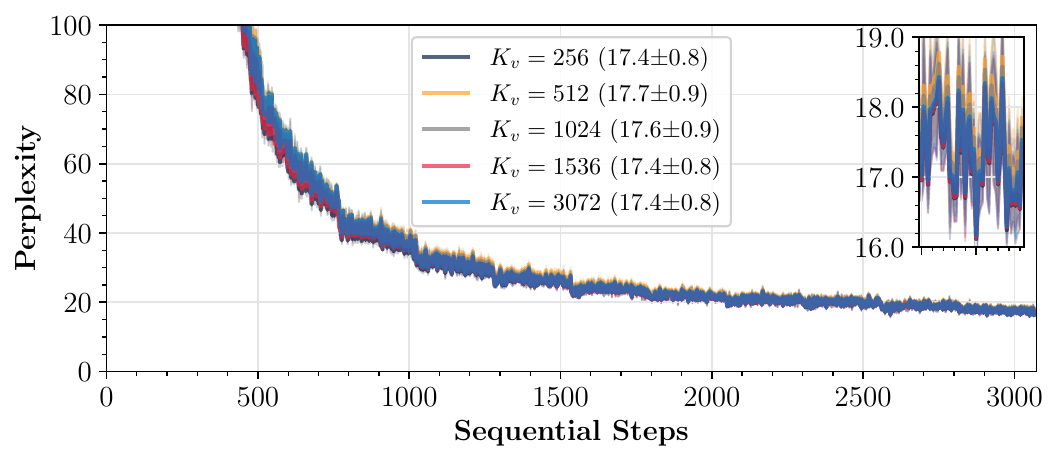}} 
    \hfill
    \noindent\subfloat[\method vary $K_u$, fixed $K_x=K_v=16$]
    {\includegraphics[width=0.485\columnwidth]{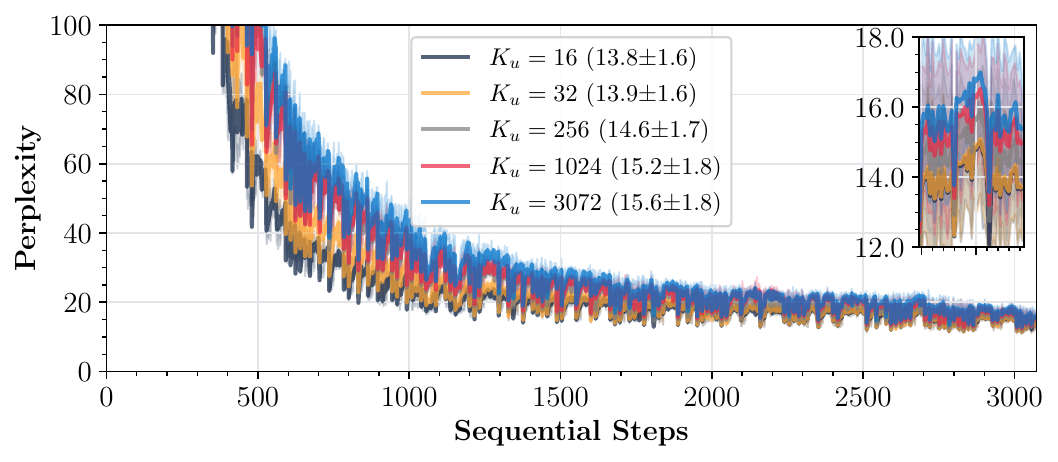}}  \hfill
    \noindent\subfloat[\method vary $K_u$, fixed $K_x=K_v=256$]
    {\includegraphics[width=0.485\columnwidth]{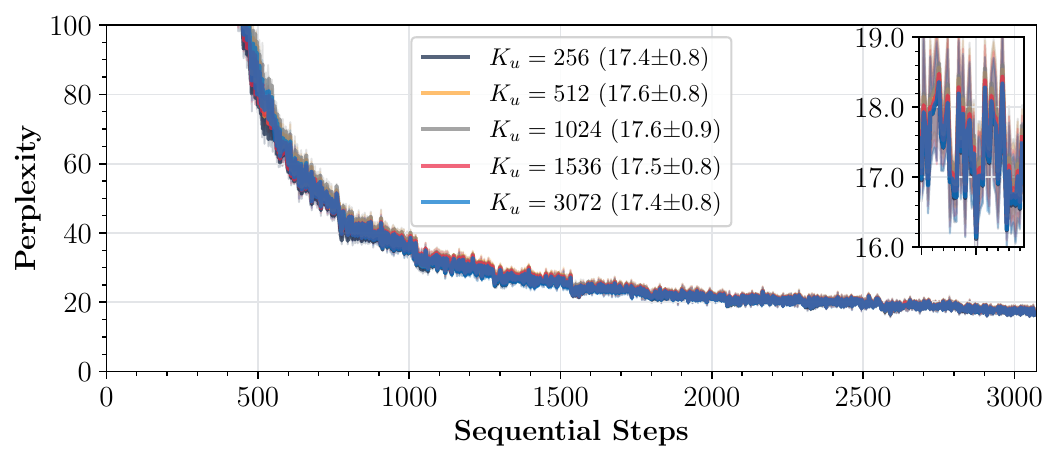}} \hfill
    \caption{Model perplexity for \method (\adopt, $\beta_1=0.95,\beta_2=0.9999$), varying synchronization periods independently (others fixed at $K_b$). Parameter synchronization (a) is critical, with sharp degradation at higher periods. Second-momentum synchronization (b) minimally affects performance due to its large half-life ($\tau_{0.5}(\beta_2)\gg K_b$). First-momentum synchronization significantly improves perplexity~(c) only when the baseline matches its half-life ($K_b=16$), having minimal impact otherwise (d). Parameters and second momentum behave similarly across sync frequencies~(\cref{app:additional_results})
    }
    \label{fig:eval:independent_sync_frequencies}

\end{figure}

\takeawaybox[Takeaway:]{%
Parameter synchronization frequency ($K_x$) strongly impacts performance, motivated by the leading term in theoretical bounds (\cref{sec:theory}). Momentum synchronization periods matter empirically only when chosen near their half-lives, consistent with~\cref{sec:theory,sec:methods}.
}

\subsection{\method Brings $2\times$ Communication Reductions Relative to \localadam ~(RQ3)}\label{eval:subsec:baseline_comparison}

\Cref{fig:eval:baseline_comparison} shows \method achieves a $2\times$ communication reduction over the prior state-of-the-art \localadam~\citep{LocalAdam} without significant perplexity degradation, even
% in node-churn environments.
when adding workers.
Synchronizing parameters at $K_x=K$ (matching \localadam) and momenta at $K_u=3K_x$, $K_v=6K_x$ consistently yields minimal degradation, aligning with the slower evolution and lower sensitivity of second-momentum sync frequency~(\cref{fig:eval:independent_sync_frequencies}). Other low communication configurations are in \cref{app:additional_results}.

\begin{figure}[H]
    \centering
    \noindent\subfloat[\methodadopt vs baselines $K_x=16$ IID]
    {\includegraphics[width=0.485\columnwidth]{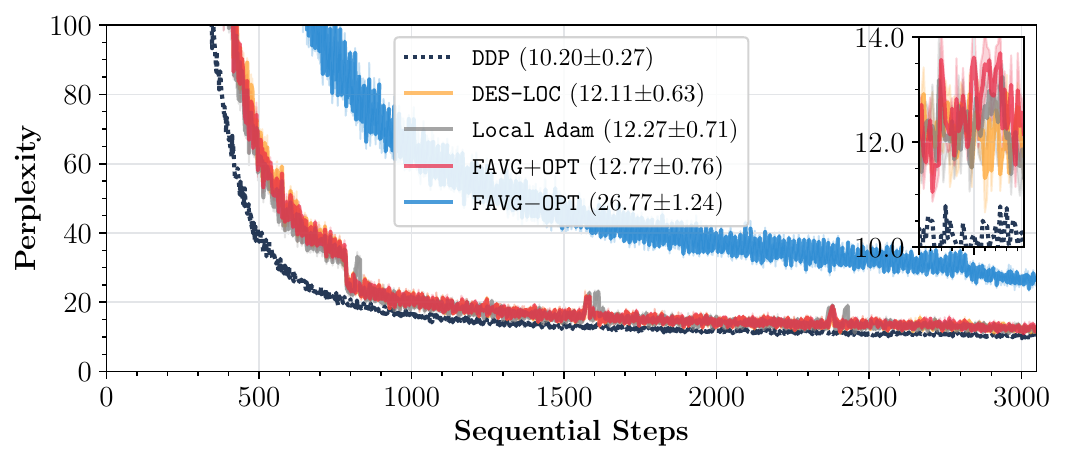}}  \hfill
    \noindent\subfloat[\methodadopt vs baselines $K_x=256$ IID]
    {\includegraphics[width=0.485\columnwidth]{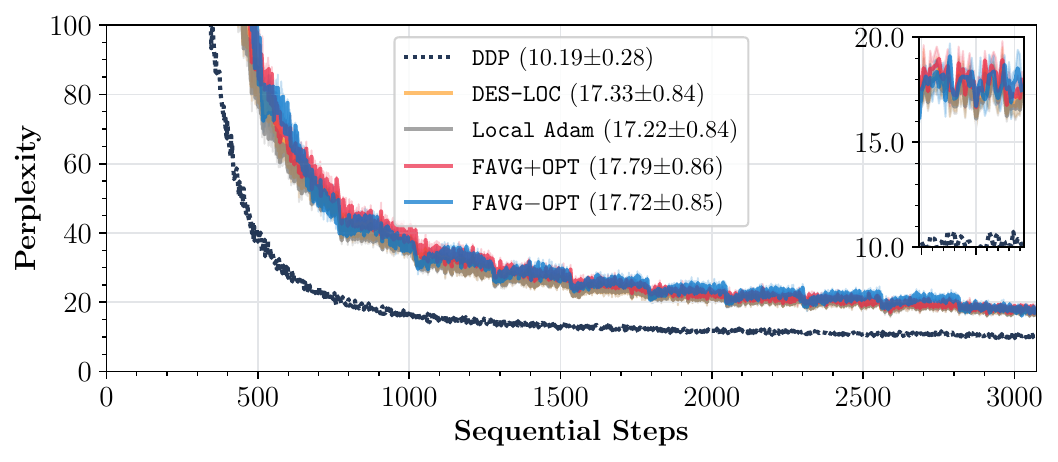}} \hfill
    \noindent\subfloat[Perplexity impact of doubling workers, $K_x=128$.]
    {\includegraphics[width=0.485\columnwidth]{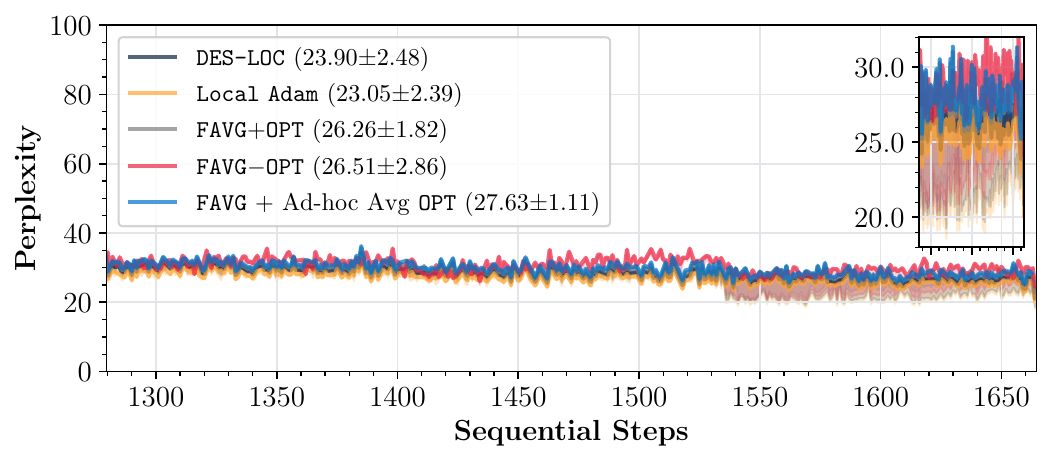}}  \hfill
    \noindent\subfloat[Gradient norms after doubling workers, $K_x=128$.]
    {\includegraphics[width=0.485\columnwidth]{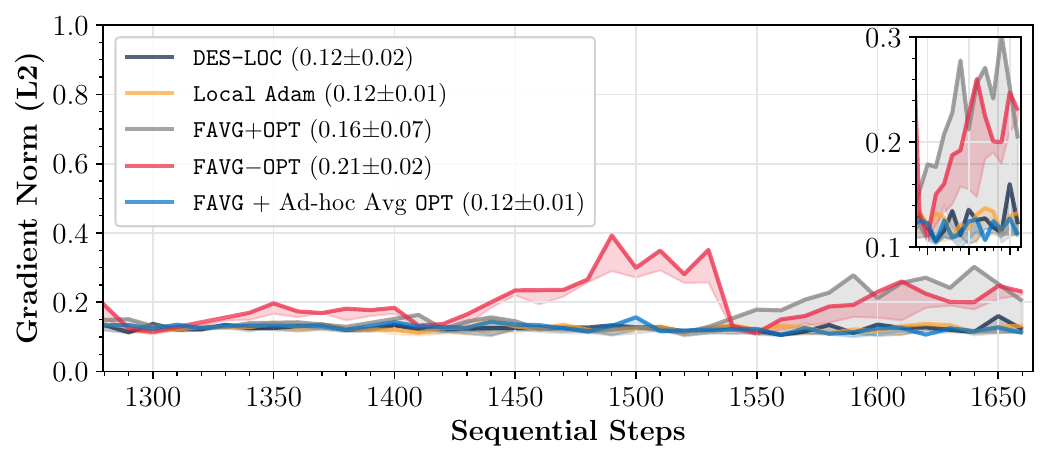}} 
    
    \caption{
    Setting $K_x=K$, $K_u=3K_x$, and $K_v=6K_x$, \method achieves a $\mathbf{2}\times$ communication reduction over \localadam, matching performance at high~(a) and low~(b) frequencies for \localadam and heuristic baselines (see \cref{sec:exp_setup}). We demonstrate robustness to 
    % node churn
    the addition of new workers
    by doubling worker count at step $1536$~(c,d); \method and \localadam remain stable in perplexity/gradient norms, outperforming heuristic methods and ad-hoc optimizer-state averaging.% right before churn.
    }
    \label{fig:eval:baseline_comparison}

\end{figure}

\takeawaybox[Takeaway:]{%
\method achieves a $2\times$ communication reduction over \localadam by leveraging two insights: optimizer-state sync matters less than parameter sync, and slower-changing states (high $\beta_2$) can sync less often. By eventually syncing all optimizer states, \method matches the robustness of \localadam with $K=\max(K_x,K_u,K_v)$ when adding new workers/responding to system failures.
}

\subsection{\method Performs Well At Large-scale Long Horizon Training~(RQ4)}\label{eval:subsec:larger_models}

\Cref{fig:eval:large_models} shows that \method reliably scales to billion-scale models and extensive training workloads. Evaluating the billion-scale models on the \texttt{ICL} tasks (\cref{tab:icl-table}), \method remains competitive with all baselines while significantly reducing communication versus \localadam and \ddp. The heuristic baseline~\citep{Photon} suffers notable training instabilities (\cref{fig:eval:large_models}.b) potentially impacting its downstream performance (\cref{tab:icl-table}) and underscoring the advantage of \method's training stability.
\begin{figure}[H]
    \centering
    \noindent\subfloat[\methodadopt $1$B-model perplexity]
    {\includegraphics[width=0.485\columnwidth]{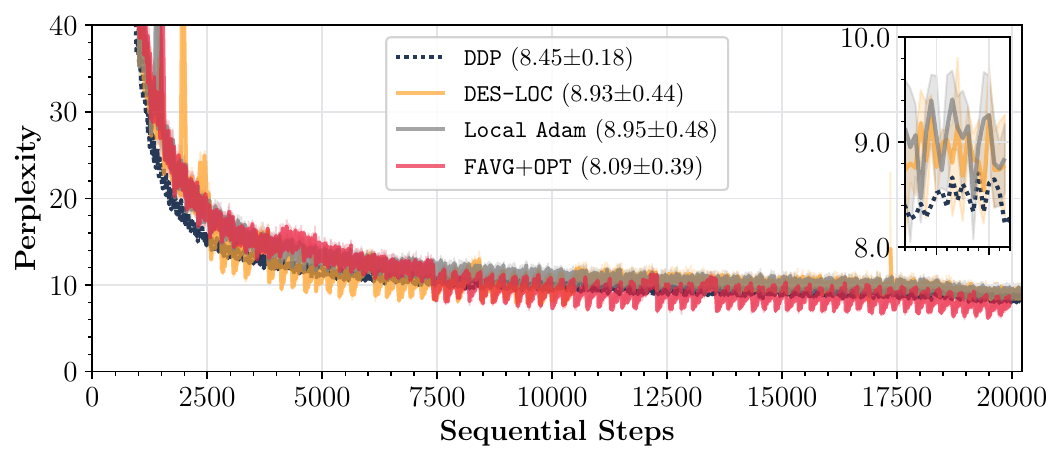}}  \hfill
    \noindent\subfloat[\methodadopt $1$B-model activations]
    {\includegraphics[width=0.485\columnwidth]{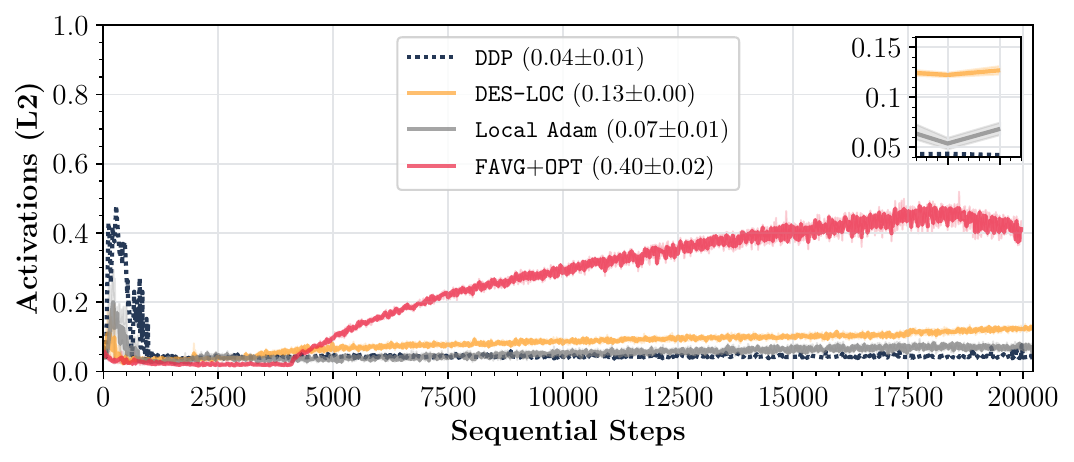}} \hfill
    
    \caption{
   \method matches \localadam perplexity for billion-scale model training at half the communication cost ($K_x=256,K_u=3K_x,K_v=6K_x$), representing a $\mathbf{170}\times$ reduction over \ddp. Though initially behind \ddp, both \method and \localadam quickly converge to competitive perplexity at \textbf{longer training horizons}. Federated Averaging (keeping optimizer states) achieves reasonable performance~(a) but suffers activation growth~(b) and parameter-norm growth~(\cref{app:additional_results}), potentially due to noisy local updates, raising concerns for extended training ($\geq11$ trillion tokens~\citep{SmolLM2}).
    }
    \label{fig:eval:large_models}

\end{figure}

\takeawaybox[Takeaway:]{%
\method enables efficient training of large-scale foundation models, especially at long training horizons, achieving downstream \texttt{ICL} performance competitive with \ddp.
}

\begin{table}
\caption{Our billion-scale model trained with \method matches or surpasses the In-context Learning (\texttt{ICL}) performance of models trained with \localadam and Federated Averaging (keeping local optimizer states), approaching \ddp performance. Federated Averaging (keeping local optimizer states) slightly underperforms compared to its perplexity results from \cref{fig:eval:large_models}.a, indicating that the activation increases~(\cref{fig:eval:large_models}.b) from the unstable training procedure may have damaged the model.}
\label{tab:icl-table}
\centering
\begin{tabular}{@{}lccccc@{}}
\toprule
\textbf{Method} &
  \textbf{Arc Challenge}~\citep{arc_challenge} &
  \textbf{Arc Easy}~\citep{arc_challenge} &
  \textbf{PIQA}~\citep{piqa} &
  \textbf{HellaSwag}~\citep{hellaswag} &
  \textbf{Avg} \\ \midrule
\textbf{\method}                    & $31.8$          & \textcolor{blue}{$59.0$} & \textcolor{blue}{$70.7$} & $44.9$          & $51.6$          \\
\textbf{\localadam} &
  \textcolor{blue}{$31.9$} &
  \textcolor{blue}{$59.0$} &
  $70.6$ &
  \textcolor{blue}{$45.8$} &
  \textcolor{blue}{$51.8$} \\
\textbf{\texttt{FAVG}+\texttt{OPT}} & $30.1$          & $58.0$                            & $70.0$                            & $44.8$          & $50.7$          \\ \midrule
\textbf{\ddp}                       & $\mathbf{33.8}$ & $\mathbf{62.5}$                   & $\mathbf{71.1}$                   & $\mathbf{47.8}$ & $\mathbf{53.8}$ \\ \bottomrule
\end{tabular}
\end{table}

% \subsection{\method Is Robust to Node Churn~(RQ5)}
% \cref{fig:eval:node_churn} highlights that explicit synchronization of optimizer states makes \method and \local \adopt robust to node churn or methods like doubling batch size~\citep{}. Heuristic methods, lacking proper initialization for new workers, degrade significantly in both perplexity and gradient norms. 
% \begin{figure}[H]
%     \centering
    
%     \caption{Comparison between \method, \local \adopt, and heuristic methods (Federated Averaging keeping or resetting optimizer states) upon doubling worker count at step $1536$. We also test an ad-hoc averaging approach immediately before doubling, which underperforms \method. Only \method and \local \adopt remain stable in perplexity (a) and gradient norms (b), unlike heuristic methods.
%     }
%     \label{fig:eval:node_churn}
% \end{figure}
% \takeawaybox[Takeaway:]{%
% In dynamic environments with frequent node churn, \method is as robust as \localadam with $K=\max(K_x,K_u,K_v)$, outperforming ad-hoc methods even with perfect churn timing.
% }

%% file: files/background.tex
\section{Related  Work and Limitations}\label{sec:background}

\textbf{Communication bottlenecks for \ddp.} In synchronous data-parallel training, workers exchange full gradients or parameters \emph{every} iteration, incurring linear communication costs using \texttt{Ring-AllReduce}~\citep{Horovod}. When hardware is weakly connected or widely distributed, communication significantly slows wall-clock training time~\citep{Photon} as workers need to wait for synchronization to finish.
% Classical mitigations include: (i) gradient compression and (ii) less frequent synchronization.
% We focus on the latter to avoid compression's computational overhead.

\textbf{Periodic local updates.} \textit{Federated Averaging} (\texttt{FedAvg})~\citep{fedavg} and \texttt{Local SGD}~\citep{LocalSGD} reduce communication by performing $K$ local optimization steps before averaging parameters, decreasing communication rounds by a factor of $K$. Although provably convergent for distributed \texttt{SGD}, these guarantees do not extend to adaptive optimizers commonly used for foundation models, due to local optimizer states. Ad-hoc solutions either keep optimizer states local~\citep{DiLoCo,DiLoCoScalingLaws,AsyncDiLoCo} or reset them after each sync~\citep{LLMFL,Photon}, both lacking robust convergence guarantees, unlike \texttt{Local SGD}.

\textbf{Local stateful optimizers under communication constraints.} \adam~\citep{Adam} is widely adopted for pre-training because it scales to larger batches than \texttt{SGD}~\citep{NoiseIsNotTheMainFactorSGDAdam}, suiting large GPU clusters~\citep{llama3}. It approximates the gradient's sign~\citep{DissectingAdam} using exponential moving averages of gradients and their squares; however, its convergence is not guaranteed as it requires $\beta_1<\sqrt{\beta_2}<1$, with large, problem-specific $\beta_2$~\citep{OnTheConvergenceOfAdamAndBeyond,AdamCanConvergeWithoutAnyModificationToUpdateRules}. Other momentum-based and adaptive optimizers~\citep{NesterovIlya,LION,LAMB,ADOPT} also track gradient moments. \emph{Local Adam}~\citep{LocalAdam} reduces communication by allowing multiple local optimization steps before global averaging and converges faster than \ddp per communication round, provided each worker syncs parameters and optimizer states. However, synchronizing these states triples communication relative to \texttt{Local SGD}/\ddp, offsetting the reduced frequency. In general, sync costs scale linearly with the number of optimizer states.

\textbf{Limitations.} First, while our main non-convex convergence result holds for \texttt{SGDM}, in the case of \adam we discuss the analyses (both in expectation and high-probability results) with additional assumptions such as bounded gradient condition and homogeneous data distribution. Nevertheless, these assumptions are commonly used in the non-convex adaptive optimization. Second, due to compute limitations (two machines with 4$\times$\texttt{A40}s and two with 4$\times$\texttt{H100}s), our hyperparameter search was extensive yet constrained to smaller models. Lastly, while our analysis uses \adam/AMSGrad, many experiments use modified \adam~(\adopt)~\citep{ADOPT}.

%% initial paragraph
% First, our non-convex convergence results hold for \texttt{SGDM}, not \adam, which we prove in a less general setting. Second, due to compute limitations (two machines with 4$\times$\texttt{A40}s and two with 4$\times$\texttt{H100}s), our hyperparameter search was extensive yet constrained to smaller models. Lastly, while our analysis uses vanilla \adam, many experiments use modified \adam~(\adopt)~\citep{ADOPT}.

%% file: files/limitations.tex
% \section{Limitations}
% In this section, we briefly discuss some of our work's limitations. 

%% file: files/conclusion.tex
\section{Conclusion}
% \ls{Make this as bombastic as it gets. I feel this is one of those papers where the conclusions are indeed important.}
% \ls{Make sure we mention the potential extension to a layer-wise thing. I don't want this to be taken by someone else without citing us.}
% \ls{Make sure that we have one sentence mentioning the potential of using this coupled with compression techniques for further reducing costs.}

\method reconciles communication efficiency with rigorous convergence guarantees in distributed adaptive optimization. By extending theory to the independent synchronization of \adam and \texttt{SGDM} optimizer states, we empirically demonstrate convergence alongside $\mathbf{170\times}$ and $\mathbf{2\times}$ communication reductions over \ddp and prior state-of-the-art methods at billion-scale LLM training, even in 
% challenging node-churn environments.
environments prone to system failures.
Our findings yield clear guidelines: i) \textbf{frequently} synchronize parameters, and ii) synchronize optimizer states \textbf{less often}, proportional to their half-lives. These insights open avenues for future research, including layer-wise synchronization, adaptive frequencies, compressed updates, as well as emerging applications, such as worldwide cross-data center training and collaborative training. As training workloads scale, we envision \method becoming the standard for efficient, resilient foundation-model training in data centers and general distributed environments.

%% file: files/appendix.tex
\input{files/appendices/experimental_details}
\input{files/appendices/additional_results}

\input{files/appendices/optimizer_specific_variants}

\input{files/appendices/proof-sgdm}

\input{files/appendices/proof}
\input{files/appendices/derivation_max_change}

\input{files/appendices/wall_time_modeling}

%% file: files/appendices/experimental_details.tex
\section{Experimental Details and Optimizer Hyperparameter Sweeps~(See \cref{sec:exp_setup})}~\label{app:exp_details}

Here we provide additional experimental details complementing those in \cref{sec:exp_setup}, including: a) model architecture details and hyperparameters independent of optimizer choice (\cref{app:subsec:architecture_details}), b) our hyperparameter sweep procedure to select optimizer-specific settings (\cref{app:subsec:hyperparameter_sweeping_procedure}), and c) the optimal hyperparameters with those used in \cref{sec:evaluation} highlighted in bold. 

\subsection{Architecture Details and Hyperparameters}\label{app:subsec:architecture_details}

\input{tables/arch_details}

\Cref{tab:model_architectures} summarizes the architectural details of our models, following established practices for large language models at their respective scales. Unless otherwise stated, we adopt the hyperparameters recommended by \citet{SmolLM2} for both the $135$M and the $1.7$B models. We operate at a batch size of $2$M tokens, which is very large for the $135$M model at the length of training we perform~\citep{HowDoesBatchSizeScaleInPreTraining} and industry-standard for the $1.7$B model~\citep{llama2}, we chose to operate at large batch sizes because adaptive optimizers provide benefits primarily in large-batch training regimes~\citep{NoiseIsNotTheMainFactorSGDAdam}. Moreover, we intend \method for use in cross data-center scenarios, where effectively utilizing available accelerators naturally demands large batch sizes and/or model scales. For both model sizes, we train for approximately $2\times$ the compute-optimal token budget~\citep{TrainingComputeOptimalLLMs}, placing our evaluations within the context of extended-duration foundation model training~\citep{SmolLM2}. Our chosen token budget is conservative due to resource constraints; for comparison, \citet{SmolLM2} used $11$ trillion tokens which is over $4000\times$ compute-optimal for the $135$M model, and $300\times$ for the $1.7$B.

We select warmup and decay schedules following recommendations from \citet{HowDoesBatchSizeScaleInPreTraining,BeyondFixedTrainingDuration,SmolLM2}. For the $135$M model, the warmup period is set to $T_{\mathrm{WARM}} = 512$ steps, corresponding to the roughly $40\%$ of the compute-optimal training tokens recommended by \citet{HowDoesBatchSizeScaleInPreTraining}. For the $1.7$B model, we use the recommended $T_{\mathrm{WARM}} = 2048$ steps from \citet{SmolLM2}, roughly $10\%$ of total training. The stable-decay period uses a $1-\mathrm{SQRT}$ schedule over the final $T_{\mathrm{DECAY}}=10\%\times T$ steps~\citep{BeyondFixedTrainingDuration}. For shorter runs, such as $T=1536$ during heterogeneous-data evaluations, we keep the warmup fixed and proportionally scale the decay to ensure well-conditioned parameter updates during the stable learning rate period. The seeds we use for data sampling and for controlling the training algorithms and model are provided in the code accompanying the appendix.

\subsection{Optimizer Parameters Sweeping Procedure}\label{app:subsec:hyperparameter_sweeping_procedure}

As detailed in \cref{sec:methods} and verified empirically in \cref{eval:subsec:relative_importance}, the choice of decay rates $\beta_1,\beta_2$ strongly influences the effective synchronization frequencies achievable by both \method and \localadam. This relationship arises directly from the half-life of optimizer states, given by $\tau_{0.5} = \frac{\ln(0.5)}{\ln(\beta)}$.

For \adam, prior studies such as \citet{StableLowPrecisionTrainingLLMLVM} have demonstrated a critical interplay between the learning rate ($\eta$), batch size, and the second-momentum decay $\beta_2$. Specifically, increasing either the learning rate or batch size typically demands a lower $\beta_2$ to maintain training stability and avoid loss spikes. Conversely, higher $\beta_2$ values constrain the learning rate and batch size. Such dynamics have also been recently observed between the learning rate and the first-momentum decay $\beta_1$ in \citet{AdemaMix}. Given that all our experiments use a fixed large batch size of roughly $2$ million tokens (appropriate for billion-scale training), we systematically tune the learning rate $\eta$ in response to changes in $\beta_1,\beta_2$. We try values of $\beta_1,\beta_2$ based on previous works~\citep{HowDoesBatchSizeScaleInPreTraining} and follow the theoretical convergence requirement of \citet{AdamCanConvergeWithoutAnyModificationToUpdateRules} setting $\beta_1 \leq \sqrt{\beta_2}$.

Due to computational constraints, we cannot jointly optimize synchronization periods, data distributions, and decay parameters, and instead adopt a structured two-stage tuning approach:

\begin{enumerate}
    \item \textbf{Stage 1: Tuning $\eta$ for \ddp}. Starting from the recommended baseline learning rate ($\eta_0$) from \citet{SmolLM2}, we conduct a grid search as outlined by \citet{DiLoCoScalingLaws}:
    $
        \{\dots,\sqrt{2}^{-2}\eta_0, \sqrt{2}^{-1}\eta_0, \eta_0, \sqrt{2}\eta_0, \sqrt{2}^{2}\eta_0,\dots\}
    $ 
    We expand this search until perplexity stops improving, identifying an optimal learning rate $\eta^\ast_{\ddp}$ for each $(\beta_1,\beta_2)$ configuration.
    
    \item \textbf{Stage 2: Tuning $\eta$ for \localadam}. We then repeat this procedure for \localadam, using $\eta^\ast_{\ddp}$ as the new baseline. To balance generalizability and computational cost, we set the synchronization period to an intermediate value of $K=64$, between high-frequency ($K=16$) and low-frequency ($K=256$) scenarios.
\end{enumerate}

Additionally, following \citet{HowDoesBatchSizeScaleInPreTraining}, we omit weight decay (set to zero) to simplify the hyperparameter tuning process, as it directly affects only model parameters, not optimizer states.

\subsubsection{Optimizers' Hyperparameter Configurations}

\input{tables/optimizer_optimal_values}

Our hyperparameter sweep (\cref{tab:optimizer_optimal_values}) indicates that the optimal learning rate $\eta^\ast$ under the warmup-stable-decay scheduler~\citep{BeyondFixedTrainingDuration} strongly depends on both optimizer type and the chosen $\beta_1,\beta_2$ values. For \adam, optimal learning rates and second-momentum decay ($\beta_2$) align closely with recommendations from \citet{SmolLM2}, though a slightly higher first-momentum decay ($\beta_1$) consistently performs better, in agreement with prior findings~\citep{HowDoesBatchSizeScaleInPreTraining}. For \adopt (default $\beta_2$), we observe a lower optimal learning rate compared to \adam, but similar best-performing $\beta_1$ values. We also find that the optimal learning rate does not differ between \ddp and \localadam for given $\beta_1,\beta_2$ when $K=64$ and using a $\sqrt{2}$ sweep, higher learning rates either do not provide a benefit or diverge while lower learning rates are only necessary when pushing $K$ far closer to the complete training duration.

We find that increasing $\beta_1$ for \adopt, and $\beta_1,\beta_2$ for \adam, leads to rapid performance degradation, particularly at or above $0.99$. Since the half-life at $\beta=0.99$ ($\tau_{0.5}\approx69$) is not sufficiently longer than at $\beta=0.95$ ($\tau_{0.5}\approx13.5$) to justify the observed performance drop, we select $\beta_1=0.95$ for all experiments, along with the default $\beta_2$ for \adopt and $\beta_2 = 0.95$ for \adam.

\takeawaybox[Takeaway:]{%
Increasing an optimizer state's $\beta$ significantly affects performance. Since linear increases in $\beta$ cause only logarithmic changes in half-life $\tau_{0.5}$, raising $\beta$ beyond the optimal value degrades performance without substantially improving the achievable synchronization frequency~(\cref{eval:subsec:relative_importance}).
}

%% file: tables/arch_details.tex
\begin{table}[H]
\caption{Model architecture and training parameters. We denote the number of transformer blocks by \#Blocks, number of attention heads by \#Heads, embedding dimension by $d_\mathrm{model}$, vocabulary size by $|\mathcal{V}|$, and feedforward-layer expansion by Exp.~Ratio. All models use positional embeddings~\citep{RopeEmbeddings}, the \texttt{silu} activation function, and norm-based gradient clipping with clip-bound $\rho$. Global batch size (summed across all workers) is $|\mathcal{B}_{\mathrm{G}}|$, and sequence length is standard for models at these scales. For model initialization we use $\sigma=\nicefrac{1}{\sqrt{d_{\mathrm{model}}}}$. The total number of steps is denoted by $T$.}\label{tab:model_architectures} 
\centering
\resizebox{\textwidth}{!}{%
\begin{tabular}{@{}lcccccccccccc@{}}
\toprule
\textbf{Model Size} &
  \textbf{Blocks} &
  \textbf{$\boldsymbol{d_{\mathrm{model}}}$} &
  \textbf{$\mathbf{|\mathcal{V}|}$} &
  \textbf{\#Heads} &
  \textbf{Exp.~Ratio} &
  \textbf{ROPE $\theta$} &
  \textbf{ACT} &
  \textbf{Init $\sigma$} &
  \textbf{$\rho$} &
  \textbf{Seq Len} &
  \textbf{$\mathbf{|\mathcal{B}_{\mathrm{G}}|}$} &
  \textbf{$\mathbf{T}$} \\ \midrule
$135$M &
  $30$ &
  $576$ &
  $50$K &
  $9$ &
  $4$ &
  $10000$ &
  \texttt{silu} &
  $0.04$ &
  $1.0$ &
  $2048$ &
  $1024$ &
  \textbf{$1536,3072$} \\
$1.7$B &
  $24$ &
  $2048$ &
  $50$K &
  $16$ &
  $4$ &
  $10000$ &
  \texttt{silu} &
  $0.02$ &
  $1.0$ &
  $2048$ &
  $1024$ &
  \textbf{$20480$} \\ \bottomrule
\end{tabular}%
}
\end{table}

%% file: tables/optimizer_optimal_values.tex
\begin{table}[H]
\caption{Optimal learning rates $\eta^\ast$ for $\beta_1,\beta_2$ configurations of \adopt/\adam. The hyperparameter sweep procedure (see \cref{app:subsec:hyperparameter_sweeping_procedure}) involves incrementally adjusting the learning rate by factors of $\sqrt{2}$ around the initial value from \citet{SmolLM2} until performance stops improving. 
% Reported mean perplexities and standard deviations are computed for the cooldown period, \localadam uses $K=64$.
}
\label{tab:optimizer_optimal_values}
% \resizebox{\columnwidth}{!}{%
\centering
\begin{tabular}{lccc}
\hline
\textbf{Optimizer} & \textbf{$\mathbf{\beta_1}$} & \textbf{$\mathbf{\beta_2}$} & \textbf{$\eta^\ast$} \\ \hline
\multirow{4}{*}{\adopt} & $0.9$           & $0.9999$          & $0.0021$          \\
                        & $\mathbf{0.95}$ & $\mathbf{0.9999}$ & $\mathbf{0.0021}$ \\
                        & $0.99$          & $0.9999$          & $0.0014$          \\
                        & $0.995$         & $0.9999$          & $0.0007$          \\ \hline
\multirow{5}{*}{\adam}  & $0.9$           & $0.95$            & $0.0042$          \\
                        & $\mathbf{0.95}$ & $\mathbf{0.95}$   & $\mathbf{0.003}$           \\
                        & $0.9$           & $0.99$            & $0.003$           \\
                        & $0.95$          & $0.99$            & $0.003$           \\
                        & $0.99$          & $0.99$            & $0.0021$          \\ \hline
\end{tabular}
% %
% }
\end{table}

%% file: files/appendices/additional_results.tex
\section{Complementary Results to \cref{subsec:method_algorithm,sec:evaluation} }\label{app:additional_results}

We now provide additional results supplementing those presented in the main text. Specifically:

\begin{enumerate}
    \item \textbf{\cref{app:subsec:toy_problem}} complements \cref{fig:toy_distances_iid} by including results on the heterogeneous data distribution described in \cref{sec:exp_setup}. This highlights \method's robustness under imperfect sampling or strongly \texttt{Non-IID} federated scenarios~\citep[see][Sec 3.1]{AdancesAndOpenProblems}.

    \item \textbf{\cref{app:subsubsec:independent_sync_params_second_momentum}} complements \cref{fig:eval:independent_sync_frequencies} by showing the separate impact of varying synchronization frequencies for parameters and the second momentum when the base frequency is $K_b=16$. It supports our claim that parameters and second momentum exhibit similar behavior across different synchronization regimes, unlike the first momentum.

    \item \textbf{\cref{app:subsubsec:independent_sync_adam_results}} extends \cref{fig:eval:independent_sync_frequencies} by evaluating \methodadam. We confirm that the parameter synchronization frequency is the most important, as predicted by our theory. In contrast, the momenta sync frequency is far less impactful, especially for low parameter sync frequencies.

    \item \textbf{\cref{app:subsubsec:comm_reduction_het_data}} complements \cref{fig:eval:baseline_comparison} by showing \methodadopt's perplexity against baseline methods on heterogeneous data (as defined in \cref{sec:exp_setup}). This validates our claim from Contribution 2 regarding \method's effectiveness on heterogeneous datasets.

    \item \textbf{\cref{app:subsubsec:low_comms_configurations_ablation}} presents an ablation study examining alternative low-communication configurations of \method, justifying our choice of $K_u=3K_x, K_v=6K_x$ used in \cref{fig:eval:baseline_comparison}.

    \item \textbf{\cref{app:subsubsec:baselines_adam_results}} repeats the baseline comparison from \cref{fig:eval:baseline_comparison} for \methodadam, demonstrating that \method achieves similar communication reductions and performance when using \adam instead of \adopt.

    \item \textbf{\cref{app:subsec:additional_metrics_favg_opt}} provides additional metrics illustrating training instabilities for the \texttt{FAVG}+\texttt{OPT} baseline, including rapidly growing parameter norms, supporting observations in \cref{fig:eval:large_models}.b.

\end{enumerate}

\subsection{Toy Problem on \texttt{Non-IID} Data~(See  \cref{fig:toy_distances_iid})}\label{app:subsec:toy_problem}

\takeawaybox[Toy Example \texttt{Non-IID}:]{
\cref{fig:toy_example_niid} simulates the scenario from \cref{sec:theory}, where each worker $m$ optimizes a distinct loss $f_m$ on heterogeneous data. Both \method and \localadam show more stable convergence and get closer to the optimum than heuristic baselines.
}

\begin{figure}[H]
    \centering
    \noindent\subfloat[Distance to the optimum.]
        {\includegraphics[height=4.25cm, width=0.485\columnwidth]{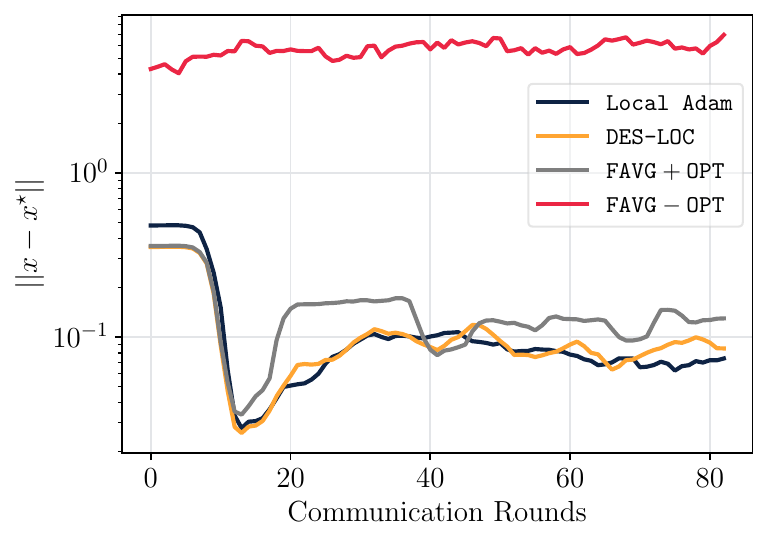}\label{fig:toy_distances_niid}}  \hfill
    \noindent\subfloat[2-D contour.]
    {\includegraphics[height=4.25cm, width=0.485\columnwidth]{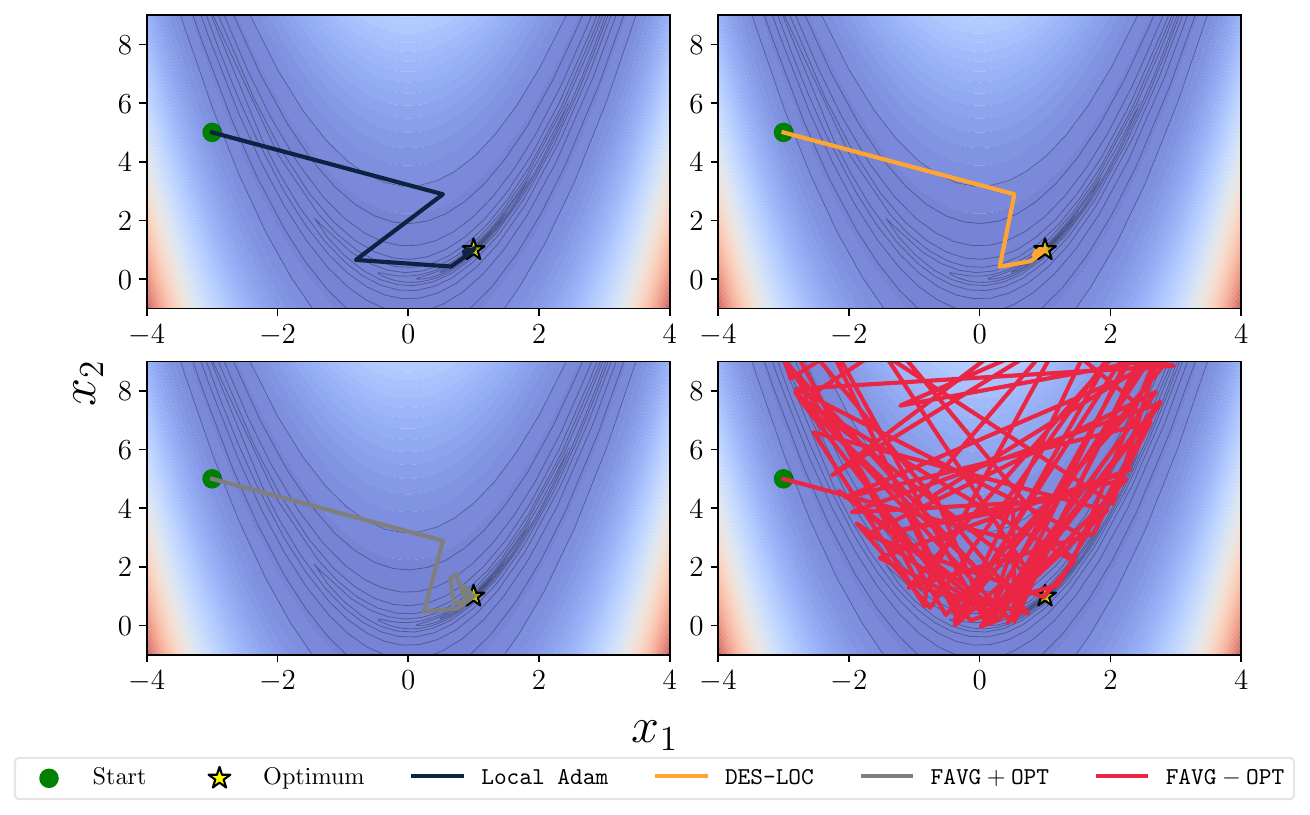}} \hfill
    \caption{We present a toy problem in a \texttt{Non-IID} setting, where \method (with synchronization periods $K_x=192, K_u=192, K_v=692$) and Local Adam (with $K=K_x$) converge to a superior solution compared to methods that keep optimizer states local \colorbox{grey_plot}{}~\citep{DiLoCo,Photon} or periodically reset them \colorbox{red_plot}{}~\citep{LLMFL,DEPT}. Like the \texttt{IID} scenario, resetting optimizer states prevents convergence due to repeated oscillations caused by reinitializations. Additionally, as seen in panel (a) between rounds $15$ and $40$, methods keeping optimizer states local suffer from larger oscillations further away from the optimum. The function optimized is $f(x_1,x_2)=(1-x_1)^2+100(x_2 - x_1^2)^2$, and we simulate $M=256$ workers, each adding Gaussian noise with worker-specific standard deviation $\sigma^m \sim \mathcal{N}(0,3)$.}
    \label{fig:toy_example_niid}
    \vspace{-0.2cm} 
\end{figure}

\subsection{\textbf{RQ2:} Independent Sync Frequencies}

This section provides supplementary results for \textbf{RQ2}, complementing \cref{eval:subsec:relative_importance}. \Cref{app:subsubsec:independent_sync_params_second_momentum} shows that perplexity has similar sensitivity to the first and second momentum synchronization frequencies at both high and low base synchronization frequencies. Additionally, \cref{app:subsubsec:independent_sync_adam_results} repeats the comparison from \cref{fig:eval:independent_sync_frequencies} for \methodadam, revealing similar trends regarding the importance of the parameters, with a reduced importance for the momenta due to lower $\beta_2$.

\subsubsection{Parameter and Second Momentum At $K_b=16$~(See \cref{fig:eval:independent_sync_frequencies}.a,\cref{fig:eval:independent_sync_frequencies}.b)}\label{app:subsubsec:independent_sync_params_second_momentum}
\Cref{app:fig:eval:independent_sync_frequencies_params_second_momentum} examines the effects of independently varying synchronization periods ($K_x,K_v$) for parameters and second momentum under \methodadopt in the high-frequency regime ($K_b=16$), chosen based on the first momentum's half-life ($\tau_{0.5}\approx13.5$). Similar to the low-frequency results in \cref{fig:eval:independent_sync_frequencies}.a, parameter synchronization frequency~($K_x$) strongly influences perplexity, while the second momentum~($K_v$) has minimal impact due to its long half-life. This contrasts with the first momentum, whose half-life closely matches the high-frequency period.

\begin{figure}[H]
    \centering
    \noindent\subfloat[\method vary $K_x$ , fixed $K_u=K_v=16$]
    {\includegraphics[width=0.485\columnwidth]{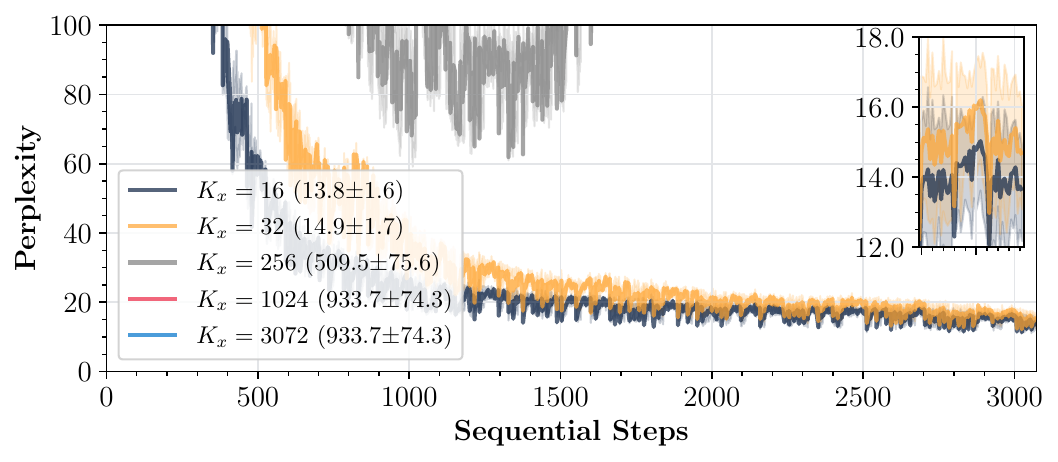}} \hfill
    \noindent\subfloat[\method vary $K_v$ , fixed $K_x=K_u=16$]
    {\includegraphics[width=0.485\columnwidth]{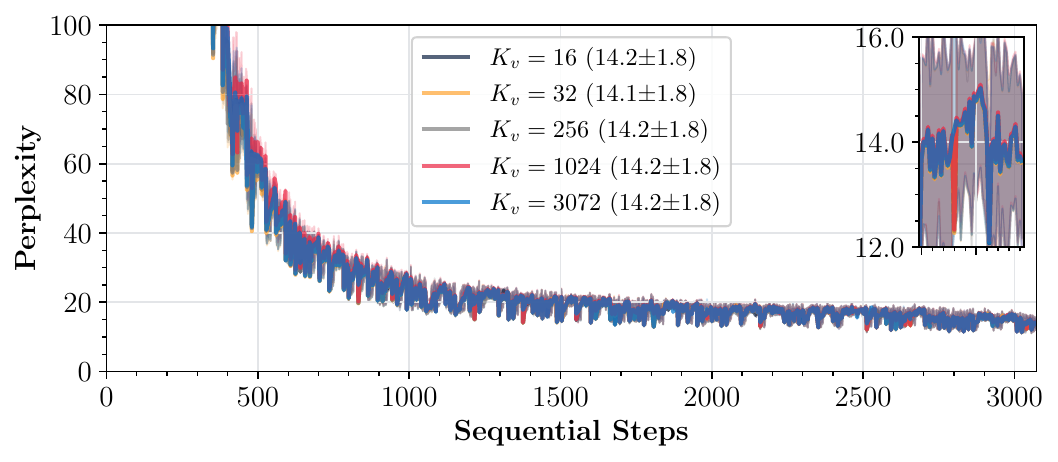}} 

    \caption{Model perplexity for \method (\adopt, $\beta_1=0.95,\beta_2=0.9999$), independently varying synchronization periods at a high baseline frequency ($K_b=16$). Similar to \cref{fig:eval:independent_sync_frequencies}, parameter synchronization~(a) is critical, with performance sharply degrading at higher periods, while second-momentum synchronization~(b) has minimal impact due to its large half-life ($\tau_{0.5}(\beta_2)\gg K_b$).
    }
    \label{app:fig:eval:independent_sync_frequencies_params_second_momentum}
\end{figure}

\takeawaybox[Takeaway:]{%
In high-frequency synchronization regimes, the importance of parameters and the second momentum remains similar to the low-frequency regime shown in \cref{eval:subsec:relative_importance}, 
}

\subsubsection{\adam Results~(See \cref{fig:eval:independent_sync_frequencies})}\label{app:subsubsec:independent_sync_adam_results}

\Cref{app:fig:eval:independent_sync_frequencies_adam} provides complementary results to \cref{fig:eval:independent_sync_frequencies,app:fig:eval:independent_sync_frequencies_params_second_momentum} using \methodadam with $\beta_1=\beta_2=0.95$. Unlike \adopt, the relatively low $\beta$ result in both the first and second momentum quickly adapting to the local gradients, reducing the impact of their sync frequency. 

\begin{figure}
    \centering
    \noindent\subfloat[\method vary $K_x$ , fixed $K_u=K_v=16$]
    {\includegraphics[width=0.485\columnwidth]{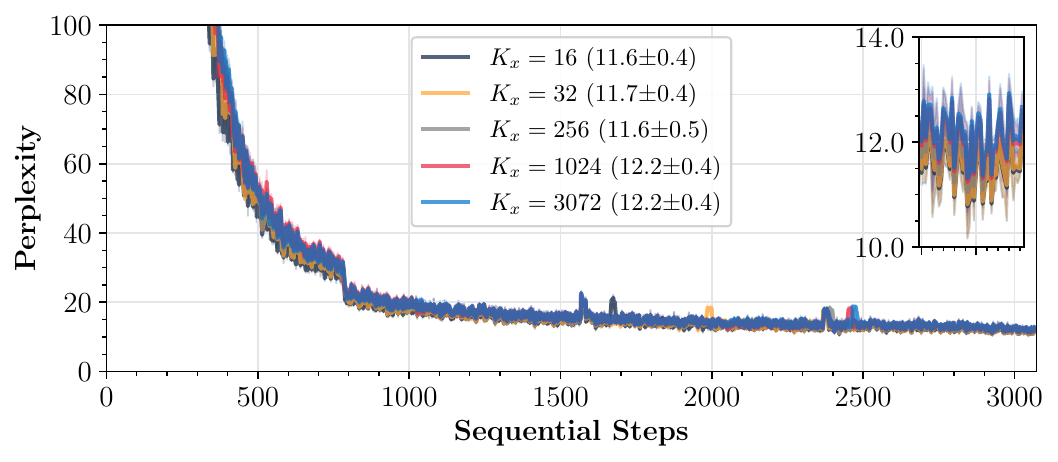}}  \hfill
    \noindent\subfloat[\method vary $K_x$ , fixed $K_u=K_v=256$]
    {\includegraphics[width=0.485\columnwidth]{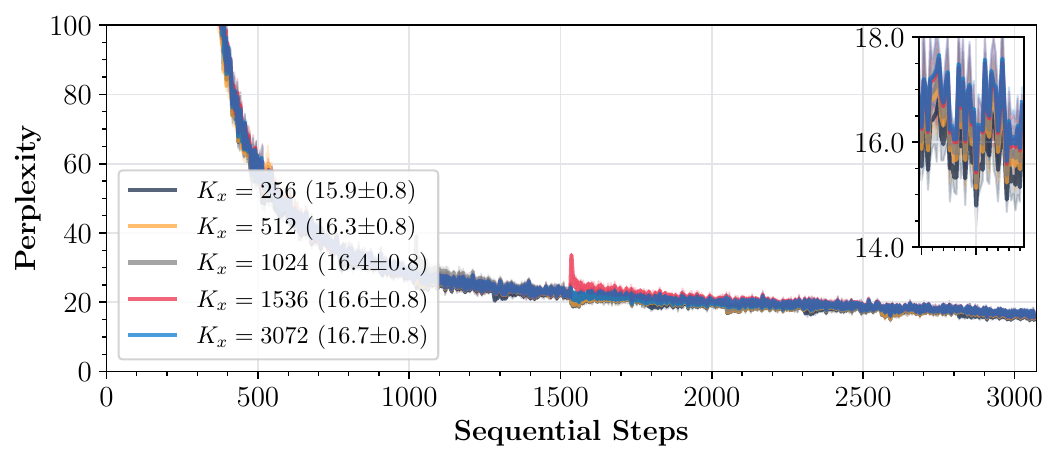}} 
    \hfill
    \noindent\subfloat[\method vary $K_u$, fixed $K_x=K_v=16$]
    {\includegraphics[width=0.485\columnwidth]{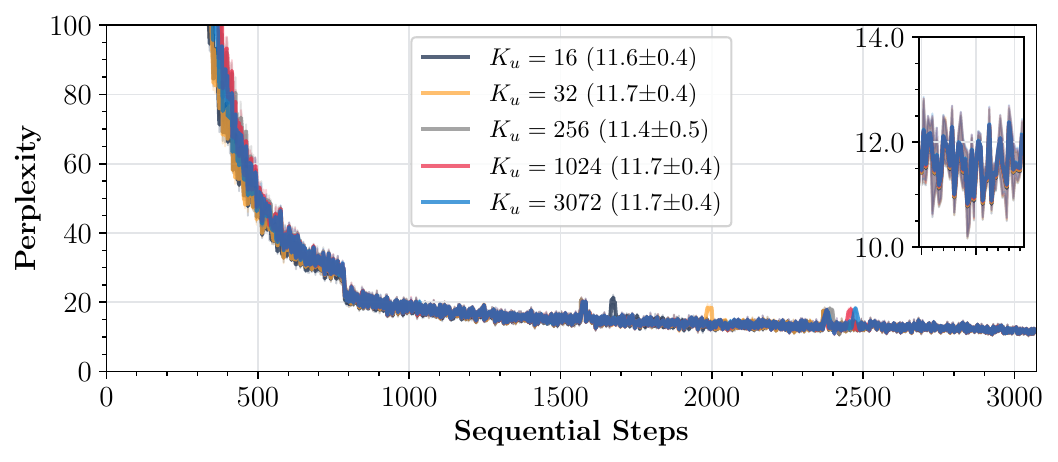}}  \hfill
    \noindent\subfloat[\method vary $K_u$, fixed $K_x=K_v=256$]
    {\includegraphics[width=0.485\columnwidth]{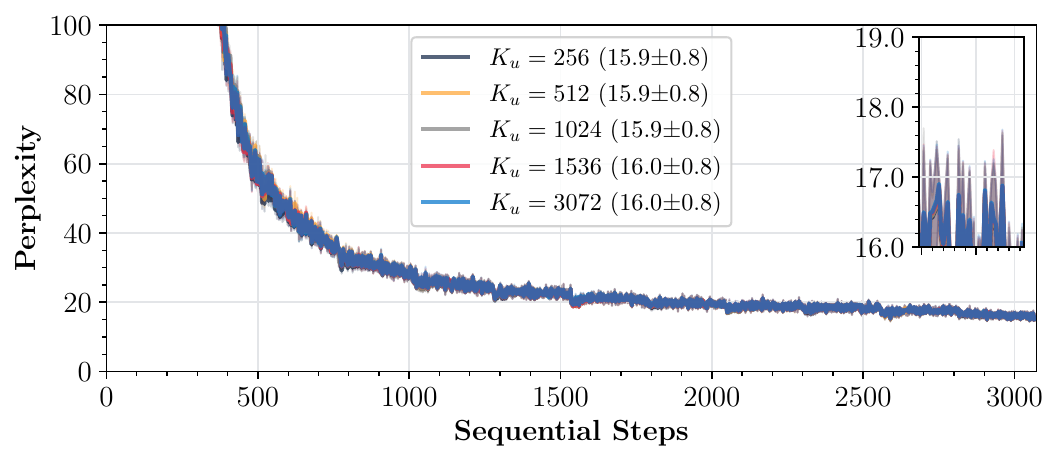}} \hfill
    
    \noindent\subfloat[\method vary $K_v$, fixed $K_x=K_u=16$]{\includegraphics[width=0.485\columnwidth]{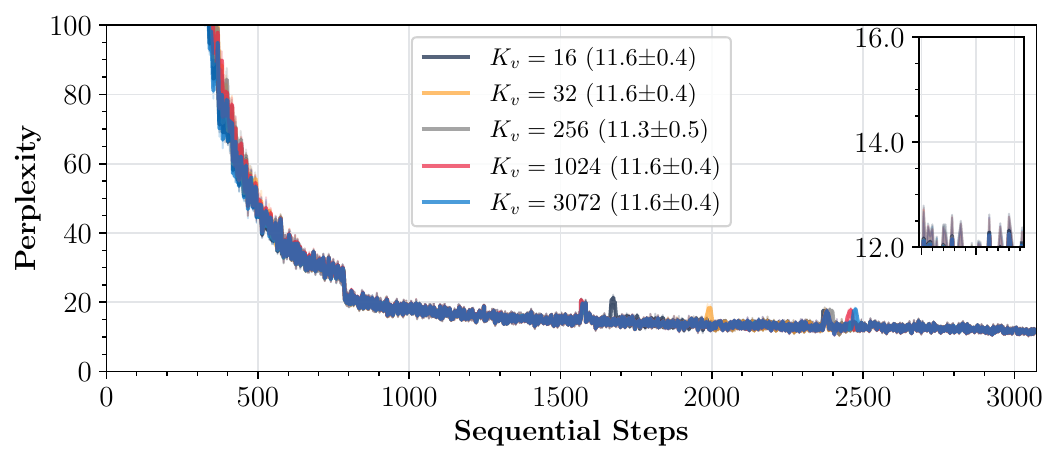}}  \hfill
    \noindent\subfloat[\method vary $K_v$, fixed $K_x=K_u=256$]
   {\includegraphics[width=0.485\columnwidth]{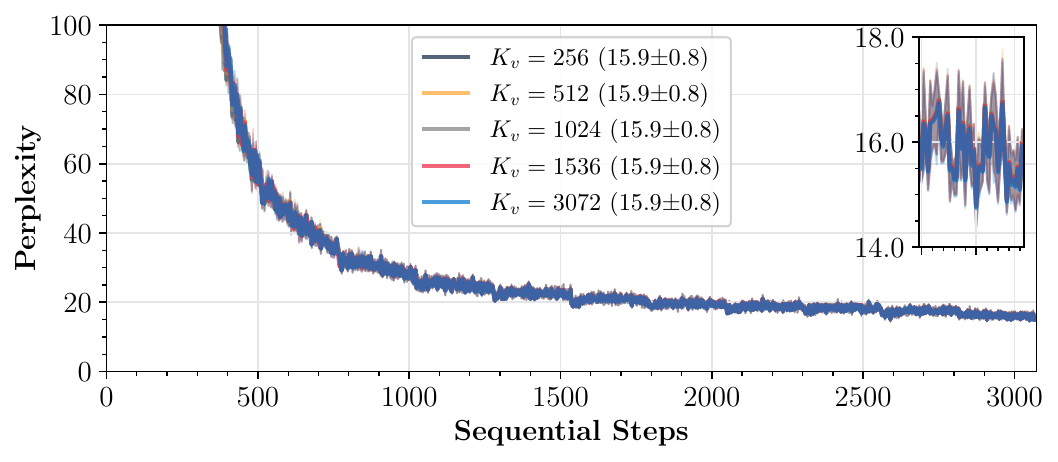}} \hfill
    \caption{Model perplexity for \methodadam~($\beta_1=\beta_2=0.95$) when independently varying sync periods ($K_x,K_u,K_v$) while fixing others at baseline $K_b$. Parameter synchronization (a,b) influences performance in both high ($K_b=16$) and low ($K_b=256$) frequency regimes. Momenta synchronization minimally impacts perplexity due to both states' high adaptivity (low $\beta$), with potentially minor effects during the early stages of training in high-frequency regimes~(c,e).
    }
    \label{app:fig:eval:independent_sync_frequencies_adam}

\end{figure}

\takeawaybox[Takeaway:]{%
For \methodadam, parameter synchronization remains critical, consistent with theory. However, due to reduced $\beta_2$, momenta synchronization is less impactful since both the numerator and denominator of \adam updates are driven by local worker gradients after a few initial steps.
}

\subsection{\textbf{RQ3:} Communication Reduction And Baseline Comparisons}

This section provides supplementary results for \textbf{RQ3}, complementing \cref{eval:subsec:baseline_comparison}. \Cref{app:subsubsec:comm_reduction_het_data} shows the perplexity of different configurations providing a $2\times$ communication reduction over \localadam. Additionally, \cref{app:subsubsec:baselines_adam_results} repeats the comparison against baselines from \cref{eval:subsec:baseline_comparison} for \methodadam, showing similar communication reductions relative to \localadam.

\subsubsection{\method on Heterogeneous Data~(See Contribution 2)}\label{app:subsubsec:comm_reduction_het_data}

\Cref{app:fig:eval:baseline_comparison_niid_data} evaluates the robustness of \method against baselines under heterogeneous (\texttt{Non-IID}) data distributions as described in \cref{sec:exp_setup}. We set synchronization periods to $K_x=K$, $K_u=3K_x$, and $K_v=6K_x$ to achieve a targeted $\mathbf{2}\times$ communication reduction over \localadam.

\begin{figure}[H]
    \centering
    \noindent\subfloat[\methodadam, high sync frequency $K_x=16$.]
    {\includegraphics[width=0.485\columnwidth]{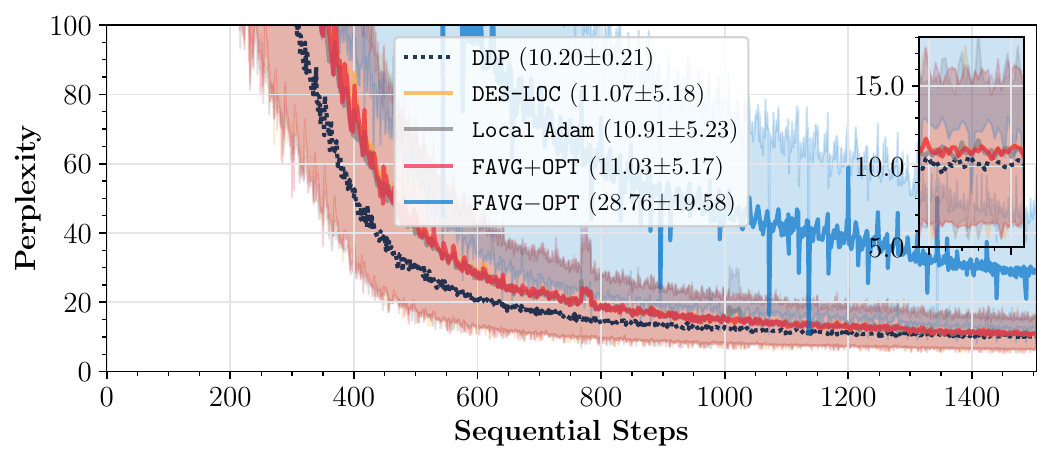}}  \hfill
    \noindent\subfloat[\methodadam, low sync frequency $K_x=128$.]
    {\includegraphics[width=0.485\columnwidth]{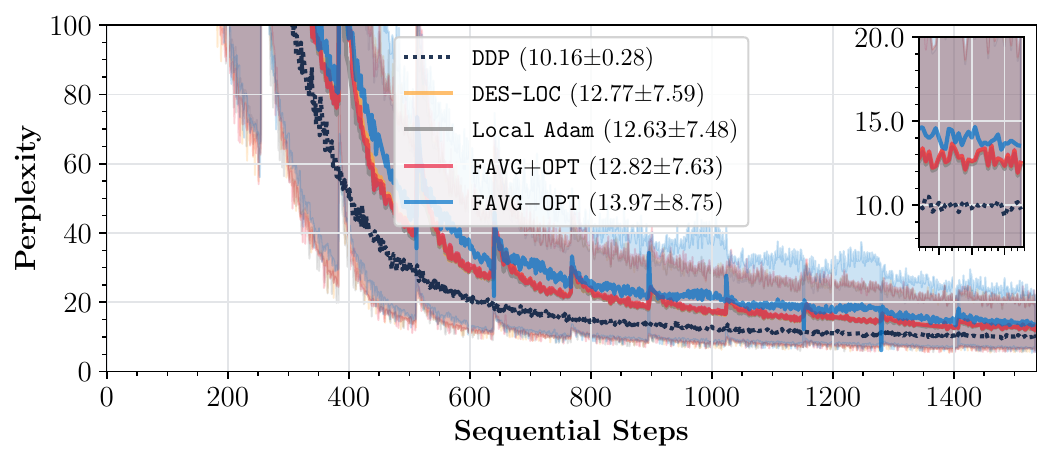}} 
    \caption{Comparison of perplexity under \texttt{Non-IID} conditions for \method, \localadam ($K_x=K_u=K_v$), and heuristic baselines (defined in \cref{sec:exp_setup}) at high (a) and low (b) synchronization frequencies. Due to higher cross-worker variance caused by heterogeneous data, parameters require slightly more frequent synchronization in the low-frequency regime ($K_x=128<256$). Experiments are limited to $T=1536$ steps ($\sim$compute-optimal) for computational feasibility.}
    \label{app:fig:eval:baseline_comparison_niid_data}
\end{figure}

\takeawaybox[Takeaway:]{%
\method effectively converges on heterogeneous data distributions, maintaining the $\mathbf{2}\times$ communication reduction observed in homogeneous settings. This aligns with our theoretical convergence results for heterogeneous losses (\cref{sec:theory}) and shows applicability in federated scenarios.
}

\subsubsection{\method Low Communication Configurations Ablation~(See \cref{fig:eval:baseline_comparison})}\label{app:subsubsec:low_comms_configurations_ablation}

\Cref{app:fig:eval:comms_reduction_ablation} explores alternative synchronization configurations enabling \method to achieve improved communication efficiency over \localadam. Motivated by theoretical insights (\cref{sec:methods,sec:theory}) and empirical evidence (\cref{subsect:eval:rates_of_change,eval:subsec:relative_importance}), we only consider settings where parameter synchronization is most frequent ($K_x \leq \min(K_u,K_v)$). This constraint follows from experiments in \cref{eval:subsec:relative_importance}, which show that infrequent parameter synchronization significantly degrades perplexity, while momentum synchronization frequency has a smaller impact. For a fixed $2\times$ communication reduction over \localadam, our findings confirm that synchronizing the first momentum more frequently than the second aligns with their respective half-lives and maintains performance close to \localadam.

\begin{figure}[H]
    \centering
    \noindent\subfloat[\methodadopt, high sync frequency $K_x=16$.]
    {\includegraphics[width=0.485\columnwidth]{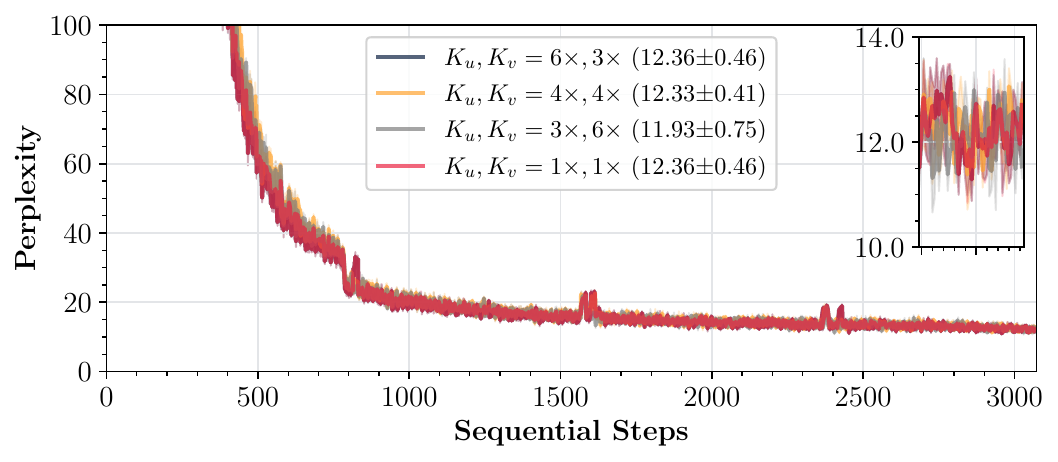}}  \hfill
    \noindent\subfloat[\methodadopt, low sync frequency $K_x=256$.]
    {\includegraphics[width=0.485\columnwidth]{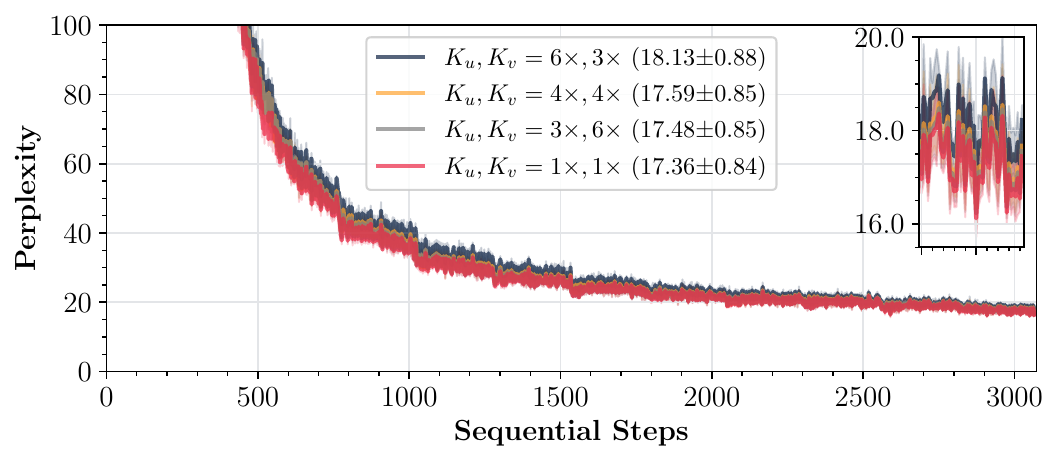}} 
    \caption{Configurations of \method targeting $2\times$ lower communication than \localadam ($K_x=K_u=K_v$), setting $K_u,K_v$ as multiples of $K_x$. In both high (a) and low-frequency (b) regimes, performance depends on how communication is split between momenta for $\beta_1 \ll \beta_2$. Syncing the first momentum less often ($K_u=6K_x,K_v=3K_x$) degrades performance, wasting communication on the slow second momentum. Conversely, syncing it frequently ($K_u=3K_x,K_v=6K_x$) yields performance comparable to \localadam. Setting $K_u=K_v=4 K_x$ produces intermediate results. 
    }
    \label{app:fig:eval:comms_reduction_ablation}
\end{figure}

\takeawaybox[Takeaway:]{%
For a given parameter synchronization period $K_x$ determined by bandwidth constraints, choose momentum synchronization periods $K_u,K_v$ as multiples of $K_x$. When $\beta_1 \ll \beta_2$, set $K_u < K_v$, with $K_u=3 \times K_x$ and $K_v=6 \times K_x$ providing robust default choices.}

\subsubsection{Adam Results~(See \cref{fig:eval:baseline_comparison})}\label{app:subsubsec:baselines_adam_results}

We now present results for \methodadam with $\beta_1=\beta_2=0.95$. \methodadam achieves similar communication reductions over \localadam and \ddp as \adopt. However, due to the lower $\beta_2$, the second-momentum half-life ($\tau_{0.5}(0.95)\approx13.5$) is significantly shorter than for \adopt ($\tau_{0.5}(0.9999)\approx6931$). \Cref{app:fig:eval:comms_reduction_ablation_adam} shows that with both momenta evolving at similar rates, the benefit of selecting $K_u<K_v$ diminishes. For consistency and due to meaningful empirical differences in rates of change (\cref{subsect:eval:rates_of_change}), we keep $K_u=3\times K_x$ and $K_v=6\times K_x$ in subsequent comparisons.

\begin{figure}[H]
    \centering
    \noindent\subfloat[\methodadam, high sync frequency $K_x=16$.]
    {\includegraphics[width=0.485\columnwidth]{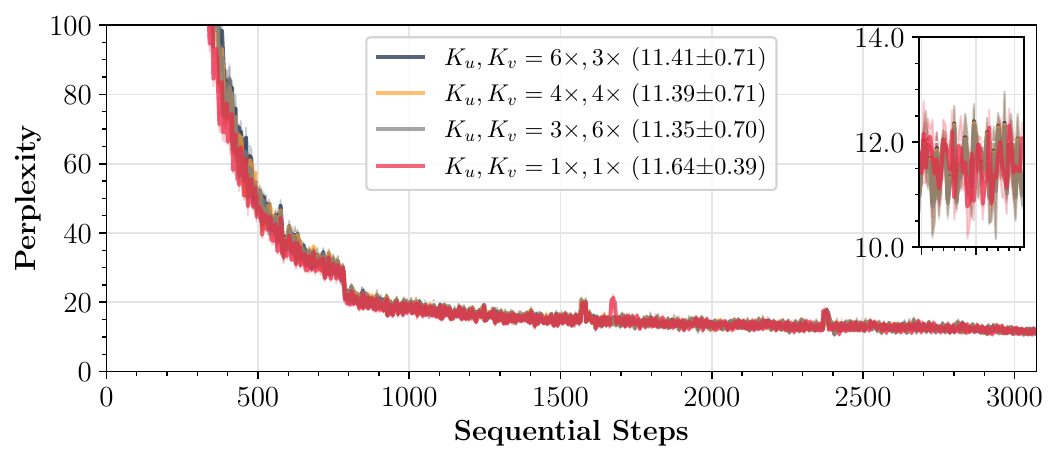}}  \hfill
    \noindent\subfloat[\methodadam, low sync frequency $K_x=256$.]
    {\includegraphics[width=0.485\columnwidth]{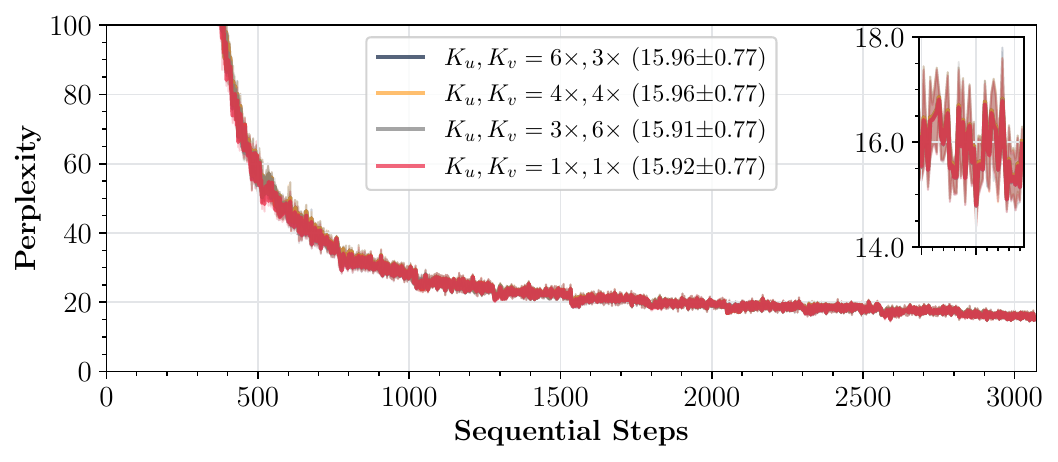}} 
    \caption{Configurations of \method targeting $2\times$ lower communication than \localadam ($K_x=K_u=K_v$), using \adam ($\beta_1=\beta_2=0.95$). In contrast to \methodadopt (where $\beta_1 \ll \beta_2$ yields an advantage for $K_u < K_v$ as shown in \cref{app:fig:eval:comms_reduction_ablation}), the similar half-lives in \adam make perplexity insensitive to how communication is split between momenta for high (a) and low-frequencies (b).}
    \label{app:fig:eval:comms_reduction_ablation_adam}
\end{figure}

\Cref{app:fig:eval:baseline_comparison_adam} shows \methodadam achieves a $2\times$ communication reduction over the prior state-of-the-art \localadam~\citep{LocalAdam} without significant perplexity degradation. Due to the much faster evolution of the optimizer states using \adam compared to \adopt, local worker gradients drive the optimization reducing the benefit of allocating more of the communication budget to the first momentum. 

\begin{figure}[H]
    \centering
    \noindent\subfloat[\methodadam, high sync frequency $K_x=16$.]
    {\includegraphics[width=0.485\columnwidth]{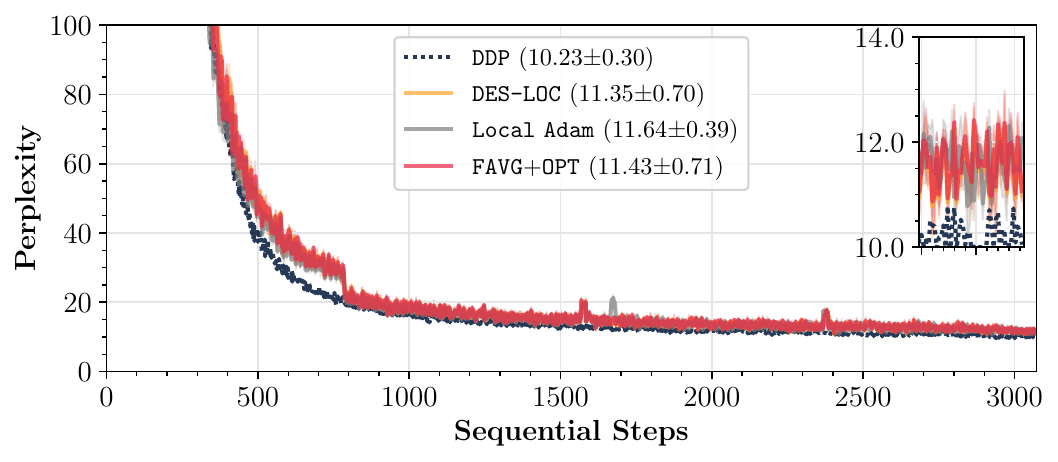}}  \hfill
    \noindent\subfloat[\methodadam, low sync frequency $K_x=256$.]
    {\includegraphics[width=0.485\columnwidth]{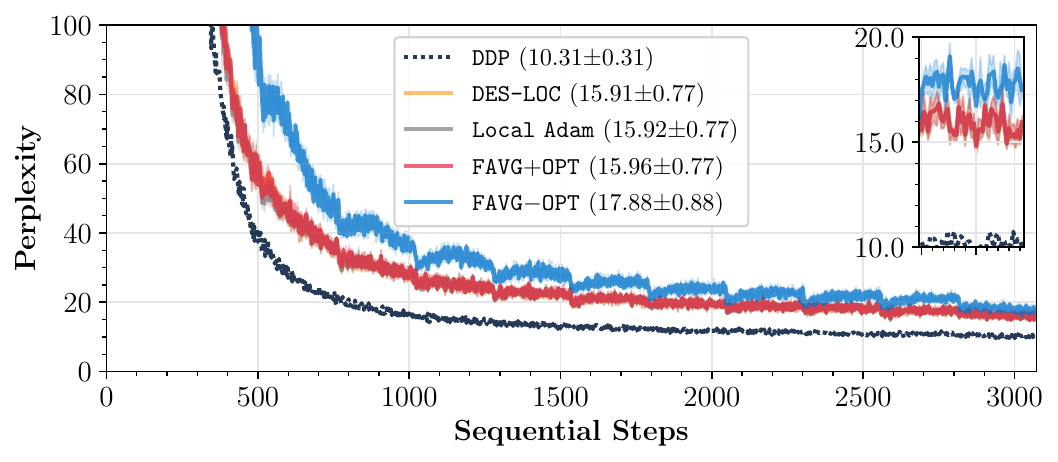}} 
    \caption{Setting $K_x=K$, $K_u=3K_x$, and $K_v=6K_x$, \methodadam achieves a $\mathbf{2}\times$ communication reduction over \localadam, matching performance at high~(a) and low~(b) frequencies for \localadam and heuristic baselines (see \cref{sec:exp_setup}).}
    \label{app:fig:eval:baseline_comparison_adam}
\end{figure}

\takeawaybox[Takeaway:]{%
\methodadam achieves a similar $2\times$ communication reduction over \localadam as \methodadopt by exploiting the reduced importance of optimizer-state synchronization relative to parameters. However, due to the smaller $\beta_2$ in \adam, there is limited benefit from assigning different synchronization frequencies to the first and second momenta compared to \adopt.}

\subsection{\textbf{RQ4:} Additional Metrics and Training Instabilities of \texttt{FAVG}$+$\texttt{OPT}~(See \cref{fig:eval:large_models}.b)} \label{app:subsec:additional_metrics_favg_opt}

\Cref{app:fig:eval:params_and_momentum_large_scale} complements \cref{fig:eval:large_models}.b by showing parameter and update norms for \method and baseline methods when training billion-scale models. 
Both \method and \localadam regularize updates by synchronizing optimizer states, effectively reducing update norms due to averaging across workers (triangle inequality). 
In contrast, the heuristic baseline~\citep{Photon} experiences large updates, leading to uncontrolled parameter growth, increased activations (\cref{fig:eval:large_models}.b), and degraded performance on downstream \texttt{ICL} tasks (\cref{tab:icl-table}) relative to its perplexity~(\cref{fig:eval:large_models}.a).

\begin{figure}[H]
    \centering
      \noindent\subfloat[\methodadopt $1$B-update norms]
    {\includegraphics[width=0.485\columnwidth]{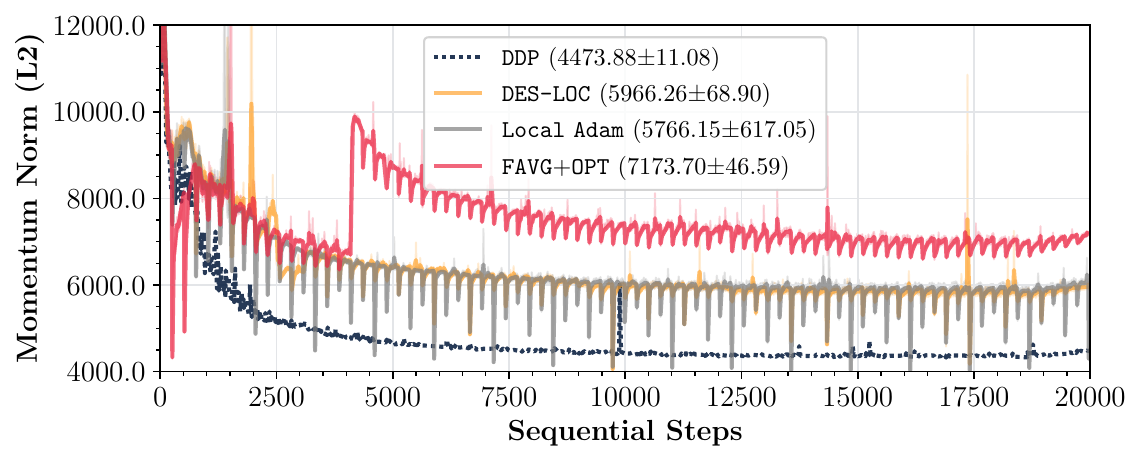}} \hfill
    \noindent\subfloat[\methodadopt $1$B-parameter norms]
    {\includegraphics[width=0.485\columnwidth]{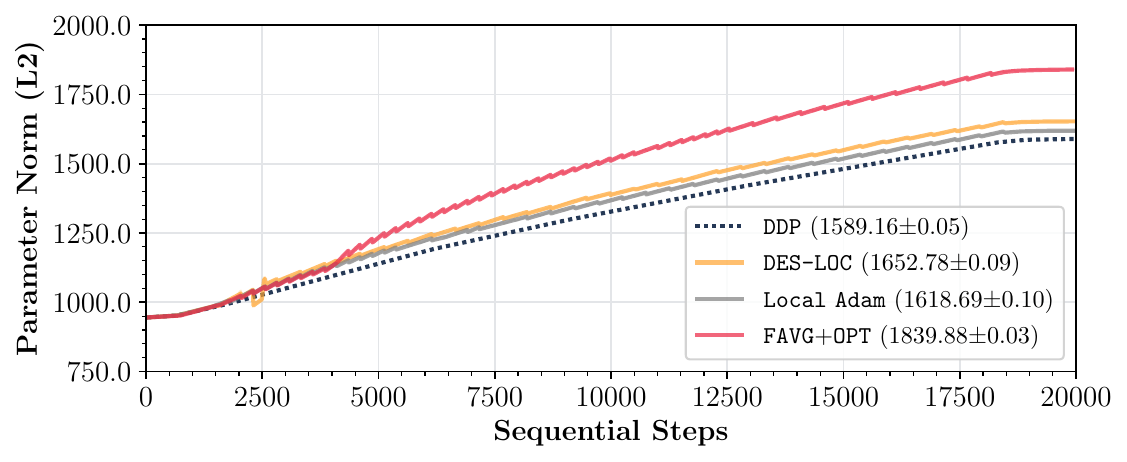}}  
    \caption{Comparison of update~(a) and parameter norms~(b) for billion-scale models trained with \method ($K_x=256,K_u=768,K_v=1536$), \localadam ($K=256$), \ddp, and Federated Averaging with persistent optimizer states (\texttt{FAVG}$+$\texttt{OPT}). Frequent synchronization in \localadam and \ddp consistently reduces update and parameter norms. Similarly, \method achieves comparable reductions at intervals corresponding to multiples of $\mathtt{lcm}(K_x,K_u,K_v)$, with smaller intermediate drops. Conversely, \texttt{FAVG}$+$\texttt{OPT}, which does not synchronize optimizer states, experiences persistently larger and noisier updates, becoming vulnerable to spikes (notably before step $5000$). This leads to uncontrolled parameter growth~(b).}
    \label{app:fig:eval:params_and_momentum_large_scale}
\end{figure}

\takeawaybox[Takeaway:]{%
Unlike heuristic methods, which maintain purely local optimizer states leading to unstable, noisy updates, \method provides stable regularization similar to \localadam and \ddp by periodically synchronizing parameters and momenta, reducing training instabilities.
}

%% file: files/appendices/optimizer_specific_variants.tex
\section{Deterministic Optimizer-specific Variants of \cref{alg:desync_generic}}
\addcontentsline{toc}{subsection}{\methodadam}
\input{algorithms/desync_adam}
\addcontentsline{toc}{subsection}{\methodadopt}
\input{algorithms/desync_adopt}

%% file: algorithms/desync_adam.tex
\begin{algorithm}[H]
\caption{\methodadam}

\label{alg:desync_adam}
\small
\begin{algorithmic}[1]
% ---------------- REQUIRE ------------------------------------------------
  \Require \textbf{Model tensors, Hyper-parameters} \\
          \quad $x_0$,$u_{-1},v_{-1} \in \mathbb{R}^{d}$ — initial parameter vector,seeds for first and second moments \\
          \quad $\{\eta_t\}_{t=0}^{T-1} \subset \mathbb{R}_{>0}$ — step-size schedule \\
          \quad $\beta_{1},\beta_{2} \in [0,1)$ — Adam decay factors \\
          \quad $\rho,\lambda \in \mathbb{R}_{>0}$ — gradient clipping term, $\ell_{2}$ stability term \\
         \quad $T,M \in \mathbb{N}_{+}$ — total iterations, number of workers \\
          \quad \textcolor{blue}{$K_x,K_u,K_v$} $\in \mathbb{N}_{+}$ — sync periods for parameters, first and second moments 
  \Ensure $x_T,\;u_{T-1},v_{T-1}$
% ------------------------------------------------------------------------
  \State \textbf{for each worker} $m$: $x_0^m = x_0,\;u_{-1}^m = v_{-1}^m = 0$
         \hfill\textcolor{gray}{\scriptsize local init ($t=-1$ seeds)}
  \For{$t = 0,\dots,T-1$} \hfill\textcolor{gray}{\scriptsize training loop}
      \ForAll{workers $m=0,\dots,M-1$ \textbf{in parallel}}
% -------- gradient -------------------------------------------------------
    \State $g_t^m \gets \nabla F(x_t^m;\xi_t^m)$
           \hfill\textcolor{gray}{\scriptsize stochastic gradient}
    \State $\widehat{g}_t^m \gets \clip(g_t^m,\rho)$
           \hfill\textcolor{gray}{\scriptsize clip to radius $\rho$}
% -------- first moment ---------------------------------------------------
    \If{\textcolor{blue}{$t \bmod K_u = 0$}} \hfill\textcolor{gray}{\scriptsize sync $u$}
      \State $u_t^m \gets
             \beta_1\,\mathbb{E}_m[u_{t-1}^m] +
             (1-\beta_1)\widehat{g}_t^m$
    \Else
      \State $u_t^m \gets
             \beta_1\,u_{t-1}^m +
             (1-\beta_1)\widehat{g}_t^m$
    \EndIf
% -------- second moment --------------------------------------------------
    \If{\textcolor{blue}{$t \bmod K_v = 0$}} \hfill\textcolor{gray}{\scriptsize sync $v$}
      \State $v_t^m \gets
             \beta_2\,\mathbb{E}_m[v_{t-1}^m] +
             (1-\beta_2)(\widehat{g}_t^m\odot\widehat{g}_t^m)$
    \Else
      \State $v_t^m \gets
             \beta_2\,v_{t-1}^m +
             (1-\beta_2)(\widehat{g}_t^m\odot\widehat{g}_t^m)$
    \EndIf
% -------- Adam direction -------------------------------------------------
    \State $d_t^m \gets
           \dfrac{\eta_t}{\sqrt{v_t^m+\lambda^{2}}}\odot u_t^m$
           \hfill\textcolor{gray}{\scriptsize bias-corrected step}
% -------- parameter update -----------------------------------------------
    \If{\textcolor{blue}{$t \bmod K_x = 0$}} \hfill\textcolor{gray}{\scriptsize sync $x$}
      \State $x_{t+1}^m \gets \mathbb{E}_m[x_t^m] - d_t^m$
    \Else
      \State $x_{t+1}^m \gets x_t^m - d_t^m$
    \EndIf
    \EndFor
  \EndFor
\end{algorithmic}
\end{algorithm}

%% file: algorithms/desync_adopt.tex
\begin{algorithm}[H]
\caption{\methodadopt}
\label{alg:desync_adopt}
\small
\begin{algorithmic}[1]
% ---------------- REQUIRE ------------------------------------------------
  \Require \textbf{Model tensors,Hyper-parameters} \\
          \quad $x_0,m_{-1}, v_{-1} \in \mathbb{R}^{d}$ --- initial parameter vector and momenta \\
          \quad $\{\eta_t\}_{t=0}^{T-1} \subset \mathbb{R}_{>0}$ --- learning rate schedule \\
          \quad $\beta_{1},\beta_{2} \in [0,1)$ --- decay factors \\
          \quad $\rho,\epsilon \in \mathbb{R}_{>0}$ --- gradient clipping term,small stability constant \\
         \quad $T,M \in \mathbb{N}_{+}$ --- total iterations,number of workers \\
          \quad \textcolor{blue}{$K_x,K_m,K_v$} $\in \mathbb{N}_{+}$ --- sync periods for parameters, first and second moments 
  \Ensure $x_T,\;m_{T-1},v_{T-1}$
% ------------------------------------------------------------------------
  \State \textbf{for each worker} $m$: $x_0^m = x_0,\;m_{-1}^m = v_{-1}^m = 0$
         \hfill\textcolor{gray}{\scriptsize local initialization}
  \For{$t = 0,\dots,T-1$}
      \ForAll{workers $m=0,\dots,M-1$ \textbf{in parallel}}
% -------- gradient -------------------------------------------------------
    \State $g_t^m \gets \nabla F(x_t^m;\xi_t^m)$
           \hfill\textcolor{gray}{\scriptsize stochastic gradient}
    \State $\widehat{g}_t^m \gets \clip(g_t^m,\rho)$
           \hfill\textcolor{gray}{\scriptsize gradient clipping}
% -------- second moment --------------------------------------------------
    \If{\textcolor{blue}{$t \bmod K_v = 0$}}
      \State $v_t^m \gets
             \beta_2\,\mathbb{E}_m[v_{t-1}^m] +
             (1-\beta_2)(\widehat{g}_t^m\odot\widehat{g}_t^m)$
    \Else
      \State $v_t^m \gets
             \beta_2\,v_{t-1}^m +
             (1-\beta_2)(\widehat{g}_t^m\odot\widehat{g}_t^m)$
    \EndIf
% -------- first moment ---------------------------------------------------
    \If{\textcolor{blue}{$t \bmod K_m = 0$}}
      \State $m_t^m \gets
             \beta_1\,\mathbb{E}_m[m_{t-1}^m] +
             (1-\beta_1)\frac{\widehat{g}_t^m}{\max\{\sqrt{v_{t-1}^m},\epsilon\}}$
    \Else
      \State $m_t^m \gets
             \beta_1\,m_{t-1}^m +
             (1-\beta_1)\frac{\widehat{g}_t^m}{\max\{\sqrt{v_{t-1}^m},\epsilon\}}$
    \EndIf
% -------- parameter update -----------------------------------------------
    \State $d_t^m \gets \eta_t m_t^m$
           \hfill\textcolor{gray}{\scriptsize ADOPT update}
    \If{\textcolor{blue}{$t \bmod K_x = 0$}}
      \State $x_{t+1}^m \gets \mathbb{E}_m[x_t^m] - d_t^m$
    \Else
      \State $x_{t+1}^m \gets x_t^m - d_t^m$
    \EndIf
    \EndFor
  \EndFor
\end{algorithmic}
\end{algorithm}

%% file: files/appendices/proof-sgdm.tex
\clearpage

\section{Convergence Analysis of \methodsgdm (in expectation bounds)}\label{app:proof-sgdm}

Here we provide a non-convex convergence analysis of the proposed \method approach applied to the SGDM optimizer which has a single state ($N=1$, momentum). The complete description of the algorithm can be found in Algorithm \ref{alg:desync_sgdm}.

\input{algorithms/desync_sgdm}

In order to facilitate the technical presentation, we model synchronization frequencies by assigning probabilities to each averaging event. For example, the parameters $x_t^m$ are synchronized with the probability $p_x = \frac{1}{K_x}$, which is statistically equivalent to performing the averaging in every $\frac{1}{p_x} = K_x$ iteration. Similarly, momentum $u_t^m$ synchronization happens with probability $p_u = \frac{1}{K_u}$, which can differ from $p_x$.

\underline{Step 1 (virtual iterates).} For each step $t\ge0$, denote the average parameters, momentum and gradient as follows:
$$
x_t \eqdef \E_m[x_t^m],
\quad
u_t \eqdef \E_m[u_t^m],
\quad
g_t \eqdef \E_m[g_t^m].
$$
Then these averaged variables follow the ``standard'' centralized SGDM dynamics:
\begin{eqnarray*}
u_t &=& \beta u_{t-1} + (1-\beta)g_{t} \\
x_{t+1} &=& x_{t} - \eta u_t.
\end{eqnarray*}
Letting $x_{-1} = x_0$, define the global virtual iterations as follows
$$
z_t \eqdef \frac{1}{1-\beta} x_t - \frac{\beta}{1-\beta} x_{t-1},
\quad t\ge0.
$$
The key property of this virtual iterates we are going to exploit in the next steps is that they follow averaged gradients, namely for any $t\ge0$ we have
\begin{eqnarray*}
z_{t+1} - z_t
&=& \frac{1}{1-\beta}(x_{t+1} - x_{t}) - \frac{\beta}{1-\beta}(x_{t} - x_{t-1}) \\
&=& -\frac{\eta}{1-\beta} u_t + \frac{\eta\beta}{1-\beta}u_{t-1}
= -\frac{\eta}{1-\beta} (u_t - \beta u_{t-1})
= - \eta g_t.
\end{eqnarray*}

\underline{Step 2 (smoothness over virtual iterates).} Then we apply smoothness of the global loss function $f$ over these global virtual iterates.
\begin{eqnarray*}
f(z_{t+1})
&\le& f(z_t) + \langle \nabla f(z_t), z_{t+1} - z_t \rangle + \frac{L}{2}\|z_{t+1} - z_t\|^2 \\
&=& f(z_t)
+ \underbrace{\langle \nabla f(x_t), z_{t+1} - z_t \rangle}_{I}
+ \underbrace{\langle \nabla f(z_t) - \nabla f(x_t), z_{t+1} - z_t \rangle}_{II}
+ \underbrace{\frac{L}{2}\|z_{t+1} - z_t\|^2}_{III}.
\end{eqnarray*}

In the next step, we separately bound each term appearing in the above bound.

\underline{Step 3a (one step progress).} Bounding term I.
\begin{eqnarray*}
&& \E{\langle \nabla f(x_t), z_{t+1} - z_t \rangle} \\
&=& -\eta \E{\left\langle \nabla f(x_t), \frac{1}{M}\sum_{m=1}^M g_t^m \right\rangle}
= -\eta \E{\left\langle \nabla f(x_t), \frac{1}{M}\sum_{m=1}^M \nabla f_m(x_t^m) \right\rangle} \\
&=& -\frac{\eta}{2}\E{\|\nabla f(x_t)\|^2} - \frac{\eta}{2}\E{\left\|\frac{1}{M}\sum_{m=1}^M \nabla f_m(x_t^m) \right\|^2} + \frac{\eta}{2}\E{\left\| \nabla f(x_t) - \frac{1}{M}\sum_{m=1}^M \nabla f_m(x_t^m) \right\|^2} \\
&=& -\frac{\eta}{2}\E{\|\nabla f(x_t)\|^2} - \frac{\eta}{2}\E{\left\|\frac{1}{M}\sum_{m=1}^M \nabla f_m(x_t^m) \right\|^2} + \frac{\eta}{2}\E{\left\| \frac{1}{M}\sum_{m=1}^M \nabla f_m(x_t) - \nabla f_m(x_t^m) \right\|^2} \\
&\le& -\frac{\eta}{2}\E{\|\nabla f(x_t)\|^2} - \frac{\eta}{2}\E{\left\|\frac{1}{M}\sum_{m=1}^M \nabla f_m(x_t^m) \right\|^2} + \frac{\eta}{2M}\sum_{m=1}^M\E{\left\| \nabla f_m(x_t) - \nabla f_m(x_t^m) \right\|^2} \\
&\le& -\frac{\eta}{2}\E{\|\nabla f(x_t)\|^2} - \frac{\eta}{2}\E{\left\|\frac{1}{M}\sum_{m=1}^M \nabla f_m(x_t^m) \right\|^2} + \frac{\eta L^2}{2M}\sum_{m=1}^M \underbrace{\E{\left\| x_t - x_t^m \right\|^2}}_{\rm Lemma\;\ref{lem:x-xm}}.
\end{eqnarray*}

\underline{Step 3b (one step progress).} Bounding term II.
\begin{eqnarray*}
\E{\langle \nabla f(z_t) - \nabla f(x_t), z_{t+1} - z_t \rangle}
&=& -\eta \E{\left\langle \nabla f(z_t) - \nabla f(x_t), \frac{1}{M}\sum_{m=1}^M \nabla f_m(x_t^m) \right\rangle} \\
&\le& \frac{\eta\rho}{2} \E{\| \nabla f(z_t) - \nabla f(x_t)\|^2} + \frac{\eta}{2\rho} \E{\left\|\frac{1}{M}\sum_{m=1}^M \nabla f_m(x_t^m) \right\|^2} \\
&\le& \frac{\eta\rho L^2}{2} \underbrace{\E{\| z_t - x_t \|^2}}_{\rm Lemma\;\ref{lem:z-x}} + \frac{\eta}{2\rho} \E{\left\|\frac{1}{M}\sum_{m=1}^M \nabla f_m(x_t^m) \right\|^2}.
\end{eqnarray*}

\underline{Step 3c (one step progress).} Bounding term III.
\begin{eqnarray*}
\frac{L}{2}\E{\|z_{t+1} - z_t\|^2}
&=& \frac{\eta^2L}{2}\E{\left\| \frac{1}{M}\sum_{m=1}^M g_t^m \right\|^2} \\
&=& \frac{\eta^2L}{2}\E{\left\| \frac{1}{M}\sum_{m=1}^M g_t^m - \nabla f_m(x_t^m) \right\|^2} + \frac{\eta^2L}{2}\E{\left\| \frac{1}{M}\sum_{m=1}^M \nabla f_m(x_t^m) \right\|^2} \\
&=& \frac{\eta^2L}{2M^2}\sum_{m=1}^M\E{\left\| g_t^m - \nabla f_m(x_t^m) \right\|^2} + \frac{\eta^2L}{2}\E{\left\| \frac{1}{M}\sum_{m=1}^M \nabla f_m(x_t^m) \right\|^2} \\
&\le& \frac{\eta^2L}{2M}\sigma^2 + \frac{\eta^2L}{2}\E{\left\| \frac{1}{M}\sum_{m=1}^M \nabla f_m(x_t^m) \right\|^2}.
\end{eqnarray*}

\underline{Step 3abc (one step progress).} Combining previous bounds.
\begin{eqnarray*}
\E{f(z_{t+1})} - \E{f(z_t)}
&\le& \E \underbrace{\langle \nabla f(x_t), z_{t+1} - z_t \rangle}_{I}
+ \E \underbrace{\langle \nabla f(z_t) - \nabla f(x_t), z_{t+1} - z_t \rangle}_{II}
+ \E \underbrace{\frac{L}{2}\|z_{t+1} - z_t\|^2}_{III} \\
&\le& -\frac{\eta}{2}\E{\|\nabla f(x_t)\|^2} - \frac{\eta}{2}\E{\left\|\frac{1}{M}\sum_{m=1}^M \nabla f_m(x_t^m) \right\|^2} + \frac{\eta L^2}{2M}\sum_{m=1}^M \underbrace{\E{\left\| x_t - x_t^m \right\|^2}}_{\rm Lemma\;\ref{lem:x-xm}} \\
&& +\; \frac{\eta\rho L^2}{2} \underbrace{\E{\| z_t - x_t \|^2}}_{\rm Lemma\;\ref{lem:z-x}} + \frac{\eta}{2\rho} \E{\left\|\frac{1}{M}\sum_{m=1}^M \nabla f_m(x_t^m) \right\|^2} \\
&& +\; \frac{\eta^2L}{2K}\sigma^2 + \frac{\eta^2L}{2}\E{\left\| \frac{1}{M}\sum_{m=1}^M \nabla f_m(x_t^m) \right\|^2} \\
&\le& -\frac{\eta}{2}\E{\|\nabla f(x_t)\|^2}
      - \frac{\eta}{2} \left( 1 - \frac{1}{\rho} - \eta L \right)\E{\left\|\frac{1}{M}\sum_{m=1}^M \nabla f_m(x_t^m) \right\|^2} \\
&& +\; \frac{\eta\rho L^2}{2} \underbrace{\E{\| z_t - x_t \|^2}}_{\rm Lemma\;\ref{lem:z-x}} + \frac{\eta L^2}{2M}\sum_{m=1}^M \underbrace{\E{\left\| x_t - x_t^m \right\|^2}}_{\rm Lemma\;\ref{lem:x-xm}} + \frac{\eta^2L}{2M}\sigma^2.
\end{eqnarray*}

\underline{Step 4 (final).} Now we average over the iterates and apply the bounds derived in Lemmas 1,2.
\begin{eqnarray*}
\frac{\E[f(z_T) - f(z_0)]}{T}
&=& \frac{1}{T}\sum_{t=0}^{T-1} \E[f(z_{t+1}) - f(z_t)] \\
&\le& -\frac{\eta}{2T}\sum_{t=0}^{T-1}\E{\|\nabla f(x_t)\|^2}
      - \frac{\eta}{2} \left( 1 - \frac{1}{\rho} - \eta L \right) \frac{1}{T}\sum_{t=0}^{T-1}\E{\left\|\frac{1}{M}\sum_{m=1}^M \nabla f_m(x_t^m) \right\|^2} \\
&& +\; \frac{\eta\rho L^2}{2} \underbrace{\frac{1}{T}\sum_{t=0}^{T-1}\E{\| z_t - x_t \|^2}}_{\rm Lemma\;1}
+ \frac{\eta L^2}{2} \underbrace{\frac{1}{TM}\sum_{t=0}^{T-1}\sum_{m=1}^M\E{\left\| x_t - x_t^m \right\|^2}}_{\rm Lemma\;2} + \frac{\eta^2L}{2M}\sigma^2 \\
&\le& -\frac{\eta}{2T}\sum_{t=0}^{T-1}\E{\|\nabla f(x_t)\|^2}
      - \frac{\eta}{2} \left( 1 - \frac{1}{\rho} - \eta L \right) \frac{1}{T}\sum_{t=0}^{T-1}\E{\left\|\frac{1}{M}\sum_{m=1}^M \nabla f_m(x_t^m) \right\|^2}
       + \frac{\eta^2L}{2M}\sigma^2 \\
&& +\; \frac{\eta\rho L^2}{2} \left( \frac{\eta^2\beta^2}{(1-\beta)^2 M}\sigma^2 + \frac{\eta^2\beta^2}{(1-\beta)^2} \frac{1}{T}\sum_{\tau=0}^{T-1} \E\left\| \frac{1}{M}\sum_{m=1}^M \nabla f_m(x_{\tau}^m) \right\|^2 \right) \\
&& +\;   \frac{\eta L^2}{2} \left( 12\eta^2 (B^2-1) \psi \cdot \frac{1}{T}\sum_{t=0}^{T-1}\E\|\nabla f(\theta^{t})\|^2
+ 4\eta^2\psi(\sigma^2 + 3G^2) \right) \\
&\le& -\frac{\eta}{2}\left(1 - 12\eta^2L^2(B^2-1)\psi\right) \frac{1}{T}\sum_{t=0}^{T-1}\E{\|\nabla f(x_t)\|^2} \\
&& -\; \frac{\eta}{2} \left( 1 - \frac{1}{\rho} - \eta L - \frac{\eta^2\beta^2 \rho L^2}{(1-\beta)^2} \right) \frac{1}{T}\sum_{t=0}^{T-1}\E{\left\|\frac{1}{M}\sum_{m=1}^M \nabla f_m(x_t^m) \right\|^2} \\
&& +\; \frac{\eta^2L}{2M}\sigma^2 
+ \frac{\eta^3\rho L^2 \beta^2}{2(1-\beta)^2 M}\sigma^2 
+ 2\eta^3 L^2\psi(\sigma^2 + 3G^2).
\end{eqnarray*}
Next, we choose $\rho=2$ and step size $\eta$ such that
\begin{eqnarray*}
12\eta^2L^2(B^2-1)\psi \le \frac{1}{2} &\iff& \text{to bound the first term} \\
\eta L + \frac{2\eta^2\beta^2 L^2}{(1-\beta)^2} \le \frac{1}{2} &\iff& \text{to bound the second term} \\
12\eta^2 L^2 \psi \le \frac{1}{2} &\iff& \text{from Lemma~\ref{lem:x-xm}}
\end{eqnarray*}
Note that
$$
\eta_0 \eqdef \frac{1}{4L}\min\left(1-\beta, \frac{1}{6\sqrt{\psi\max(1,B^2-1)}} \right)
$$
satisfies all three bounds. Then, with any $\eta\le\eta_0$ we get
\begin{eqnarray*}
\frac{\E[f(z_T) - f(z_0)]}{T}
&\le& -\frac{\eta}{4T}\sum_{t=0}^{T-1}\E{\|\nabla f(x_t)\|^2} \\
&& +\; \frac{\eta^2L}{2M}\sigma^2 
+ \frac{\eta^3\rho L^2 \beta^2}{2(1-\beta)^2 M}\sigma^2 
+ 2\eta^3 L^2\psi(\sigma^2 + 3G^2).
\end{eqnarray*}
Noticing that $z_0=x_0$ and $f^* \le f(z_T)$, we have
\begin{eqnarray*}
\frac{1}{T}\sum_{t=0}^{T-1}\E{\|\nabla f(x_t)\|^2}
\le \frac{4(f(x_0) - f^*)}{\eta T}
+ \frac{2\eta L}{M}\sigma^2 
+ \frac{4\eta^2 L^2 \beta^2}{(1-\beta)^2 M}\sigma^2 
+ 8\eta^2 L^2\psi(\sigma^2 + 3G^2).
\end{eqnarray*}

Furthermore, choosing $\eta = \min(\eta_0, \frac{1}{\sqrt{T}})$, we get the following rate:
\begin{eqnarray*}
&& \frac{1}{T}\sum_{t=0}^{T-1}\E{\|\nabla f(x_t)\|^2} \\
&\le& \max\left(1,\frac{1}{\eta_0\sqrt T}\right) \frac{4(f(x_0) - f^*)}{\sqrt T}
+ \frac{2L\sigma^2 }{M\sqrt{T}}
+ \frac{4 L^2 \beta^2 \sigma^2 }{(1-\beta)^2 MT}
+ \frac{8 L^2\psi(\sigma^2 + 3G^2)}{T} \\
&\le& \frac{4(f(x_0) - f^*)}{\sqrt T}
+ \frac{2L\sigma^2 }{M\sqrt{T}}
+ \frac{4(f(x_0) - f^*)}{\eta_0 T}
+ \frac{4 L^2 \beta^2 \sigma^2 }{(1-\beta)^2 MT}
+ \frac{8 L^2\psi(\sigma^2 + 3G^2)}{T} \\
&=& \frac{4}{\sqrt{T}}\left(f(x_0) - f^*
+ \frac{L\sigma^2 }{2M} \right)
+ \mathcal{O}\left(\frac{1+\psi}{T}\right).
\end{eqnarray*}

% ======================================================
% ======================================================
\subsection{Extension to Adam optimizer}

Here we discuss extension of the previous analysis for the Adam optimizer including the second-order momentum in the analysis. The addition is similar to the first-order momentum while the synchronization probability $p_v$ can differ from other probabilities $p_u$ and $p_u$. The complete description of the
algorithm can be found in Algorithm~\ref{alg:desync_adam_prob}. Instead of bounded heterogeneity Assumption \ref{ass:het}, in this analysis we use stronger condition mentioned below: 

\begin{asp}[Bounded gradient]\label{ass:bounded-grad}
    For any iterate $t\ge0$ and worker $m$, the local stochastic gradient is bounded, namely $\|g_t^m\|_2 \le G$.
\end{asp}

This condition facilitates the analysis by providing uniform upper bounds for gradients/momentum variables and is commonly used in the analysis of adaptive optimization.

\input{algorithms/desync_adam_prob}

\underline{Step 1 (preconditioning and virtual iterates).}
Let $\Gamma_t^m \eqdef \diag^{-\nicefrac{1}{2}}(\tilde{v}_t^m + \lambda^2)$ be the preconditioning matrix and for each step $t\ge0$, denote the averaged variables
$$
x_t \eqdef \E_m[x_t^m],
\quad
u_t \eqdef \E_m[u_t^m],
\quad
v_t \eqdef \E_m[v_t^m],
\quad
\tilde v_t \eqdef \E_m[\tilde v_t^m],
\quad
g_t \eqdef \E_m[g_t^m].
$$
Then
% these averaged variables follow the ``standard'' centralized AMSGrad dynamics:
\begin{eqnarray*}
u_t &=& \beta_1 u_{t-1} + (1-\beta_1) g_{t} \\
% v_t &=& \beta_2 v_{t-1} + (1-\beta_2) (g_{t} \odot g_t) \\
% \tilde v_t &=& \max(v_t, \tilde v_{t-1}) \\
x_{t+1} &=& x_{t} - d_t = x_t - \eta \E_m[\Gamma_t^m u_t^m].
\end{eqnarray*}

Consider the same averaged iterates $x_t$ and virtual iterates $z_t$ as before:
\begin{equation*}\label{def:virtual-iter}
z_{t}
= \frac{1}{1-\beta_1} x_{t} - \frac{\beta_1}{1-\beta_1} x_{t-1}.
\end{equation*}
In particular, $z_0 = x_0$. Then,
\begin{eqnarray*}
z_{t+1} - z_t
&=& \frac{1}{1-\beta_1}(x_{t+1}-x_t) - \frac{\beta_1}{1-\beta_1}(x_{t} - x_{t-1}) \\
&=& -\frac{\eta}{1-\beta_1}\E_m[\Gamma_t^m u_t^m] + \frac{\eta\beta_1}{1-\beta_1} \E_m[\Gamma_{t-1}^m u_{t-1}^m] \\
&=& -\frac{\eta}{1-\beta_1}\E_m[\Gamma_t^m u_t^m] + \frac{\eta\beta_1}{1-\beta_1} \E_m[\Gamma_{t-1}^m u_{t-1}^m] \pm \frac{\eta\beta_1}{1-\beta_1}\E_m[\Gamma_t^m u_{t-1}^m] \\
&=& -\frac{\eta}{1-\beta_1} \E_m[ \Gamma_t^m (u_t^m - \beta_1 u_{t-1}^m)] + \frac{\eta\beta_1}{1-\beta_1} \E_m[(\Gamma_{t-1}^m - \Gamma_{t}^m )u_{t-1}^m] \\
% z_{t+1} - z_t
% &=& -\eta \E_m[\Gamma_t^m \widetilde{g}_t^m] + \frac{\eta\beta_1}{1-\beta_1} \E_m[(\Gamma_{t-1}^m - \Gamma_{t}^m ) u_{t-1}^m] \\
% &=& -\eta \E_m[\Gamma_t \widetilde{g}_t^m] + \eta \E_m[(\Gamma_t - \Gamma_t^m) \widetilde{g}_t^m] + \frac{\eta\beta_1}{1-\beta_1} \E_m[(\Gamma_{t-1}^m - \Gamma_{t}^m ) u_{t-1}^m] \\
% &=& -\eta \Gamma_t g_t + \eta \E_m[(\Gamma_t - \Gamma_t^m) \widetilde{g}_t^m] + \frac{\eta\beta_1}{1-\beta_1} \E_m[(\Gamma_{t-1}^m - \Gamma_{t}^m ) u_{t-1}^m], \\
% z_{t+1} - z_t
&=& -\eta \E_m[\Gamma_t^m \widetilde{g}_t^m] + \frac{\eta\beta_1}{1-\beta_1} \E_m[(\Gamma_{t-1}^m - \Gamma_{t}^m ) u_{t-1}^m] \\
&=& -\eta \E_m[\Gamma_t^m g_t] + \eta \E_m[\Gamma_t^m (g_t - \widetilde{g}_t^m)] + \frac{\eta\beta_1}{1-\beta_1} \E_m[(\Gamma_{t-1}^m - \Gamma_{t}^m ) u_{t-1}^m] \\
&=& -\eta \Gamma_t g_t + \eta \cdot \underbrace{\E_m[\Gamma_t^m (g_t - \widetilde{g}_t^m)]}_{\eqdef U_t} + \eta\cdot \underbrace{\frac{\beta_1}{1-\beta_1} \E_m[(\Gamma_{t-1}^m - \Gamma_{t}^m ) u_{t-1}^m]}_{\eqdef V_t},
\end{eqnarray*}
where $\Gamma_t \eqdef \E_m[\Gamma_t^m]$ and $\widetilde{g}_t^m \eqdef \frac{u_t^m - \beta_1 u_{t-1}^m}{1-\beta_1}$
% \begin{equation}\label{pseudo-grad}
%     \widetilde{g}_t^m \eqdef \frac{u_t^m - \beta_1 u_{t-1}^m}{1-\beta_1} =
%     \begin{cases}
%         g_t^m & \text{with probability } 1-p_u, \\
%         \frac{u_t - \beta_1 u_{t-1}^m}{1-\beta_1} & \text{with probability } p_u,
%     \end{cases}
% \end{equation}
for which, $\E_m[\widetilde{g}_t^m] = \E_m[g_t^m] = g_t$.

\underline{Step 2 (smoothness over virtual iterates).} Then we apply smoothness of the global loss function $f$ over these global virtual iterates.
\begin{eqnarray*}
f(z_{t+1}) - f(z_t)
&\le& \langle \nabla f(z_t), z_{t+1} - z_t \rangle + \frac{L}{2}\|z_{t+1} - z_t\|^2 \\
&=&  -\eta \left\langle \nabla f(z_t), \Gamma_t g_t \right\rangle
+ \eta\langle \nabla f(z_t), U_t \rangle
+ \eta\langle \nabla f(z_t), V_t \rangle
+ \frac{L}{2}\|z_{t+1} - z_t\|^2 \\
&=&  \underbrace{-\eta \left\langle \nabla f(x_t), \Gamma_t g_t \right\rangle}_{I}
+ \underbrace{\eta\langle \nabla f(z_t), U_t \rangle}_{II}
+ \underbrace{\eta \left\langle \nabla f(z_t), V_t \right\rangle}_{III} \\
&& +\; \underbrace{\frac{\eta^2 L}{2}\left\| \Gamma_tg_t - U_t -V_t \right\|^2}_{IV}
+ \underbrace{\eta \left\langle \nabla f(x_t) - \nabla f(z_t), \Gamma_tg_t \right\rangle}_{V}.
\end{eqnarray*}

In the next step, we separately bound each term appearing in the above bound. For clarity, we are also going to use $\|\nabla f(x_t)\| \le G$ and $\|\nabla f(z_t)\| \le G$. However, these conditions can be avoided through linking $\nabla f(z_t)$ term to $\nabla f(x_t)$, and $\nabla f(x_t)$ term to $\E_m\nabla f_m(x_t^m)$ with the bound for $\E[\|x_t-x_t^m\|^2]$.

\underline{Step 3a (one step progress).} Bounding term I.

\begin{eqnarray*}
    I&=&
    -\eta \left\langle \nabla f(x_t), \Gamma_t g_t] \right\rangle \\
    &=& -\eta\E\left[\left\langle \nabla f(x_t), \Gamma_{t-1} g_t \right\rangle\right] + \eta\E\left[\left\langle \nabla f(x_t), (\Gamma_{t-1} - \Gamma_{t})g_t\right\rangle\right] \nonumber\\
    &\leq& -\eta\E\left[\left\langle \nabla f(x_t), \frac{1}{M}\sum_{m=1}^M \nabla f_m(x_t^m)\right\rangle_{\Gamma_{t-1}}\right]+\eta G^2\E[\|\Gamma_{t-1} - \Gamma_{t}\|].  \nonumber\\
    &\leq& - \frac{\eta}{2}\E\left[\|\nabla f(x_t)\|^2_{\Gamma_{t-1}}\right]
    - \frac{\eta}{2}\E\left[\left\|\frac{1}{M}\sum_{m=1}^M\nabla f_m(x_t^m)\right\|^2_{\Gamma_{t-1}}\right] \\
    && +\; \frac{\eta}{2}\E\left[\left\| \nabla f(x_t) - \frac{1}{M}\sum_{m=1}^M\nabla f_m(x_t^m)\right\|^2_{\Gamma_{t-1}}\right]
    + \eta G^2\E[\|\Gamma_{t-1} - \Gamma_{t}\|]  \nonumber\\
    &\leq& - \frac{\eta}{2}\|\Gamma_{t-1}\|_{\min}\E\|\nabla f(x_t)\|^2
    - \frac{\eta}{2}\E\left[\left\|\frac{1}{M}\sum_{m=1}^M\nabla f_m(x_t^m)\right\|^2_{\Gamma_{t-1}}\right] \\
    && +\; \frac{\eta}{2}\|\Gamma_{t-1}\|_{\max}\E\left[\left\| \frac{1}{M}\sum_{m=1}^M \nabla f_m(x_t) - \nabla f_m(x_t^m) \right\|^2\right]
    + \eta G^2\E[\|\Gamma_{t-1} - \Gamma_{t}\|]  \nonumber\\
    &\leq& - \frac{\eta}{2C_0}\E\|\nabla f(x_t)\|^2
    - \frac{\eta}{2}\E\left[\left\|\frac{1}{M}\sum_{m=1}^M\nabla f_m(x_t^m)\right\|^2_{\Gamma_{t-1}}\right] \\
    && +\; \frac{\eta}{2\lambda M} \sum_{m=1}^M \E\left[\left\|\nabla f_m(x_t) - \nabla f_m(x_t^m) \right\|^2\right]
    + \eta G^2\E[\|\Gamma_{t-1} - \Gamma_{t}\|]  \nonumber\\
    &\leq& - \frac{\eta}{2C_0}\E\|\nabla f(x_t)\|^2
    + \frac{\eta L^2}{2\lambda M} \sum_{m=1}^M \E\left[\|x_t - x_t^m\|^2\right]
    + \eta G^2\E[\|\Gamma_{t-1} - \Gamma_{t}\|],  \nonumber\\
    % \label{eq:I}
\end{eqnarray*}
where $\|\cdot\|$ indicates the spectral norm for matrices, and we used the following inequalities:
$$
\|\Gamma_{t-1}\|_{\min}
= \left\|\frac{1}{M}\sum_{m=1}^M \Gamma_{t-1}^m\right\|_{\min}
= \frac{1}{M}\sum_{m=1}^M \Gamma_{t-1}^m[i,i]
= \frac{1}{M}\sum_{m=1}^M \frac{1}{\sqrt{\tilde v_{t-1}[i] + \lambda^2}}
\ge \frac{1}{\sqrt{G^2 + \lambda^2} } \eqdef \frac{1}{C_0}.
$$

\underline{Step 3b (one step progress).} Bounding term II.
\begin{eqnarray*}
    II&=&
    \eta\langle \nabla f(z_t), U_t \rangle
    \le \eta \|\nabla f(z_t)\| \|U_t\|
    \le \frac{\eta G}{M}\sum_{m=1}^M \|\Gamma_t^m(g_t - \widetilde{g}_t^m)\| \\
    &\le& \frac{\eta G}{\lambda M}\sum_{m=1}^M \|g_t - \widetilde{g}_t^m\|.
\end{eqnarray*}

\underline{Step 3c (one step progress).} Bounding term III.
\begin{eqnarray*}
    III&=&
    \eta\langle \nabla f(z_t), V_t \rangle
    \le \eta \|\nabla f(z_t)\| \|V_t\|
    \le \frac{\eta\beta_1}{1-\beta_1}\frac{G}{M}\sum_{m=1}^M \|(\Gamma_{t-1}^m-\Gamma_t^m)u_{t-1}^m\| \\
    &\le& \frac{\eta\beta_1}{1-\beta_1}\frac{G^2}{M}\sum_{m=1}^M \|\Gamma_{t-1}^m-\Gamma_t^m\|.
\end{eqnarray*}

\underline{Step 3d (one step progress).} Bounding term IV.
\begin{eqnarray*}
    IV&=&
    \frac{\eta^2 L}{2}\left\| \Gamma_tg_t - U_t -V_t \right\|^2 \\
    &\le& \frac{3\eta^2 L}{2}\left\| \Gamma_tg_t \right\|^2
        + \frac{3\eta^2 L}{2}\left\| U_t \right\|^2
        + \frac{3\eta^2 L}{2}\left\| V_t \right\|^2 \\
    &\le& \frac{3\eta^2 L G^2}{2\lambda^2}
        + \frac{3\eta^2 L}{2\lambda^2M}\sum_{m=1}^M \left\| g_t - \widetilde{g}_t^m \right\|^2
        + \frac{3\eta^2 \beta_1 LG}{2(1-\beta_1)M}\sum_{m=1}^M\left\| \Gamma_{t-1}^m - \Gamma_t^m \right\|^2 \\
\end{eqnarray*}

\underline{Step 3e (one step progress).} Bounding term V.

\begin{eqnarray*}
    V
    &=& \eta \left\langle \nabla f(x_t) - \nabla f(z_t), \Gamma_tg_t \right\rangle \\
    &=& \eta\E\left[\left\langle \nabla f(x_t)-\nabla f(z_t), \Gamma_{t-1}g_t \right\rangle\right] + \eta\E\left[\left\langle \nabla f(x_t)-\nabla f(z_t), (\Gamma_t - \Gamma_{t-1})g_t \right\rangle\right] \nonumber\\
    &\leq& \eta\E\left[\left\langle \nabla f(x_t)-\nabla f(z_t), \frac{1}{M}\sum_{m=1}^M \nabla f_m(x_t^m) \right\rangle_{\Gamma_{t-1}}\right] + \frac{\eta^2 L\beta_1}{1-\beta_1}\E\left[\left\| \E_m[\Gamma_{t-1}^m u^m_{t-1}] \right\| \|(\Gamma_t-\Gamma_{t-1}) g_t\|\right] \nonumber\\
    &\leq& \eta\E\left[\left\langle \nabla f(x_t)-\nabla f(z_t), \nabla f(x_t) \right\rangle_{\Gamma_{t-1}}\right] \\
    && +\; \eta\E\left[\left\langle \nabla f(x_t)-\nabla f(z_t), \frac{1}{M}\sum_{m=1}^M \nabla f_m(x_t^m) - \nabla f_m(x_t) \right\rangle_{\Gamma_{t-1}}\right]
    + \frac{\eta^2 L\beta_1 G^2}{(1-\beta_1)\lambda}\E\left[ \|\Gamma_t-\Gamma_{t-1}\|\right] \nonumber\\
    %%%%%%%%%%%%%%%%%%%%%
    &\leq& \frac{\eta}{\lambda}\E\left[ \| \nabla f(x_t)-\nabla f(z_t)\| \|\nabla f(x_t)\|\right] \\
    && +\; \frac{\eta}{\lambda}\E\left[ \|\nabla f(x_t)-\nabla f(z_t)\| \cdot \frac{1}{M}\sum_{m=1}^M \|\nabla f_m(x_t^m) - \nabla f_m(x_t) \| \right]
    + \frac{\eta^2 L\beta_1 G^2}{(1-\beta_1)\lambda}\E\left[ \|\Gamma_t-\Gamma_{t-1}\|\right] \nonumber\\
    %%%%%%%%%%%%%%%%%%%%%
    &\leq& \frac{\eta}{\lambda}\E\left[ \frac{1}{2\rho}\|\nabla f(x_t)-\nabla f(z_t)\|^2 +  \frac{\rho}{2}\|\nabla f(x_t)\|^2 \right] \\
    && +\; \frac{\eta}{\lambda}\E\left[ \frac{1}{2}\|\nabla f(x_t)-\nabla f(z_t)\|^2 + \frac{1}{2}\frac{L^2}{M}\sum_{m=1}^M \|x_t^m-x_t \|^2 \right]
    + \frac{\eta^2 L\beta_1 G^2}{(1-\beta_1)\lambda}\E\left[ \|\Gamma_t-\Gamma_{t-1}\|\right] \nonumber,
    %%%%%%%%%%%%%%%%%%%%%
\end{eqnarray*}

where we used the following uniform bound on $\| \nabla f(x_t)-\nabla f(z_t) \|$:
\begin{eqnarray*}
    \left\| \nabla f(x_t)-\nabla f(z_t) \right\|
    &\le& L\left\|x_t-z_t \right\|
    \le \frac{\beta_1 L}{1-\beta_1} \left\|x_t-x_{t-1} \right\|
    = \frac{\eta\beta_1 L}{1-\beta_1} \left\|\E_m[\Gamma_{t-1}^m u_{t-1}^m] \right\| \\
    &\le& \frac{\eta\beta_1 L}{1-\beta_1} \E_m[\left\|\Gamma_{t-1}^m\| \| u_{t-1}^m] \right\|
    \le \frac{\eta\beta_1 L}{1-\beta_1} \frac{G}{\lambda}.
\end{eqnarray*}

Therefore, ignoring the constants, we have the following bounds:
\begin{eqnarray*}
V &\le&
  \mathcal{O}\left(\frac{\eta^2}{\rho}\right)
+ \frac{\eta\rho}{2\lambda} \cdot \E[\|\nabla f(x_t)\|^2]
+ \mathcal{O}\left(\eta\right) \cdot \frac{1}{M}\sum_{m=1}^M \E[\|x_t^m-x_t \|^2] + \mathcal{O}\left(\eta^2\right) \\
IV &\le&
  \mathcal{O}\left(\eta^2\right) \\
III &\le&
  \mathcal{O}\left(\eta\right) \cdot \frac{1}{M}\sum_{m=1}^M \E[\|\Gamma_{t-1}^m-\Gamma_t^m\|] \\
II &\le&
  \mathcal{O}\left(\eta\right) \cdot \frac{1}{M}\sum_{m=1}^M \E[\|g_t - \widetilde{g}_t^m\|] \\
I &\le&
  - \frac{\eta}{2C_0}\E\|\nabla f(x_t)\|^2
    + \mathcal{O}\left(\eta\right) \cdot \frac{1}{M} \sum_{m=1}^M \E\left[\|x_t - x_t^m\|^2\right]
    + \mathcal{O}\left(\eta\right) \cdot \frac{1}{M}\sum_{m=1}^M \E[\|\Gamma_{t-1}^m-\Gamma_t^m\|]
\end{eqnarray*}

% ======================================================
% ======================================================

To get the $\mathcal{O}\left(\frac{1}{\sqrt{T}}\right)$ bound for the averaged gradients $\E[\|\nabla f(x_t)\|^2]$, note that we are left to choose small value for $\rho = \frac{\lambda}{2C_0}$ and show the following bounds:
\begin{eqnarray*}
    && \frac{1}{TM}\sum_{t=0}^{T-1}\sum_{m=1}^M \E[\|x_t^m-x_t \|^2] = \mathcal{O}(\eta^2), \qquad \textrm{(extension of Lemma \ref{lem:x-xm})} \\
    && \sum_{t=0}^{T-1} \E[\|\Gamma_{t-1}^m-\Gamma_t^m\|] = \mathcal{O}(1), \qquad \textrm{(follows from AMSGrad normalization)} \\
    && \frac{1}{M}\sum_{t=0}^{T-1}\sum_{m=1}^M \E[\|g_t - \widetilde{g}_t^m\|] = \mathcal{O}(1), \qquad \textrm{(see below)}.
\end{eqnarray*}

For the last bound, we can use similar steps as in Lemma~\ref{lem:x-xm}, namely

\begin{eqnarray*}
\E[\|u_t - u_t^m\|]
&=& p_u \cdot 0 + (1-p_u)\E[\|\beta_1 u_{t-1} + (1-\beta_1)g_{t} - (\beta_1 u_{t-1}^m + (1-\beta_1)g_{t}^m)\|] \\
&\le& (1-p_u)\beta_1\E[\|u_{t-1} - u_{t-1}^m)\|] + (1-p_u)(1-\beta_1)\E[\|g_{t} - g_{t}^m)\|] \\
&\le& (1-p_u)(1-\beta_1)\sum_{\tau=0}^t ((1-p_u)\beta_1)^{t-\tau}\E[\|g_{\tau}-g_{\tau}^m\|]. \\
% \E[\|g_t - \widetilde{g}_t^m\|]
% &=& \E\left\|\frac{u_t - \beta_1 u_{t-1}}{1-\beta_1} - \frac{u_t^m - \beta_1 u_{t-1}^m}{1-\beta_1}\right\| \\
% &=& \frac{p_u \beta_1}{1-\beta_1}\E\|u_{t-1}-u_{t-1}^m\| + (1-p_u)\E[\|g_t - g_t^m\|] \\
% &\le& \frac{p_u \beta_1}{1-\beta_1}\left[ (1-p_u)\beta_1\E[\|u_{t-2} - u_{t-2}^m)\|] + (1-p_u)(1-\beta_1)\E[\|g_{t-1} - g_{t-1}^m)\|] \right] \\
% && +\; (1-p_u)\E[\|g_t - g_t^m\|] \\
% &\le& \max(1-\beta_1, (1-p_u)\beta_1)\left[ \frac{p_u \beta_1}{1-\beta_1}\left[ \E[\|u_{t-2} - u_{t-2}^m)\|] + (1-p_u)\E[\|g_{t-1} - g_{t-1}^m)\|] \right] \right] \\
% && +\; (1-p_u)\E[\|g_t - g_t^m\|] \\
% &=& \underbrace{\max(1-\beta_1, (1-p_u)\beta_1)}_{\eqdef q < 1} \E[\|g_{t-1} - \widetilde{g}_{t-1}^m\|] + (1-p_u)\E[\|g_t - g_t^m\|] \\
% &\le& (1-p_u)\sum_{\tau=1}^t q^{t-\tau} \E[\|g_{\tau} - g_{\tau}^m\|] \\
\E[\|g_t - \widetilde{g}_t^m\|]
&=& \E\left\|\frac{u_t - \beta_1 u_{t-1}}{1-\beta_1} - \frac{u_t^m - \beta_1 u_{t-1}^m}{1-\beta_1}\right\| \\
&\le& \frac{\beta_1}{1-\beta_1}\E\|u_{t-1}-u_{t-1}^m\| + \frac{1}{1-\beta_1}\E[\|u_t - u_t^m\|] \\
&=& \frac{1}{1-\beta_1} \sum_{\tau = t-1}^t \beta_1^{t-\tau}\E\|u_{\tau}-u_{\tau}^m\| \\
&=& (1-p_u)\sum_{\tau = t-1}^t \sum_{\nu=0}^{\tau} \beta_1^{t-\tau}((1-p_u)\beta_1)^{\tau-\nu}\E[\|g_{\nu}-g_{\nu}^m\|] \\
&=& \sum_{\tau = t}^{t+1} \sum_{\nu=0}^{\tau-1} \beta_1^{t-\tau}(\underbrace{(1-p_u)\beta_1}_{= q_2})^{\tau-\nu}\E[\|g_{\nu}-g_{\nu}^m\|],
% &=& \sum_{\nu = 0}^{t} \sum_{\tau=t-1}^{t} \beta_1^{t-\tau}(\underbrace{(1-p_u)\beta_1}_{= q_2})^{\tau-\nu}\E[\|g_{\nu}-g_{\nu}^m\|] \\
% \sum_{t=1}^T \E[\|g_t - \widetilde{g}_t^m\|]
% &\le& \frac{1}{1-\beta_1}\sum_{t=1}^T \sum_{\tau = t-1}^t \beta_1^{t-\tau}\E\|u_{\tau}-u_{\tau}^m\| \\
\end{eqnarray*}
which has the same double geometric sum structure as \eqref{double-geom-sum}.

% ======================================================
% ======================================================
\subsection{Key Lemmas}

\begin{lemma}\label{lem:z-x}
  For all $T\geq 1$, we have
  \begin{align}
    \sum_{t=0}^{T-1} \norm{z_t - x_t}^2
    \leq \frac{\eta^2\beta^2}{(1-\beta)^2 M} T\sigma^2 + \frac{\eta^2\beta^2}{(1-\beta)^2} \sum_{t=0}^{T-1} \E\left\| \frac{1}{M}\sum_{m=1}^M \nabla f_m(x_{t}^m) \right\|^2.
  \end{align}
\end{lemma}
\begin{proof}
  Since $u_{-1} = 0$, unrolling the update rule of momentum, for any $t\geq 0$ we get
  \begin{align*}
    u_t = \beta u_{t-1} + (1-\beta)g_t = (1-\beta)\sum_{\tau=0}^{t} \beta^{t-\tau} g^{\tau}.
  \end{align*}
Using this and the definition of the average iterates, we have
\begin{align*}
  z_t - x_t
  = \frac{\beta}{1-\beta} (x_t - x_{t-1}) 
  = -\frac{\beta \eta }{1-\beta} u_t
  = -\beta\eta \sum_{\tau=0}^{t} \beta^{t-\tau} g_{\tau}.
\end{align*}

Using convexity of squared norm function and letting $s_t \eqdef \sum_{\tau=0}^{t} \beta^{t-\tau} = \frac{1-\beta^{t+1}}{1-\beta}$, for all $t\geq 0$, we have 
\begin{eqnarray*}
  \norm{z_t - x_t}^2
  = \eta^2\beta^2 s_t^2 \left\| \sum_{\tau=0}^{t} \frac{\beta^{t-\tau}}{s_t} g_{\tau} \right\|^2
  \le \eta^2\beta^2 s_t^2 \sum_{\tau=0}^{t} \frac{\beta^{t-\tau}}{s_t} \|g_{\tau}\|^2
  \le \frac{\eta^2\beta^2}{1-\beta} \sum_{\tau=0}^{t} \beta^{t-\tau} \|g_{\tau}\|^2.
\end{eqnarray*}
Summing over the iterates yields
\begin{eqnarray*}
  \sum_{t=0}^{T-1} \E\norm{z_t - x_t}^2
  &\le& \frac{\eta^2\beta^2}{1-\beta} \sum_{t=0}^{T-1}\sum_{\tau=0}^{t} \beta^{t-\tau} \E\|g_{\tau}\|^2 \\
  &=& \frac{\eta^2\beta^2}{1-\beta} \sum_{\tau=0}^{T-1} \sum_{t=\tau}^{T-1} \beta^{t-\tau} \E\|g_{\tau}\|^2 \\
  &=& \frac{\eta^2\beta^2}{1-\beta} \sum_{\tau=0}^{T-1} \frac{1-\beta^{T-\tau}}{1-\beta} \E\|g_{\tau}\|^2 \\
  &\le& \frac{\eta^2\beta^2}{(1-\beta)^2} \sum_{\tau=0}^{T-1} \E\|g_{\tau}\|^2 \\
  &=& \frac{\eta^2\beta^2}{(1-\beta)^2} \sum_{\tau=0}^{T-1} \E\left\| \frac{1}{M}\sum_{m=1}^M g_{\tau}^m - \nabla f_m(x_{\tau}^m) \right\|^2
  + \frac{\eta^2\beta^2}{(1-\beta)^2} \sum_{\tau=0}^{T-1} \E\left\| \frac{1}{M}\sum_{m=1}^M \nabla f_m(x_{\tau}^m) \right\|^2 \\
  &=& \frac{\eta^2\beta^2}{(1-\beta)^2 M^2} \sum_{\tau=0}^{T-1} \sum_{m=1}^M \E\left\| g_{\tau}^m - \nabla f_m(x_{\tau}^m) \right\|^2
  + \frac{\eta^2\beta^2}{(1-\beta)^2} \sum_{\tau=0}^{T-1} \E\left\| \frac{1}{M}\sum_{m=1}^M \nabla f_m(x_{\tau}^m) \right\|^2 \\
  &=& \frac{\eta^2\beta^2}{(1-\beta)^2 M} T\sigma^2
  + \frac{\eta^2\beta^2}{(1-\beta)^2} \sum_{\tau=0}^{T-1} \E\left\| \frac{1}{M}\sum_{m=1}^M \nabla f_m(x_{\tau}^m) \right\|^2.
\end{eqnarray*}
\end{proof}

\begin{lemma} \label{lem:x-xm}
If $24\eta^2 L^2 \psi \le 1$, then
\begin{align*}
\frac{1}{MT} \sum_{t=0}^{T-1} \sum_{m=1}^{M} \E\norm{x_t - x_t^m}^2
\leq 12\eta^2 (B^2-1) \psi \cdot \frac{1}{T}\sum_{t=0}^{T-1}\E\|\nabla f(x_t)\|^2
+ 4\eta^2\psi(\sigma^2 + 3G^2),
\end{align*}
where
$$\psi = \frac{4(1-p_x)}{p_x^2} \cdot \frac{(1-\beta)(1-p_u)}{1-(1-p_u)\beta}$$
\end{lemma}
\begin{proof}
Let us expand the term $\E{\|x_{t+1} - x_{t+1}^m\|^2}$ using $x_{t+1}^m$'s probabilistic update rule:
\begin{eqnarray*}
\E{\|x_{t+1} - x_{t+1}^m\|^2}
&=& p_x\cdot 0 + (1-p_x)\cdot \E{\|x_{t} - \eta u_{t} - (x_{t}^m - \eta u_{t}^m) \|^2}\\
&=& (1-p_x)\cdot \E{\|x_{t} - x_{t}^m - \eta(u^t - u_t^m)\|^2}\\
&\le& (1-p_x)(1+s) \E{\|x_{t} - x_{t}^m\|^2} + \eta^2(1-p_x)(1+\nicefrac{1}{s})\E{\|u_t - u_t^m\|^2}\\
&\le& \eta^2 (1-p_x)(1+\nicefrac{1}{s}) \sum_{\tau=1}^t ((1-p_x)(1+s))^{t-\tau} \E{\|u_{\tau} - u_{\tau}^m\|^2}.
\end{eqnarray*}
where $s>0$ will be chosen later. Next we expand the term $\E{\|u_t - u_t^m\|^2}$ using $u_t^m$'s probabilistic update rule:
\begin{eqnarray*}
\E{\|u_t - u_t^m\|^2}
&=& p_u \cdot 0 + (1-p_u)\cdot \E{\left\|\frac{1}{M}\sum_{m=1}^M (\beta u_{t-1}^m + (1-\beta)g_{t-1}^m) - (\beta u_{t-1}^m + (1-\beta)g_{t-1}^m) \right\|^2} \\
&=& (1-p_u)\E{\left\|\beta (u_{t-1}-u_{t-1}^m) + (1-\beta)(g_{t-1} - g_{t-1}^m) \right\|^2} \\
&\le& (1-p_u)\beta\E{\|(u_{t-1}-u_{t-1}^m) \|^2}
+ (1-p_u)(1-\beta)\E{\| g_{t-1} - g_{t-1}^m \|^2} \\
&\le& (1-p_u)(1-\beta)\sum_{\tau=0}^{t-1} ((1-p_u)\beta)^{t-1-\tau} \E{\| g_{\tau} - g_{\tau}^m \|^2} \\
&\le& \frac{1-\beta}{\beta}\sum_{\tau=0}^{t-1} ((1-p_u)\beta)^{t-\tau} \E{\| g_{\tau} - g_{\tau}^m \|^2}
\end{eqnarray*}
Denote $q_1 = (1-p_x)(1+s)$ and $q_2 = (1-p_u)\beta$. Combining the previous two bounds, we get
\begin{eqnarray}
&& \frac{1}{M}\sum_{m=1}^M\E{\|x_t - x_t^m\|^2} \nonumber \\
&\le& \eta^2(1-p_x)(1+\nicefrac{1}{s}) \sum_{\tau=1}^t ((1-p)(1+s))^{t-\tau} \frac{1}{M}\sum_{m=1}^M\E{\|u_{\tau} - u_{\tau}^m\|^2} \label{double-geom-sum} \\
&\le& \eta^2 (1-p_x)(1+\nicefrac{1}{s}) \sum_{\tau=1}^t ((1-p_u)(1+s))^{t-\tau} \frac{1}{M}\sum_{m=1}^M \left[\frac{1-\beta}{\beta} \sum_{\nu=0}^{\tau-1} ((1-p_u)\beta)^{\tau-\nu} \E{\| g_{\nu} - g_{\nu}^m \|^2}\right] \nonumber \\
&=& \eta^2 (1-p_x)(1+\nicefrac{1}{s}) \frac{1-\beta}{\beta} \sum_{\tau=1}^t \sum_{\nu=0}^{\tau-1} q_1^{t-\tau} q_2^{\tau-\nu} \left[\frac{1}{M}\sum_{m=1}^M\E{\| g_{\nu} - g_{\nu}^m \|^2}\right] \nonumber \\
&=& \eta^2 (1-p_x)(1+\nicefrac{1}{s}) \frac{1-\beta}{\beta} \sum_{\nu=0}^{t-1} \sum_{\tau=\nu+1}^{t} q_1^{t-\tau} q_2^{\tau-\nu} \left[\frac{1}{M}\sum_{m=1}^M\E{\| g_{\nu} - g_{\nu}^m \|^2}\right] \nonumber \\
&=& \eta^2 (1-p_x)(1+\nicefrac{1}{s}) \frac{1-\beta}{\beta} \sum_{\nu=0}^{t-1}
q_2\frac{q_1^{t-\nu} - q_2^{t-\nu}}{q_1-q_2} \left[\frac{1}{M}\sum_{m=1}^M\E{\| g_{\nu} - g_{\nu}^m \|^2}\right], \nonumber \\
&=& \eta^2 \underbrace{(1-p_x)(1+\nicefrac{1}{s}) (1-\beta)(1-p_u)}_{\eqdef\phi} \sum_{\nu=0}^{t-1}
\frac{q_1^{t-\nu} - q_2^{t-\nu}}{q_1-q_2} \left[\frac{1}{M}\sum_{m=1}^M\E{\| g_{\nu} - g_{\nu}^m \|^2}\right]. \nonumber
\end{eqnarray}

Next, we bound the gradient term above.
\begin{eqnarray*}
\frac{1}{M}\sum_{m=1}^M \E{\| g_t^m - g_t \|^2}
&=& \frac{1}{M}\sum_{m=1}^M \E{\left\| g_t^m - \frac{1}{M}\sum_{i=1}^K g_i^{t} \right\|^2} \\
&\le& \frac{2}{K}\sum_{m=1}^M \E{\left\| g_t^m - \nabla f_m(x_t^m) - \frac{1}{M}\sum_{m=1}^M (g_{t}^m - \nabla f_m(x_t^m)) \right\|^2} \\
&&\quad +\; \frac{2}{M}\sum_{m=1}^M\E{\left\|\nabla f_m(x_t^m) - \frac{1}{M}\sum_{m=1}^M \nabla f_m(x_t^m) \right\|^2} \\
\textrm{(Lemma \ref{lem:het-var})}&\le& \frac{2}{M}\sum_{m=1}^M \E{\left\| g_t^m - \nabla f_m(x_t^m) \right\|^2}
- 2 \E{\left\| \frac{1}{M}\sum_{m=1}^M (g_{t}^m - \nabla f_m(x_t^m)) \right\|^2} \\
&&\quad +\; \frac{12L^2}{M}\sum_{m=1}^{M} \E\norm{x_t - x_t^m}^2 + 6(B^2-1) \E\|\nabla f(x_t)\|^2 + 6G^2 \\
&\le& 2\sigma^2 + \frac{12L^2}{M}\sum_{m=1}^{M} \E\norm{x_t - x_t^m}^2 + 6(B^2-1) \E\|\nabla f(x_t)\|^2 + 6 G^2.
\end{eqnarray*}

Again, plugging this bound to the previous one, we get
\begin{eqnarray*}
&& \frac{1}{MT}\sum_{t=0}^{T-1}\sum_{m=1}^M\E{\|x_t - x_t^m\|^2} \\
&\le& \frac{1}{MT}\sum_{t=1}^T\sum_{m=1}^M\E{\|x_t - x_t^m\|^2} \\
&\le& \frac{\eta^2\phi}{T} \sum_{t=1}^T\sum_{\tau=0}^{t-1} \frac{q_1^{t-\tau} - q_2^{t-\tau}}{q_1-q_2} \left[\frac{1}{M}\sum_{m=1}^M \E{\| g_{\tau} - g_{\tau}^m \|^2}\right] \\
&=& \frac{\eta^2\phi}{T} \sum_{\tau=0}^{T-1} \sum_{t=\tau+1}^T \frac{q_1^{t-\tau} - q_2^{t-\tau}}{q_1-q_2} \left[\frac{1}{M}\sum_{m=1}^M \E{\| g_{\tau} - g_{\tau}^m \|^2}\right] \\
&=& \frac{\eta^2\phi}{T} \sum_{\tau=0}^{T-1} \frac{1}{q_1-q_2}\left( \frac{q_1(1-q_1^{T-\tau})}{1-q_1} - \frac{q_2(1-q_2^{T-\tau})}{1-q_2} \right) \left[\frac{1}{M}\sum_{m=1}^M \E{\| g_{\tau} - g_{\tau}^m \|^2}\right] \\
&\le& \frac{\eta^2\phi}{T} \sum_{\tau=0}^{T-1} \frac{1}{q_1-q_2}\left( \frac{q_1}{1-q_1} - \frac{q_2}{1-q_2} \right) \left[\frac{1}{M}\sum_{m=1}^M \E{\| g_{\tau} - g_{\tau}^m \|^2}\right] \\
&=& \frac{\eta^2\phi}{(1-q_1)(1-q_2)} \frac{1}{T}\sum_{\tau=0}^{T-1} \left[\frac{1}{M}\sum_{m=1}^M \E{\| g_{\tau} - g_{\tau}^m \|^2}\right].
\end{eqnarray*}

Now, let us optimize the factor
$$
\frac{\phi}{(1-q_1)(1-q_2)}
= \frac{(1-p_x)(1+\nicefrac{1}{s}) (1-\beta)(1-p_u)}{(1-(1-p_x)(1+s))(1-(1-p_u)\beta)}
= \frac{(1-p_x)(1+\nicefrac{1}{s})}{1-(1-p_x)(1+s)} \cdot \frac{(1-\beta)(1-p_u)}{1-(1-p_u)\beta}
$$
by choosing optimal value for $s$ introduced earlier. By the first order optimality condition, we find that the optimal value is $s^* = \frac{1}{\sqrt{1-p_x}}-1$. Hence, the minimal value of the factor is
\begin{eqnarray*}
\frac{\phi}{(1-q_1)(1-q_2)}
&=& \frac{1-p_x}{(1-\sqrt{1-p_x})^2} \cdot \frac{(1-\beta)(1-p_u)}{1-(1-p_u)\beta} \\
&=& \frac{(1-p_x)(1-\sqrt{1-p_x})^2}{(1-\sqrt{1-p_x})^2(1+\sqrt{1-p_x})^2} \cdot \frac{(1-\beta)(1-p_u)}{1-(1-p_u)\beta} \\    
&=& \frac{(1-p_x)(1+\sqrt{1-p_x})^2}{p_x^2} \cdot \frac{(1-\beta)(1-p_u)}{1-(1-p_u)\beta} \\
&\le& \frac{4(1-p_x)}{p_x^2} \cdot \frac{(1-\beta)(1-p_u)}{1-(1-p_u)\beta} \eqdef \psi. \\
\end{eqnarray*}

Continuing the chain of bounds
\begin{eqnarray*}
&& \frac{1}{MT}\sum_{t=0}^{T-1}\sum_{m=1}^M\E{\|x_t - x_t^m\|^2} \\
&\le& \eta^2\psi \cdot \frac{1}{T}\sum_{t=0}^{T-1} \left[\frac{1}{K}\sum_{m=1}^M \E{\| g_t - g_t^m \|^2}\right] \\
&\le& \eta^2\psi \cdot \frac{1}{T}\sum_{t=0}^{T-1} \left[ \frac{12L^2}{M}\sum_{m=1}^{M} \E\norm{x_t - x_{t}^m}^2 + 6(B^2-1) \E\|\nabla f(x_t)\|^2 + 2\sigma^2 + 6 G^2 \right] \\
&\le& 12\eta^2 L^2 \psi \cdot \frac{1}{TM}\sum_{t=0}^{T-1} \sum_{m=1}^{M} \E\norm{x_t - x_{t}^m}^2 \\
&& +\; 6\eta^2 (B^2-1) \psi \cdot \frac{1}{T}\sum_{t=0}^{T-1}\E\|\nabla f(x_t)\|^2
+ 2\eta^2\psi(\sigma^2 + 3G^2).
\end{eqnarray*}
Assuming $12\eta^2 L^2 \psi \le \nicefrac{1}{2}$ and reordering the first term in the bound, we arrive
$$
\frac{1}{MT}\sum_{t=0}^{T-1}\sum_{m=1}^M\E{\|x_t - x_t^m\|^2}
\le 12\eta^2 (B^2-1) \psi \cdot \frac{1}{T}\sum_{t=0}^{T-1}\E\|\nabla f(x_t)\|^2
+ 4\eta^2\psi(\sigma^2 + 3G^2).
$$
\end{proof}

\begin{lemma}\label{lem:het-var} Under smoothness and bounded heterogeneity assumptions \ref{ass:smooth} and \ref{ass:het}, we have
\begin{equation*}
\frac{1}{M}\sum_{m=1}^M \left\|\nabla f_m(x_t^m) - \frac{1}{K}\sum_{i=1}^K \nabla f_i(x_t^i)\right\|^2
\leq \frac{6L^2}{M}\sum_{m=1}^{M}\norm{x_t - x_t^m}^2 + 3(B^2-1)\|\nabla f(x_t)\|^2 + 3 G^2.
\end{equation*} 
\end{lemma}
\begin{proof}
The bound follows from simple algebraic manipulations and Jensen's inequality.
\begin{eqnarray*}
  &&\frac{1}{K}\sum_{m=1}^M  \|\nabla f_m(x_t^m)- \frac{1}{K}\sum_{i=1}^{N} \nabla f_i(x_t^i) \|^2 \\
  &=& \frac{1}{K}\sum_{m=1}^M  \left\| \nabla f_m(x_t^m) - \nabla f_m(x_t) + \nabla f_m(x_t) - \nabla f(x_t) + \nabla f(x_t) - \frac{1}{K}\sum_{i=1}^N \nabla f_i (x_t^i)\right\|^2 \\
  &\le& \frac{3}{K}\sum_{m=1}^M \|\nabla f_m(x_t^m) - \nabla f_m(x_t)\|^2
        + \frac{3}{K}\sum_{m=1}^M \|\nabla f_m(x_t) - \nabla f(x_t)\|^2 \\
  &&\quad  +\; \frac{3}{K}\sum_{m=1}^M \left\|\nabla f(x_t) - \frac{1}{K}\sum_{i=1}^K \nabla f_i (x_t^i)\right\|^2 \\
  &\le& \frac{3L^2}{K}\sum_{m=1}^M \|x_t^m - x_t\|^2
  + \frac{3}{K}\sum_{m=1}^M \|\nabla f_m(x_t) - \nabla f(x_t)\|^2
  + \frac{3L^2}{K}\sum_{i=1}^K \left\| x_t - x_t^i\right\|^2 \\
  &=& \frac{6L^2}{K}\sum_{m=1}^M \|x_t^m - x_t\|^2
  + \frac{3}{K}\sum_{m=1}^M \|\nabla f_m(x_t) - \nabla f(x_t)\|^2 \\
  &=& \frac{6L^2}{K}\sum_{m=1}^M \|x_t^m - x_t\|^2
  + 3 G^2 + 3(B^2-1)\|\nabla f(x_t)\|^2.
  \end{eqnarray*}
\end{proof}

%% file: algorithms/desync_sgdm.tex
\begin{algorithm}[H]
\caption{\methodsgdm}
\label{alg:desync_sgdm}
\small
\begin{algorithmic}[1]
% ---------------- REQUIRE ------------------------------------------------
  \Require \textbf{Model tensors} \\
          \quad $x_0 \in \mathbb{R}^{d}$ — initial parameter vector \\
          \quad $u_{-1} \in \mathbb{R}^{d}$ — seed for the momentum, initialised to \textbf{0} 
  \Require \textbf{Hyper-parameters} \\
          \quad $\{\eta_t\}_{t=0}^{T-1} \subset \mathbb{R}_{>0}$ — step-size schedule \\
          \quad $\beta \in [0,1)$ — Momentum decay factor \\
          \quad $T \in \mathbb{N}_{+}$ — total optimisation iterations \\
          \quad $M \in \mathbb{N}_{+}$ — number of workers \\
          \quad \textcolor{blue}{$p_x = \frac{1}{K_x},p_u = \frac{1}{K_u}$} $\in [0,1]$ — synchronization probabilities for parameters and momentum
  \Ensure $x_T,\;u_{T-1},v_{T-1}$
% ------------------------------------------------------------------------
  \State \textbf{for each worker} $m$: $x_0^m = x_0,\;u_{-1}^m = v_{-1}^m = 0$
         \hfill\textcolor{gray}{\scriptsize local init ($t=-1$ seeds)}
  \For{$t = 0,\dots,T-1$} \hfill\textcolor{gray}{\scriptsize training loop}
      \ForAll{workers $m=0,\dots,M-1$ \textbf{in parallel}}
% -------- gradient -------------------------------------------------------
    \State $g_t^m \gets \nabla F_m(x_t^m;\xi_t^m)$
           \hfill\textcolor{gray}{\scriptsize stochastic gradient}
% -------- first moment ---------------------------------------------------
\State $u_t^{m} \gets
    \begin{cases}
    \E_m[\beta u_{t-1}^m + (1-\beta)g_{t}^m], & \textcolor{blue}{\text{with probability } p_u} \\ 
    \beta u_{t-1}^m + (1-\beta)g_{t}^m, & \text{with probability } 1-p_u 
    \end{cases}$ \hfill\textcolor{gray}{\scriptsize sync $u$}
\State $x_{t+1}^{m} \gets
    \begin{cases}
    \E_m[x_{t}^m - \eta_t u_{t}^m], & \textcolor{blue}{\text{with probability } p_x} \\ 
    x_{t}^m - \eta_t u_{t}^m, & \text{with probability } 1-p_x 
    \end{cases}$ \hfill\textcolor{gray}{\scriptsize sync $x$}
    \EndFor
  \EndFor
\end{algorithmic}
\end{algorithm}

%% file: algorithms/desync_adam_prob.tex
\begin{algorithm}[h]
\caption{\methodadam (with probabilistic synchronization)}
\label{alg:desync_adam_prob}
\small
\begin{algorithmic}[1]
% ---------------- REQUIRE ------------------------------------------------
  \Require \textbf{Model tensors} \\
          \quad $x_0 \in \mathbb{R}^{d}$ — initial parameter vector \\
          \quad $u_{-1},v_{-1} \in \mathbb{R}^{d}$ — seeds for first and second moments, initialised to \textbf{0} 
  \Require \textbf{Hyper-parameters} \\
          \quad $\{\eta_t\}_{t=0}^{T-1} \subset \mathbb{R}_{>0}$ — step-size schedule \\
          \quad $\beta_{1},\beta_{2} \in [0,1)$ — Adam decay factors \\
          \quad $\lambda \in \mathbb{R}_{\ge 0}$ — $\ell_{2}$ stability term \\
         \quad $T \in \mathbb{N}_{+}$ — total optimisation iterations \\
        \quad $M \in \mathbb{N}_{+}$ — number of workers \\
          \quad \textcolor{blue}{$p_x = \frac{1}{K_x},p_u = \frac{1}{K_u},p_v = \frac{1}{K_v}$} $\in [0,1]$ — synchronization probabilities for parameters and momentums
  \Ensure $x_T,\;u_{T-1},v_{T-1}$
% ------------------------------------------------------------------------
  \State \textbf{for each worker} $m$: $x_0^m = x_0,\;u_{-1}^m = v_{-1}^m = 0$
         \hfill\textcolor{gray}{\scriptsize local init ($t=-1$ seeds)}
  \For{$t = 0,\dots,T-1$} \hfill\textcolor{gray}{\scriptsize training loop}
      \ForAll{workers $m=0,\dots,M-1$ \textbf{in parallel}}
% -------- gradient -------------------------------------------------------
    \State $g_t^m \gets \nabla F(x_t^m;\xi_t^m)$
           \hfill\textcolor{gray}{\scriptsize stochastic gradient}
% -------- first moment ---------------------------------------------------
    \State $u_t^{m} \gets
        \begin{cases}
            \E_m[\beta_1 u_{t-1}^m + (1-\beta_1) g_{t}^m], & \textcolor{blue}{\text{with probability } p_u} \\ 
            \beta_1 u_{t-1}^m + (1-\beta_1)g_{t}^m, & \text{with probability } 1-p_u 
        \end{cases}$ \hfill\textcolor{gray}{\scriptsize sync $u$}
% -------- second moment --------------------------------------------------
    \State $v_t^{m} \gets
        \begin{cases}
            \E_m[\beta_2 v_{t-1}^m + (1-\beta_2)(g_t^m\odot g_t^m)], & \textcolor{blue}{\text{with probability } p_v} \\ 
            \beta_2 v_{t-1}^m + (1-\beta_2)(g_t^m\odot g_t^m), & \text{with probability } 1-p_v 
        \end{cases}$ \hfill\textcolor{gray}{\scriptsize sync $u$}
% -------- AMSGrad Normalization --------------------------------------------
    \State $\tilde v_t^m \gets \max(v_t^m, \tilde v_{t-1}^m)$
           \hfill\textcolor{gray}{\scriptsize AMSGrad Normalization, $\tilde v_{-1}=v_{-1}$}
% -------- Adam direction -------------------------------------------------
    \State $d_t^m \gets
           \dfrac{\eta_t}{\sqrt{\tilde v_t^m+\lambda^{2}}}\odot u_t^m$
           \hfill\textcolor{gray}{\scriptsize bias-corrected update}
% -------- parameter update -----------------------------------------------
    \State $x_{t+1}^{m} \gets
        \begin{cases}
            \E_m[x_{t}^m - d_{t}^m], & \textcolor{blue}{\text{with probability } p_x} \\ 
            x_{t}^m - d_{t}^m, & \text{with probability } 1-p_x 
        \end{cases}$ \hfill\textcolor{gray}{\scriptsize sync $x$}
    \EndFor
  \EndFor
\end{algorithmic}
\end{algorithm}

%% file: files/appendices/proof.tex
\section{Convergence Analysis of \methodadam (high-probability bounds)}\label{app:proof}

 For this section, we refer to Algorithm~\ref{alg:desync_generic} as \method-\texttt{OPT}$\left(K_x, K_1, \ldots, K_N\right)$. Let us consider the second algorithm \method-\texttt{OPT}$\left(K, K, \ldots, K\right)$ with $K = \mathrm{lcm}\{K_x, K_1, \ldots, K_N\}$. These two algorithms have a property that they both fully synchronize, i.e., all states and current iterates are the same, if $T = rK$ for some $r \in \mathbb{N}.$ 

Commonly, the analysis of \method-\texttt{OPT}$\left(K, K, \ldots, K\right)$ proceeds in the following way. In each step, construct an ideal update as if you were running \method-\texttt{OPT}$\left(1, 1, \ldots, 1\right)$ using virtual iterates (see the proof in the prior section for the example of analysis with virtual iterates), and bound the drift from this idealized scenario. For the case of \method-\texttt{OPT}$\left(K, K, \ldots, K\right)$, the bound typically depends on the distance of the current iterate from the last full synchronization. Below, we show that the drift of \texttt{OPT}$\left(K_x, K_1, \ldots, K_N\right)$ is not larger than \method-\texttt{OPT}$\left(K, K, \ldots, K\right)$, since \texttt{OPT}$\left(K_x, K_1, \ldots, K_N\right)$ synchronize more often. Therefore, the convergence rate of \texttt{OPT}$\left(K_x, K_1, \ldots, K_N\right)$ is not worse than the convergence rate for \method-\texttt{OPT}$\left(K, K, \ldots, K\right)$ as its analysis also applies to \texttt{OPT}$\left(K_x, K_1, \ldots, K_N\right)$, i.e., all final upper bounds derived for \method-\texttt{OPT}$\left(K, K, \ldots, K\right)$ are also valid for \texttt{OPT}$\left(K_x, K_1, \ldots, K_N\right).$ For instance, a typical way to estimate drift is to have an assumption of type $\|s^n_i - s^n_{i-1}\| \leq U$ for all $i \in \{1, 2, \ldots, k\},$ and $n \in \{1, 2, \ldots, M\},$ where $s^n_i$ is some state on client $n$ at step $i$ and $s_0 = s^1_0 = \ldots = s^M_0$ the synchronized state. Then, drift is usually expressed as $\|s^n_k - s_0\|$. For \method-\texttt{OPT}$\left(K, K, \ldots, K\right)$, we can simply bound 
\begin{align*}
    \|s^n_k - s_0\| = \left\|\sum_{i=1}^k s^n_i - s^n_{i-1}\right\| \leq \sum_{i=1}^k \|s^n_i - s^n_{i-1}\| \leq k U. 
\end{align*}
For \method-\texttt{OPT}$\left(K_x, K_1, \ldots, K_N\right)$, we can obtain the same bound, where we for simplicity assume that $s$ is synchronized every $K_s$ steps and $k \in \{K_s + 1, \ldots, 2 K_s\}.$
\begin{align*}
     \|s^n_k - s_0\| &= \left\|\sum_{i=K_s + 1}^k (s^n_i - s^n_{i-1})
     + s_{K_s} - s_0\right\| \\
     &\leq \sum_{i=K_s + 1}^k \|s^n_i - s^n_{i-1}\| + \left\|\frac1M\sum_{m=1}^M\sum_{i=1}^{K_s} s^m_i - s^m_{i-1}\right\| \\ 
     &\leq \sum_{i=K_s + 1}^k \|s^n_i - s^n_{i-1}\| + \frac1M\sum_{m=1}^M\sum_{i=1}^{K_s}\left\|s^m_i - s^m_{i-1}\right\| \\
     &\leq k U.
\end{align*}
In a more general case, we would apply the above recursively. Such type of adjustments is the only requirement to adapt analysis of \method-\texttt{OPT}$\left(K, K, \ldots, K\right)$ to obtain the same rate for \method-\texttt{OPT}$\left(K_x, K_1, \ldots, K_N\right)$ for the type of the analysis described above. 

We do not claim any novelty for this analysis. We mainly include these results for completeness, to showcase that our method converges under different settings. The main theoretical results showing that some of the optimizer states can be synchronized less frequently are presented in the prior section above. We would also like to highlight that this result might be relatively weak and not tight since we only show that \method-\texttt{OPT}$\left(K, K, \ldots, K\right)$ and \method-\texttt{OPT}$\left(K_x, K_1, \ldots, K_N\right)$ have the same worst-case convergence, but \method-\texttt{OPT}$\left(K, K, \ldots, K\right)$ requires less communication than \method-\texttt{OPT}$\left(K_x, K_1, \ldots, K_N\right)$ under this analysis, which is not the case in practice nor in the analyses presented above. 

 Finally, detailed inspection of the analysis of \methodadam$\left(K, K, \ldots, K\right)$\cite{LocalAdam} reveals that this analysis satisfies the above criteria. Thus, we can directly apply their results under the following assumptions and preliminaries.
 
 We aim to optimize a neural network $x$ under the loss function $f$
\begin{equation}
    \min_{x\in \reals^d} f(x):= \mathbb{E}_{\xi\sim\mathcal{D}} [F(x;\xi)].
\end{equation}
using $M$ workers, each of which has access to the stochastic gradient of $f$, $\nabla F(x;\xi)$ with $\xi$ independently drawn from the data distribution $D$. We define the auxiliary sequence,
\begin{equation}
z_{t+1}^m = 
\begin{cases}
 \frac{1}{1 - \beta_1} x_{t+1}^m - \frac{\beta_1}{1 - \beta_1} x_t^m & \text{if } t \bmod K \neq -1, \\
 \frac{1}{1 - \beta_1} x_{t+1}^m - \frac{\beta_1}{1 - \beta_1} \overline{x}_t & \text{otherwise}.
\end{cases}
\end{equation}
where, $\overline{x}_{t+1} = \mathbb{E}_m[x^m_{t+1}]$. We also define $\overline{z}_{t+1} = \mathbb{E}_m[z^m_{t+1}]$.

We make the following standard assumptions.

\begin{asp}[Lower-boundedness]\label{asp:lb}
   $f$ is closed, twice continuously differentiable and $\inf_{x\in\reals^d} f(x)=:f(x_*)=:f_*>-\infty$.
\end{asp}

\begin{asp}[Smoothness]\label{asp:smooth}
    There exists some set $\Omega\subset \reals^d$ and $L>0$, such that for any $x,y\in \Omega$, 
    \begin{equation}\label{eq:smooth}
        \|\nabla f(x)-\nabla f(y)\|\leq L\|x-y\|,
    \end{equation}
    \begin{equation}
        \|\nabla f(x)\|^2 \leq 2L(f(x)-f_*).
    \end{equation}
\end{asp}

\begin{asp}[Bounded $\alpha$-moment noise]\label{asp:moment_noise}
    There exists some set $\Omega\subset \reals^d$, $\alpha\geq 4$ and constant vector $\boldsymbol{\sigma}\succeq 0$ such that for any $x\in \Omega$,
    \begin{equation}
        \mathbb{E}_{\xi\sim \mathcal{D}} |\nabla F(x;\xi) - \nabla f(x)|^{\alpha} \preceq \boldsymbol{\sigma}^{\alpha}.
    \end{equation}
    Let $\sigma_{\infty}:=\|\boldsymbol{\sigma}\|_{\infty}=\max_i\{\sigma_i\}$, $\sigma:=\|\boldsymbol{\sigma}\|=\big(\sigma_1^2+\cdots+\sigma_d^2\big)^{1/2}$.
\end{asp}

\begin{asp}[Weak convexity]\label{asp:wc}
    There exists constant $\tau>0$ such that $f$ is $\tau$-weakly convex, i.e., for any $x,y\in \reals^d$,
    \begin{equation}\label{eq:weak-mono}
        \langle \nabla f(x) - \nabla f(y), x-y \rangle \geq -\tau \|x-y\|^2, 
    \end{equation}
     \begin{equation}
        f(y) \geq f(x) + \langle \nabla f(x), y-x\rangle - \frac{\tau}{2}\|x-y\|^2,\
        \nabla^2 f(x) \succeq -\tau I_d.
    \end{equation}
\end{asp}

Based on these assumptions, the \methodadam variant of Adam converges as stated in the following theorem.

\begin{thm} [Full version of Theorem 2]
Let the Assumptions \ref{asp:lb},\ref{asp:smooth} ,\ref{asp:moment_noise}, \ref{asp:wc}, hold for $\Omega = \mathrm{conv}(\mathbf{B}_{R_0}(\Omega_0))$, where $\Omega_0 \coloneq \{x : f(x) - f_* \leq 4\Delta \}$, $\mathbf{B}_{R_0}(\Omega_0) = \{x\in R^d :\exists y: \|x-y\|_2\leq R_0\}$, $R_0 = \sqrt{\frac{\Delta}{80L}}$, $K_{\mathrm{lcm}} = \mathrm{lcm}\{K_x,K_u,K_\upsilon\}$, and the same assumptions as in Theorem D.3 of \citep{LocalAdam}, then with probability $\geq 1-\delta$, \methodadam yields,

\begin{align*}
\frac{\lambda}{K_{\mathrm{lcm}} R} \sum^{R-1}_{r=0} \sum^{K_{\mathrm{lcm}}-1}_{k=0} \|\nabla f(\bar{z}_{r,k})\|^2  =  \tilde{\mathcal{O}}\left( \frac{\tau \Delta}{R} + \frac{L \Delta}{K_{\mathrm{lcm}}R} + \sqrt{\frac{L\Delta\sigma^2}{MK_{\mathrm{lcm}}R}} + \frac{(L\Delta \sigma)^\frac{2}{3}}{K_{\mathrm{lcm}}^{\frac{1}{3}}R^{\frac{2}{3}}}  + \left(\frac{L \Delta \sigma^{\frac{a}{a-1}}}{K_{\mathrm{lcm}}R}\right)^{\frac{2(a-1)}{3a-2}}\right)
\end{align*}
\end{thm}
\begin{proof} The above corresponds to Theorem D.3 of \citep{LocalAdam} for \methodadam$(K_{\mathrm{lcm}}, \ldots, K_{\mathrm{lcm}})$.
\end{proof}

Note that for sufficiently large $R$, the leading term in the rate is $ \sqrt{\frac{L\Delta\sigma^2}{MK_{\mathrm{lcm}}R}}$, which shows up in Theorem~\ref{thm:adam_desync}.

%% file: files/appendices/derivation_max_change.tex
\section{Derivation of \cref{eq:abslute change_u,eq:abslute change_v}: Maximum Momentum Change With Clipping}\label{app:derivation_max_chance}

\paragraph{\textbf{Lemma.}}
Let the gradient at each step satisfy $ \|g_t\|_\infty \leq \rho $ for some constant $\rho > 0$.  
Assume the first-momentum state in Adam is initialized at $u_{-1}=0$ and updated by
\begin{align}
u_{t} = \beta_1 u_{t-1} + (1-\beta_1) g_t, \quad \beta_1\in[0,1).
\end{align}
Then, for all $t\geq 0$, the momentum is bounded and satisfies
\begin{align}
\|u_t\|_\infty \leq \rho, \quad\text{and}\quad
\|u_{t+K}-u_t\|_\infty \leq 2\rho\left(1-\beta_1^{K}\right)\quad\forall K\geq 1.
\end{align}

\paragraph{\textbf{Proof.}}

\subparagraph{\textbf{Step 1: Bound on $\|u_t\|_\infty$.}}

We first show by induction that the momentum is always bounded by $\rho$.

\textbf{Base Case ($t=0$):} 
Since $u_{-1}=0$, we have:
\begin{align}
\|u_0\|_\infty = \|\beta_1 u_{-1}+(1-\beta_1)g_{0}\|_\infty 
\leq (1-\beta_1)\|g_{0}\|_\infty 
\leq \rho.
\end{align}

\textbf{Inductive Hypothesis (I.H.):} Assume $\|u_t\|_\infty \leq \rho$ for some $t\geq 0$.

\textbf{Inductive Step ($t \rightarrow t+1$):} Then,
\begin{align}
\|u_{t+1}\|_\infty &= \|\beta_1 u_t + (1-\beta_1)g_{t+1}\|_\infty \\
&\leq \beta_1\|u_t\|_\infty + (1-\beta_1)\|g_{t+1}\|_\infty \\
&\leq \beta_1\rho + (1-\beta_1)\rho = \rho.
\end{align}
Thus, by induction, we have the desired result:
\begin{align}
\|u_t\|_\infty \leq \rho,\quad\forall t \geq 0.
\end{align}

\subparagraph{\textbf{Step 2: Bound on $\|u_{t+K}-u_t\|_\infty$.}}

Now we bound the change in the momentum over $K$ steps explicitly. Unrolling the recursion, we have:
\begin{align}
u_{t+K} = \beta_1^{K}u_t + (1-\beta_1)\sum_{k=0}^{K-1}\beta_1^{k}g_{t+K-k}.
\end{align}

Subtracting $u_t$ from both sides, we obtain:
\begin{align}
u_{t+K}-u_t 
&= (\beta_1^{K}-1)u_t + (1-\beta_1)\sum_{k=0}^{K-1}\beta_1^{k}g_{t+K-k}.
\end{align}

Applying the triangle inequality gives:
\begin{align}
\|u_{t+K}-u_t\|_\infty
&\leq |1-\beta_1^{K}|\|u_t\|_\infty 
+ (1-\beta_1)\sum_{k=0}^{K-1}\beta_1^{k}\|g_{t+K-k}\|_\infty.
\end{align}

Using the bounds $\|u_t\|_\infty \leq \rho$ and $\|g_t\|_\infty \leq \rho$, we simplify to:
\begin{align}
\|u_{t+K}-u_t\|_\infty 
&\leq (1-\beta_1^{K})\rho + (1-\beta_1)\rho\sum_{k=0}^{K-1}\beta_1^{k}.
\end{align}

The geometric series simplifies as:
\begin{align}
\sum_{k=0}^{K-1}\beta_1^{k} = \frac{1-\beta_1^{K}}{1-\beta_1}.
\end{align}

Substituting this back into the expression yields:
\begin{align}
\|u_{t+K}-u_t\|_\infty 
&\leq (1-\beta_1^{K})\rho + (1-\beta_1^{K})\rho = 2\rho(1-\beta_1^{K}).
\end{align}

Thus, the momentum difference satisfies:
\begin{align}
\|u_{t+K}-u_t\|_\infty \leq 2\rho(1-\beta_1^{K}),\quad\forall K \geq 1.
\end{align}

\subparagraph{\textbf{Second-moment bound.}} 
Applying the exact same logic to the second momentum $v_t$, with $\beta_1$ replaced by $\beta_2$ and the bounded gradient squared term $\|g_t \odot g_t\|_\infty \leq \rho^2$, immediately gives:
\begin{align}
\|v_{t+K}-v_t\|_\infty \leq 2\rho^{2}(1-\beta_2^{K}).
\end{align}

This completes the proof. \hfill$\square$

%% file: files/appendices/wall_time_modeling.tex
\section{Wall-Clock Time Modeling}

Understanding the practical benefits of our proposal beyond the theoretical aspects and empirical convergence curves is crucial.
This section addresses the practical implications of adopting our method for training state-of-the-art (SOTA) large language models (LLMs) in large-scale distributed training infrastructures.
The most critical metrics are based on total wall-clock time, communication time, and resource utilization, i.e., how much of the wall-clock time is spent using the compute available instead of waiting for the communication to complete.
We provide the following simplified model for estimating total wall-clock time (\cref{app:wall_time_total}), computation time (\cref{app:wall_time_compute}), and communication time (\cref{app:wall_time_comms}) that applies to any method based on distributed data parallelism (\ddp).
The notation used here is consistent with that in \Cref{alg:desync_generic}.
We conclude this section with the results obtained with this modeling and their discussion.

\subsection{Estimating Total Wall-Clock Time}\label{app:wall_time_total}

The total wall-clock time for completing an LLM pre-training is based on the number of tokens processed $D$ (dataset size), the model size $d$ (the number of trainable parameters), the number of compute units $M$ (data-parallel/local workers), the floating point operations per second $S$ that these compute units can perform, the Model FLOPS Utilization (MFU), the average peer-to-peer (P2P) bandwidth $B$ and the latency $l$ between compute units.
We separate the total wall-clock time discussion into computational time (\cref{app:wall_time_compute}) and communication time (\cref{app:wall_time_comms}).
In our modeling, the total wall-clock time is the sum of computational time and communication time:

\begin{equation}
t_{\text{total}} =t_{\text{compute}} + t_{\text{comms}}
\end{equation}

We next derive $t_{\text{compute}}$ and $t_{\text{comms}}$ separately, and then instantiate $t_{\text{total}}$ for specific training methods.

\subsubsection{Estimating Computation Time}\label{app:wall_time_compute}

The total time spent computing $T_{\text{compute}}$ depends on the number of compute units $M$, their floating point operations per second $S$, the MFU of the training pipeline, and the total number of FLOPs $C$ that the training pipeline requires.
Following the same approach as in \citet{OgScalingLaws, TrainingComputeOptimalLLMs}, the total number of FLOPs required to train an LLM can be estimated as $C=6dD$, where $d$ is the number of model parameters and $D$ the total number of tokens (dataset size).
Since the MFU can be considered a measure of efficiency, i.e., $\text{MFU}\in[0,1]$, we can estimate the total time spent computing as:

\begin{equation}
t_{\text{compute}} =  \frac{C}{\text{MFU} \cdot S \cdot M} = \frac{6 \cdot d \cdot D}{\text{MFU} \cdot S \cdot M}
\end{equation}

In other words, if the hardware can perform $S\cdot M$ FLOPs/sec at peak and is utilized at MFU fraction of peak, the training FLOPs $C$ translate to that many seconds of compute.

In practice, MFU strongly depends on how the pipeline's parallelization is locally configured across the workers $M$.
For the sake of fairness in our comparisons, we can assume that the per-batch MFU of a data-parallel worker is the same as the per-batch MFU of a worker in our proposal and other local adaptive methods.
Importantly, this holds in cases where either such workers refer to a single GPU or each worker locally performs more advanced parallelism techniques, such as the ones proposed by \citet{FSDP_ZeRO,FSDP_Pytorch}.

\textbf{Resources Utilization and MFU.} Theoretically estimating the resource utilization in large-scale training of LLMs is very challenging despite prior knowledge of the number of hardware accelerators (GPUs), their theoretical peak FLOPs, and the total amount of FLOPs $C$ required to perform the task is available.
Following previous well-established proposals \citep{PALM}, we leverage MFU and the theoretical peak FLOPs of the hardware accelerators we used in our experiments.
Recent systems research \citep{ModelParallelism} has shown it is possible to reach $~50$\% of peak FLOPs even for trillion-parameter models by carefully combining data, tensor, and pipeline parallelism. This emphasizes that our model's assumptions (e.g., each worker sees full $d$) can be adapted to those scenarios by treating a model-parallel group as one worker with higher $S$ and similar MFU.
For the sake of a fair comparison, our analysis in this section compares different methods assuming that the local workers operate with the same theoretical peak FLOPs and the same MFU.
The results reported in \cref{app:wall_time_results} describe how such values were obtained.

\subsubsection{Estimating Communication Time}\label{app:wall_time_comms}

Communication time is the most critical factor when comparing standard data-parallel approaches to our proposal, since the computation time will be the same, given that they train the same model size on the same number of tokens using the same computing infrastructure.
At each communication step, the workers $W$ synchronize a set of parameters $M$, the amount of which depends on the method used.
For example, distributed data-parallel synchronization occurs at every batch step on the complete set of gradients produced by the $M$ workers, each exchanging a payload at batch step $i$ of $P_{\text{\ddp}, i}=d$ parameters.
In our proposal, the synchronization involves model parameters and optimizer states at different frequencies, making such estimation slightly more complex.
Since their time costs simply add up, we treat the parameter sync and momentum sync contributions independently.
For instance, if parameters are synced every $K_x$ steps and momenta every $K_u, K_v$ steps, we sum the time for each series of syncs.

Any of such payloads can be exchanged and averaged using bandwidth-efficient AllReduce methods, such as RingAllReduce \citep{Horovod}, which scales only with the speed of the slowest P2P link.
Given the slowest P2P bandwidth $B$ and a latency $l$, a single communication at timestamp $i$ is performed synchronously and in parallel across the $M$ workers, taking a total time of:

\begin{equation}
t_{\text{comms}, i} = \frac{2P_i}{B} \bigg(1 - \frac{1}{M}\bigg) + l,
\end{equation}

where $P_i$ is the payload size of the communication happening at the timestamp $i$, which depends on the optimization method adopted as described above.

\textbf{\ddp.} In the \ddp training approach, each of the $T$ optimization steps to train on $D$ tokens requires communicating at every step for a total training time of:

\begin{equation}
t_{\text{total}, \text{\ddp}} = t_{\text{compute}} + T \cdot \bigg[ \frac{2d}{B} \bigg(1 - \frac{1}{M}\bigg) + l \bigg]
\end{equation}

\textbf{\fedavg.} The approach of the \fedavg method is that of synchronizing with frequency $K$ only the model parameters across the $M$ workers. This, the total training time can be estimated as:

\begin{equation}
t_{\text{total}, \text{\fedavg}} = t_{\text{compute}} + \frac{T}{K} \cdot \bigg[ \frac{2d}{B} \bigg(1 - \frac{1}{M}\bigg) + l \bigg]
\end{equation}

This optimization procedure will communicate less than \ddp when $K<T$.

\textbf{\localadam.} Using a local adaptive optimizer such as \citet{LocalAdam} with a synchronization frequency of $K$ local steps, requires training for a total training time of:

\begin{equation}
t_{\text{total}, \text{\localadam}} = t_{\text{compute}} + \frac{3T}{K} \cdot \bigg[ \frac{2d}{B} \bigg(1 - \frac{1}{M}\bigg) + l \bigg]
\end{equation}

This means that, as long as $3K<T$, \localadam will always take less wall clock time than \ddp. 

\textbf{Our Method (\method).} Adopting our proposal (\methodadam and \methodadopt specifically, which we shall use interchangeably for the purposes of this analysis) requires synchronizing model parameters $x$, fist momentum $u$ and second momentum $v$ with frequencies $k_x, K_u, K_v$, respectively. Assuming each of these sets is synchronized independently, we can compose by adding their communication time contribution to the total training wall-clock time, which results:

\begin{equation}
t_{\text{total}, \text{\methodadam}} = t_{\text{compute}} + \bigg(\frac{T}{K_x} + \frac{T}{K_u} + \frac{T}{K_v}\bigg) \cdot \bigg[ \frac{2d}{B} \bigg(1 - \frac{1}{M}\bigg) + l \bigg]
\end{equation}

This means that, as long as $\frac{1}{K_x} + \frac{1}{K_u} + \frac{1}{K_v} < \frac{3}{K} \wedge \frac{1}{K_x} + \frac{1}{K_u} + \frac{1}{K_v} < 1$, our method will always take less wall-clock time than \localadam and \ddp.

\textbf{Limitations.} We critically discuss here the limitations of the proposed modeling in order to shed light on their relevance when it comes to deploying such training algorithms in real-world scenarios.

First, our modeling approach adopts constants for several system components, such as computing capabilities and interconnects. In particular, MFU in the real world always oscillates around some average value depending on the operational performance of high-bandwidth memories (HBMs), DRAM caches, and processing units in the hardware accelerators. At the same time, the P2P bandwidth and latency between accelerators also fluctuate around average values.

Second, most efficient implementations adopted in the field take advantage of the possibility of overlapping communication and computation, reducing the communication time.
Notably, overlapping communication with computation can drastically reduce effective communication costs, for example, PyTorch's DDP implementation can overlap $95\%$ of the communication \citep{romero22usenix}.
Our model currently assumes synchronous communications, but could incorporate such approaches by reducing the effective $l$ or $B$ impact.
One extension could be adding a parameter $\alpha \in [0,1]$ representing the fraction of communication time that is not overlapped, so total time per step $i$ is $t_{\text{total}, i}= t_{\text{compute}} + \alpha t_{\text{comm}}$. Setting $\alpha=0$ would recover the fully overlapped ideal (communication is entirely hidden by computation), and $\alpha=1$ is the current no-overlap assumption. This would keep the model framework-agnostic but allow tuning to specific training setups.

Techniques in \citet{FSDP_ZeRO,FSDP_Pytorch} complement our analysis by reducing memory usage and communication volume, effectively scaling down payload $P_i$ or increasing MFU. Our approach focuses on synchronization timing rather than data partitioning; combining our method with fragmented updates (e.g., ZeRO) could further improve wall-clock time.

Despite limitations, our model was designed so that any gap with real-world performance evenly affects all methods analyzed, assuming thoughtful implementation. Thus, results in \cref{app:wall_time_results} illustrate potential improvements from adopting \method, and our model can help practitioners estimate performance at larger scales.

\subsection{Modeling Results}\label{app:wall_time_results}

\Cref{app:fig:wall_clock_time_model,app:fig:comms_and_compute_utilization} analyze the wall-clock time, communication overhead, and GPU utilization of \method compared to \ddp, \localadam, and heuristic baselines for training our $1.7$B model. By setting synchronization periods as $K_x=256, K_u=768, K_v=1536$, \method significantly reduces communication and improves GPU utilization relative to \localadam ($K=256$), closely approaching the efficiency of heuristic methods, especially in bandwidth-constrained settings.

\begin{figure}[H]
    \centering
      % \noindent\subfloat[]
    \includegraphics[width=0.8\columnwidth]{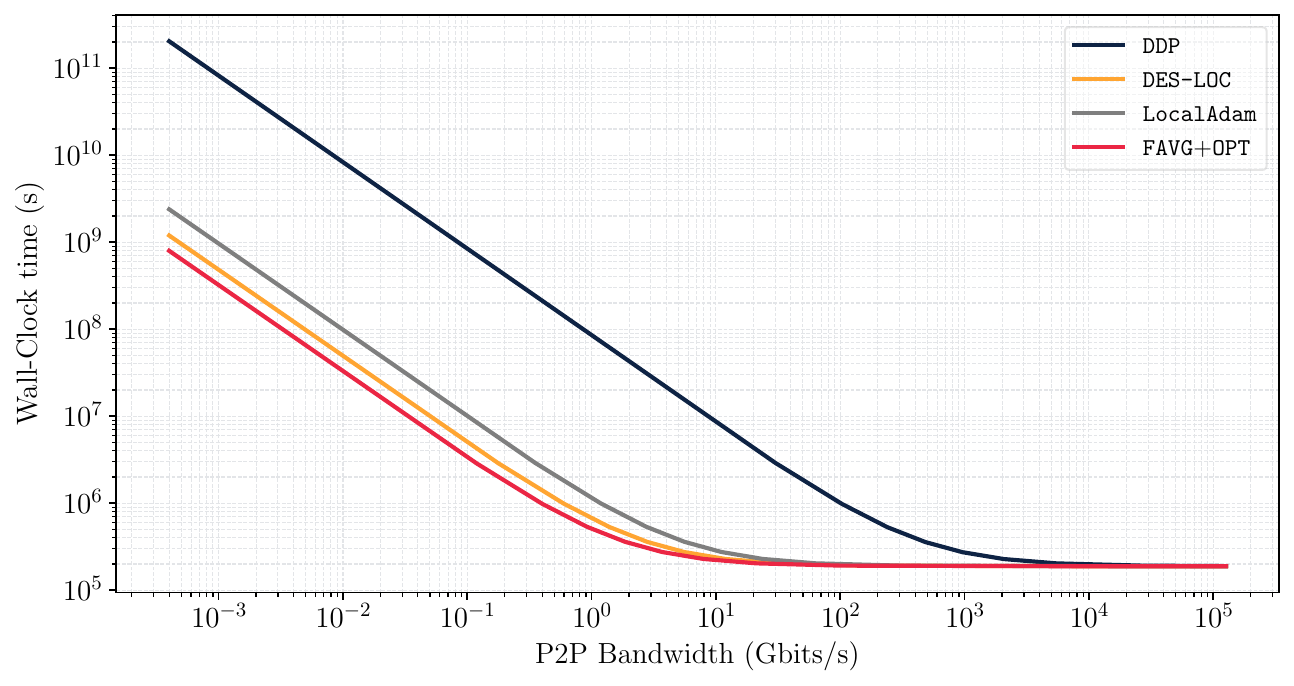} \hfill
    % \noindent\subfloat[\methodadopt $1$B-parameter norms]
    % {\includegraphics[width=0.485\columnwidth]{appendix_plots/wall_time/comm_cost_1B.pdf}}  \hfill
    % {\includegraphics[width=0.485\columnwidth]{appendix_plots/wall_time/compute_utilization_1B.pdf}} 
    \caption{Estimated wall-clock time for training the $1.7$B model with \method ($K_x=256, K_u=768, K_v=1536$), compared to \localadam ($K=256$), \ddp, and Federated Averaging with persistent optimizer states (\texttt{FAVG+OPT}, $K=256$). At low bandwidth ($<10^3$), all communication-efficient methods substantially reduce wall-clock time compared to \ddp. \method closely approaches the maximum efficiency of \texttt{FAVG+OPT}, significantly outperforming \localadam, which synchronizes all optimizer states frequently. Moreover, \method maintains stable and convergent training behavior (\cref{fig:eval:large_models}). At high bandwidth ($>10^3$), \ddp becomes competitive or preferable.}
    \label{app:fig:wall_clock_time_model}
\end{figure}

\begin{figure}[H]
    \centering
    %   \noindent\subfloat[\methodadopt $1$B-update norms]
    % {\includegraphics[width=1.0\columnwidth]{appendix_plots/wall_time/wall_time_model_1B.pdf}} \hfill
    \noindent\subfloat[Communication costs.]
    {\includegraphics[width=0.485\columnwidth]{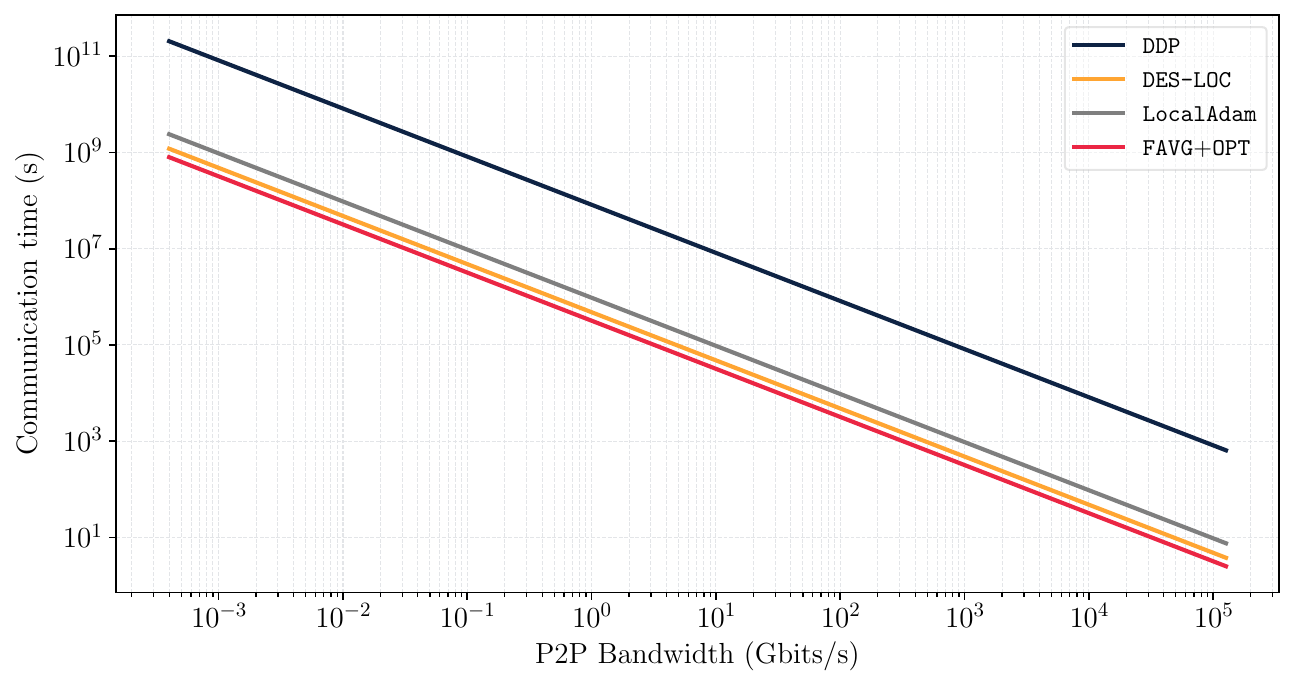}}  \hfill
    \noindent\subfloat[Compute utilization.]
    {\includegraphics[width=0.485\columnwidth]{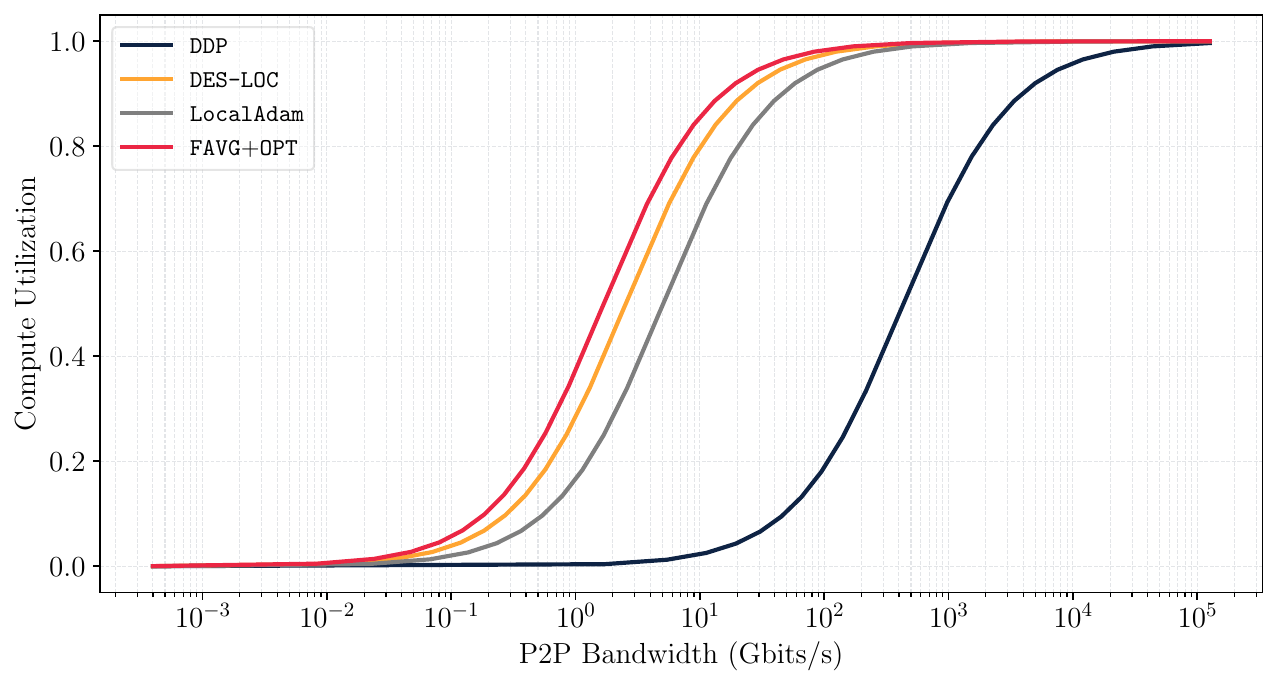}} 
    \caption{Communication overhead (a) and GPU utilization (b) for training the $1.7$B model with synchronization periods $K_x=256,K_u=768,K_v=1536$. \method reduces communication costs by $170\times$ compared to \ddp, outperforming the $85\times$ reduction achieved by \localadam while \texttt{FAVG+OPT}, communicating only parameters, achieves a theoretical maximum reduction ($256\times$). The improved communication efficiency of \method translates to higher GPU utilization at low bandwidths ($<10^3$), significantly improving over \ddp and \localadam.}
    \label{app:fig:comms_and_compute_utilization}
\end{figure}

\takeawaybox[Takeaway:]{%
By synchronizing optimizer states less frequently, \method enhances GPU utilization and total wall-clock time compared to \ddp and \localadam, especially under bandwidth constraints.
}